\documentclass[aps,prd,reprint,superscriptaddress,amsmath,longbibliography,floatfix]{revtex4-2}
\usepackage{newtxmath}
\DeclareSymbolFont{operators}{OT1}{ntxtlf}{m}{n}
\SetSymbolFont{operators}{bold}{OT1}{ntxtlf}{b}{n}
\usepackage[T1]{fontenc}
\usepackage[utf8]{inputenc}
\usepackage{newtxtext}
\usepackage{graphicx,booktabs,array,bm}
\usepackage{microtype}
\usepackage{float}
\usepackage{placeins}
\usepackage[hidelinks,breaklinks=true]{hyperref}
\makeatletter
\newenvironment{widefigure}[1][tbp]{\begin{figure*}[#1]\centering}{\end{figure*}}
\newenvironment{widetable}[1][]{%
  \par\onecolumngrid\vskip8pt\noindent\begin{minipage}{\textwidth}%
  \def\@captype{table}\centering
}{%
  \end{minipage}\par\vskip8pt\twocolumngrid\global\@ignoretrue\@endpetrue
}
\makeatother
\begin{document}
\raggedbottom
\title{Round-Trip Updates of Learned Nonequilibrium Hamiltonian Paths: Predictable Acceptance and Generalization in Lattice Field Theory}
\author{Moxian Qian}
\email{mqian@students.uni-mainz.de}
\affiliation{Johannes Gutenberg University Mainz, Mainz, Germany}
\affiliation{Institute of Molecular Biology (IMB), Mainz, Germany}
\affiliation{Helmholtz Institute Mainz, Mainz, Germany}
\hypersetup{pdftitle={Round-Trip Updates of Learned Nonequilibrium Hamiltonian Paths: Predictable Acceptance and Generalization in Lattice Field Theory},pdfauthor={Moxian Qian}}
\begin{abstract}
Observables in lattice field theory are usually computed by Markov chain Monte Carlo sampling. Learned generative models can propose candidate configurations independently. However, the mismatch between the trained distribution and the target grows with the volume and is largest near criticality. We apply round-trip updates to the learned multistage paths of neural nonequilibrium Hamiltonian Monte Carlo (NHMC). The current configuration walks back along a reverse path to an intermediate field. A full round trip walks back to the start of the path. A new configuration is then generated forward from this field. The acceptance is given by the probability ratio of the new and old paths, so the target distribution is preserved exactly for any model parameters. The mean acceptance is approximately $\operatorname{erfc}(\sigma_s/2)$. Here $\sigma_s^2$ is the conditional variance of the path work at a fixed intermediate field, and the path work is the log ratio of forward and reverse path probabilities. This variance can be measured on forward paths before any chain is run. The acceptance in turn sets the autocorrelation time of full round trips. In two-dimensional $\phi^4$ theory near criticality, round-trip updates have a far shorter autocorrelation time than independent proposals from the same model. Models trained on small lattices can be frozen and used on larger ones, and their observables agree with the references. Away from criticality their measured acceptances agree with the prediction. Combining forward paths with the reverse paths of the chains gives accurate free energies. Near criticality the number of stages needed to keep the acceptance of full round trips grows roughly as $V^{1.25}$. Critical slowing down thus appears as a requirement on path length.
\end{abstract}

\maketitle
\section{Introduction}
\label{sec:introduction}

Observables in lattice field theory are usually computed by Markov chain Monte Carlo sampling, with Metropolis updates and Hamiltonian Monte Carlo (HMC) as common choices~\cite{Metropolis:1953am,Hastings:1970aa,Duane:1987de,Neal_2011}. Near a continuous phase transition the correlation length grows, local updates relax long-wavelength fluctuations slowly, and the autocorrelation time of observables grows. This is critical slowing down~\cite{Sokal_1997,Schaefer:2010hu}. Learned generative models offer another route. Normalizing flows map a simple distribution to a learned configuration distribution~\cite{pmlr-v37-rezende15,JMLR:v22:19-1028}. In lattice field theory they can be viewed as learned trivializing maps~\cite{Luscher:2009eq,Albergo:2019eim,Bacchio:2022vje}, and independently generated candidates are corrected by a Metropolis step for equilibrium sampling~\cite{Albergo:2019eim,Nicoli:2020njz}. Gauge-equivariant, conditional and multiscale constructions extend these methods~\cite{Kanwar:2020xzo,Boyda:2020hsi,Singha:2022icw,Singha:2026slc}. The main limitation of such samplers is the mismatch between the trained distribution and the target. For a fixed model, the effective sample fraction of the importance weights decreases exponentially with the volume~\cite{Finkenrath:2022ogg,Abbott:2022zsh}.

Learned distributions are usually used for equilibrium sampling through independence Metropolis (IMH), in which the candidate does not depend on the current configuration~\cite{Hastings:1970aa,Liu_1996}. When the proposal is a whole learned path, we call it path-IMH. When the mismatch is large, the chain rejects candidates many times in a row and stays in states of large weight for long periods~\cite{Liu_1996,Mengersen_1996}. The mismatch grows with the volume and is largest near criticality. We therefore need an update that samples correctly despite the mismatch. We would also like to know, before any chain is run, how strongly the mismatch affects it and whether a model trained on a small lattice can be used on larger ones.

Existing path methods correct the proposal with the probability ratio of forward and reverse paths. The Jarzynski equality and the Crooks relation connect this ratio to nonequilibrium work and free-energy differences~\cite{Jarzynski:1996oqb,Crooks:1999ttd}. Stochastic normalizing flows and nonequilibrium lattice methods compute path weights from the complete forward and reverse records~\cite{NEURIPS2020_41d80bfc,Caselle:2022acb,Bonanno:2024udh,Bulgarelli:2024brv}. The number of steps they need to keep the sampling quality grows linearly with the number of degrees of freedom~\cite{Bonanno:2024udh,Bulgarelli:2024brv}. The stochastic path sampler learns forward and reverse drifts on overdamped Langevin paths and applies independence Metropolis to the whole path~\cite{Chen:2026bqw}. These methods all correct independently generated whole paths. Nonequilibrium candidate Monte Carlo generates candidates along nonequilibrium paths from the current configuration and corrects them with the path work in a Metropolis step~\cite{Nilmeier_2011}. Its paths, however, follow a fixed protocol. Learned corrections that depend on the current configuration can also be obtained by mixing with local updates or by block resampling~\cite{Gabrie:2021tlu,Finkenrath:2022ogg,Qian:2025uvx}.

We study neural nonequilibrium Hamiltonian Monte Carlo (NHMC). It divides a path into $T$ stages, draws a configuration-dependent momentum in every stage, and drives the evolution with a learned force. We run both directions of the same model, and both directions are used in training, sampling, and free-energy estimation. On these paths we use round-trip updates. A round-trip update walks back from the current configuration along a reverse path to an intermediate field and generates a new configuration from it. Walking back to the start of the path gives a full round trip, and walking back over only the last few stages gives a truncated round trip. This update preserves the target distribution exactly for any model parameters~\cite{pmlr-v119-neklyudov20a}. Unlike an independent proposal, it is conditioned on the current configuration. The old and new configurations share the intermediate field.

In two-dimensional $\phi^4$ theory, round-trip chains give observables that agree with the references at the training size. Independent proposals from the same model stay in place for long periods near criticality. Counted in stages, the autocorrelation time of full round trips is less than one eighth of that of these independent proposals. The efficiency of round-trip updates depends on the size of the mismatch. The relevant measure is the conditional work variance at the shared intermediate field. This variance can be measured on forward paths before any chain is run. It sets the acceptance, and the acceptance in turn sets the autocorrelation time of full round trips.

We also freeze the parameters of a trained model and apply it directly to lattices of other sizes. We call this transfer, that is, generalization of the model across lattice sizes. Away from criticality a single measurement of the conditional work variance at the training size predicts the acceptance of a frozen model on larger lattices. For models trained at $L=16$ and transferred to lattices up to $L=64$, $92\%$ of the observable values agree with the references within two standard errors. The measured acceptance of truncated round trips is within $30\%$ of this prediction. A model trained on $64\times8$ and frozen for use on $512\times8$ gives accurate free energies through bidirectional estimates. These combine forward paths with the reverse paths and their probabilities recorded by round-trip chains.

Near criticality, and for full round trips in the broken phase, the per-site conditional work variance grows with the lattice size. Transfer of frozen models reaches its limit there. One then needs more path stages or retraining at the target size. The number of stages needed to keep the acceptance of full round trips grows roughly as $V^{1.25}$. Critical slowing down thus appears as a requirement on path length.

Section~\ref{sec:nhmc} introduces NHMC and the round-trip update. Section~\ref{sec:native} gives the observables of native models, which are trained directly at the target size. Section~\ref{sec:scaling} shows that acceptance and autocorrelation are set by the conditional work variance. Section~\ref{sec:cost_size} tests the observables of frozen models, predicts the acceptance after transfer, and discusses the limits of transfer. Section~\ref{sec:free_energy} gives the bidirectional free-energy estimates. The appendices give the proof of detailed balance, the numerical settings, supplementary results, and comparisons with controlled Monte Carlo diffusion (CMCD) and diffusion bridge samplers (DBS)~\cite{vargas2024transport,richter2024improved}.

\section{Neural nonequilibrium Hamiltonian Monte Carlo}
\label{sec:nhmc}

To construct the NHMC sampler, consider a lattice field configuration
$x\in\mathbb R^V$ with $V$ real degrees of freedom and target Boltzmann distribution
\begin{equation}
 \pi(x)=\frac{e^{-S(x)}}{Z},\qquad Z=\int dx\,e^{-S(x)},
 \label{eq:nhmc_target}
\end{equation}
where $S(x)$ is the dimensionless action and $Z$ is the partition function. We denote the unnormalized target density by $\gamma(x)=e^{-S(x)}$.

\subsection{Auxiliary momenta and learned dynamics}
\label{sec:nhmc_dynamics}

As in standard HMC, we introduce an auxiliary momentum
$p$ for every field degree of freedom. This extends configuration space to phase
space~\cite{Duane:1987de,Neal_2011}. Standard HMC draws momenta from a fixed Gaussian distribution and evolves deterministically with the Hamiltonian $S(x)+p^{\mathsf T}p/2$.

NHMC generalizes the fixed momentum distribution to a conditional law depending
on the field at each stage. It also introduces a learned force to control phase-space
motion. Let $\theta$ denote model parameters and draw $x_0$ from a normalized
initial distribution $q_0$. A forward path has $T$ stages. Each stage draws a
conditional momentum and applies deterministic evolution,
\begin{equation}
 \begin{gathered}
 p_t\sim q^{\rm F}_{\theta,t}(\cdot\mid x_t),\\
 (x_{t+1},\bar p_{t+1})=\Phi_{\theta,t}(x_t,p_t),
 \quad t=0,\ldots,T-1.
 \end{gathered}
 \label{eq:nhmc_stage}
\end{equation}
Here $q^{\rm F}_{\theta,t}$ is the forward conditional momentum density,
$\Phi_{\theta,t}$ the deterministic phase-space map, and $\bar p_{t+1}$ the
outgoing momentum of the stage. The conditional momentum is Gaussian.
Neural networks give its mean and width. In the numerical implementation,
the evolution force $G_{\theta,t}(x)$ combines the analytic force of the target action,
the force of the Gaussian initial distribution $q_0$, and a learnable residual.
The combination coefficients and step sizes at each stage are trained.
Appendix~\ref{app:implementation} specifies the parameterization.

We construct reversible, volume-preserving trajectories using symmetric
leapfrog integration. Define momentum and position updates
\begin{equation}
 \begin{aligned}
 \mathsf K_\epsilon(x,p)&=(x,p+\epsilon G_{\theta,t}(x)),\\
 \mathsf D_\epsilon(x,p)&=(x+\epsilon p,p),
 \end{aligned}
 \label{eq:nhmc_shears}
\end{equation}
where $\epsilon$ is the substep length, with $\epsilon=\epsilon_t$ in stage $t$. Each stage comprises $n_{\mathrm{LF}}$ leapfrog substeps,
\begin{equation}
 \mathsf L_{\epsilon_t}
 =\mathsf K_{\epsilon_t/2}\circ\mathsf D_{\epsilon_t}
  \circ\mathsf K_{\epsilon_t/2},\qquad
 \Phi_{\theta,t}=\mathsf L_{\epsilon_t}^{n_{\mathrm{LF}}}.
 \label{eq:nhmc_kdk}
\end{equation}
Both updates are shear maps, and their symmetric composition can be inverted by momentum reversal.
With $\mathcal R(x,p)=(x,-p)$, we obtain
\begin{equation}
 |\det\partial\Phi_{\theta,t}|=1,\qquad
 \Phi_{\theta,t}^{-1}=\mathcal R\Phi_{\theta,t}\mathcal R.
 \label{eq:nhmc_reversal}
\end{equation}
These properties pair forward and reverse trajectories stage by stage.

For scalar fields with global $\mathbb Z_2$ symmetry, before the continuous
evolution of each stage we apply the sign transformation $x\mapsto-x$ with
fixed probability $p_{\rm flip}=1/2$. The discrete
choices and conditional momenta are included in the complete path record.

Because the path contains repeated random momentum refreshes,
the marginal proposal density of its endpoint is generally intractable.
We therefore turn to complete forward and reverse path probabilities for
model training and statistical weighting.

\subsection{Forward and reverse path probabilities and training}
\label{sec:nhmc_pathlaws}

Jarzynski's equality relates the exponential average of nonequilibrium work
to an equilibrium free-energy difference~\cite{Jarzynski:1996oqb}. Crooks'
relation connects forward and reverse evolution probabilities~\cite{Crooks:1999ttd}.
Following the extended-space formulation of Hamiltonian variational inference
and stochastic normalizing flows~\cite{pmlr-v37-salimans15,NEURIPS2020_41d80bfc}, we construct a
paired reverse path law for NHMC. We define the path work through the
complete path-density ratio.

For the continuous stage in Eq.~\eqref{eq:nhmc_stage}, the paired reverse
input is $-\bar p_{t+1}$. A normalized conditional Gaussian density
$q^{\rm R}_{\theta,t}(\cdot\mid x_{t+1})$ describes reverse momenta. The mean
and width of $q^{\rm R}_{\theta,t}$ also come from networks, with parameters independent of the forward law. Denote the path
endpoints by $u=x_0$ and $x=x_T$. Let $a$ collect the forward momenta
and sign labels, and $b$ the paired reverse variables. Invertibility makes
either $(u,a)$ or $(x,b)$ a complete description of the same path $\Gamma$.

Write $f_\theta(a\mid u)$ and $r_\theta(b\mid x)$ for the products of the
actual conditional densities, including all discrete selection probabilities.
Because $(u,a)$ and $(x,b)$ describe the same path, both laws can be written in the same path coordinates. In these coordinates the forward law and reverse reference law are
\begin{equation}
 \begin{aligned}
 Q_\theta(\Gamma)&=q_0(u)f_\theta(a\mid u),\\
 \widetilde P_\theta(\Gamma)&=\gamma(x)r_\theta(b\mid x),\qquad
 P_\theta=\frac{\widetilde P_\theta}{Z}.
 \end{aligned}
 \label{eq:nhmc_forward_reverse}
\end{equation}
The endpoint marginal of $P_\theta$ is $\pi$. Both densities can be evaluated directly, and the (dimensionless) path work of a complete path is
\begin{equation}
 \begin{aligned}
 W_\theta(\Gamma)
 &=\log Q_\theta(\Gamma)-\log\widetilde P_\theta(\Gamma)\\
 &=\log q_0(u)+S(x)+\log f_\theta(a\mid u)-\log r_\theta(b\mid x).
 \end{aligned}
 \label{eq:nhmc_work}
\end{equation}
Without discrete transformations, the last two terms are
the sum of
$\log[q^{\rm F}_{\theta,t}(p_t\mid x_t)/q^{\rm R}_{\theta,t}(-\bar p_{t+1}\mid x_{t+1})]$
over stages.

When every path with nonzero $P_\theta$ also has nonzero $Q_\theta$ and the averages below exist,
\begin{equation}
 \frac{dP_\theta}{dQ_\theta}=\frac{e^{-W_\theta}}{Z}.
 \label{eq:nhmc_pathratio}
\end{equation}
Normalization and the endpoint marginal of $P_\theta$ then give, for an
integrable observable $O$,
\begin{equation}
 \mathbb E_{Q_\theta}[e^{-W_\theta}]=Z,\qquad
 \langle O\rangle_\pi=
 \frac{\mathbb E_{Q_\theta}[e^{-W_\theta}O(x)]}
      {\mathbb E_{Q_\theta}[e^{-W_\theta}]}.
 \label{eq:nhmc_exponential_identity}
\end{equation}
The first identity is the Jarzynski-type relation for this path construction.
Because the initial law $q_0$ is normalized, no reference term appears, and the dimensionless free energy is
$F=-\log Z=-\log\mathbb E_{Q_\theta}[e^{-W_\theta}]$.
The second identity gives the reweighting of equilibrium observables.
Section~\ref{sec:free_energy} considers free-energy estimation using samples of both path directions.

The training objective of the learned dynamics is built from the same path work. From
Eq.~\eqref{eq:nhmc_pathratio}, the mean path work satisfies
\begin{equation}
 \mathbb E_{Q_\theta}[W_\theta]
 =F+D_{\mathrm{KL}}(Q_\theta\Vert P_\theta)\geq F,
 \label{eq:nhmc_training_kl}
\end{equation}
where $D_{\mathrm{KL}}$ denotes the Kullback--Leibler divergence.
The mean path work is a variational upper bound on the target free energy.
Since $F$ is independent of the learned parameters, minimizing the mean path work
minimizes the divergence between forward and reverse paths. To control work fluctuations,
we add a variance penalty. The forward-path training term is
\begin{equation}
 \mathcal L_{\mathrm F}
 =\mathbb E_{Q_\theta}[W_\theta]
  +\nu\operatorname{Var}_{Q_\theta}(W_\theta),
 \qquad \nu\geq0.
 \label{eq:nhmc_loss}
\end{equation}
It is evaluated on generated forward paths.
We sample the continuous momenta using a Gaussian reparameterization and propagate gradients
through the deterministic segments. Once the round-trip update is defined in Sec.~\ref{sec:nhmc_correction}, training also uses the configurations it produces.

\subsection{Independent-path and round-trip Metropolis updates}
\label{sec:nhmc_correction}

Training brings the forward-path law closer to the reverse reference law.
After training, however, the endpoint law only approximates $\pi$.
Equilibrium measurements therefore need either path reweighting or a target-preserving
update. Path reweighting supplies equilibrium averages. A target-preserving update supplies a configuration
chain for equally weighted measurements. We freeze all learned parameters after
training. Independence Metropolis gives the established
construction~\cite{Liu_1996,NEURIPS2020_41d80bfc}. The round-trip update introduced here uses the current
field and builds on involutive Markov chain Monte Carlo (MCMC)~\cite{pmlr-v119-neklyudov20a}. Since $\theta$ is fixed during sampling, we drop the subscript where no confusion arises, e.g.\ $W\equiv W_\theta$ and $Q\equiv Q_\theta$.

Path-space independence Metropolis (path-IMH) draws a candidate
$\Gamma'$ independently from $Q_\theta$ and targets $P_\theta$ in the
extended space. Equation~\eqref{eq:nhmc_pathratio} gives
\begin{equation}
 \alpha_{\mathrm{IMH}}(\Gamma,\Gamma')
 =\min\!\left\{1,\exp[W_\theta(\Gamma)-W_\theta(\Gamma')]\right\}.
 \label{eq:nhmc_imh_acceptance}
\end{equation}
On acceptance the chain moves to the new endpoint and stores its path work, and on rejection it keeps the old endpoint and its path work.

To use the current configuration in proposing the next one, we construct
round-trip Metropolis updates, which we call round trips. Starting from the current
endpoint $x$, we draw reverse momenta and follow the path back to a fixed
stage $s$. This gives the intermediate field $u_s$. Fresh forward
momenta from the same $u_s$ then generate the candidate endpoint $y$.
Let $a_{\ge s},b_{\ge s}$ denote the paired records of the old endpoint on the redrawn stages $s,\ldots,T-1$, and $a'_{\ge s},b'_{\ge s}$ those
of the new endpoint. With $f_{\theta,s},r_{\theta,s}$ the conditional-density
products over stages $s,\ldots,T-1$, the acceptance probability is
\begin{equation}
 \alpha_s=\min\!\left\{1,
 \frac{\gamma(y)r_{\theta,s}(b'_{\ge s}\mid y)f_{\theta,s}(a_{\ge s}\mid u_s)}
 {\gamma(x)r_{\theta,s}(b_{\ge s}\mid x)f_{\theta,s}(a'_{\ge s}\mid u_s)}\right\}.
 \label{eq:nhmc_rt_acceptance}
\end{equation}
Only the target action and the actual forward and reverse probabilities
are needed. The partition function and the marginal density of the
intermediate field are not needed. $\alpha_s$ is the acceptance probability of one update. Its stationary mean is $\bar\alpha_s\equiv\mathbb E[\alpha_s]$. For full round trips we write it as $\bar\alpha$. The resulting configuration kernel is $K_s$. The full round trip (RT) has $s=0$ and returns to stage zero. The truncated round trip (pRT, p for partial) has $0<s<T$ and stops at an interior stage.
The Markov chains produced by path-IMH, RT, and pRT are together called corrected chains. We call $s$ the return stage and $d=T-s$ the return depth. One path-IMH update passes through $T$ stages, and one round-trip update passes through $2d$ stages, which is $2T$ for the full round trip. Stages $s,\ldots,T-1$ keep their original parameters and time indices,
so truncation requires no retraining.

Exchanging the old and new records at their shared intermediate field gives
detailed balance for the configuration update with respect to $\pi$.
Appendix~\ref{app:path_correction} proves this statement at fixed depth and
specifies its extension to other bidirectional paths. On rejection, a full round trip retains
the current field. The next update generates a new reverse record and does not carry the previous auxiliary path.

Round-trip chains are also used in training. From epoch $1000$ of $20000$ we run $64$ full round-trip chains and keep their most recent configurations in a buffer. From these configurations we generate complete reverse paths. The loss then includes the negative mean path work of these paths and the corresponding path-work variance penalty (Appendix~\ref{app:implementation}). Since $\mathbb E_{P_\theta}[-W_\theta]=D_{\mathrm{KL}}(P_\theta\Vert Q_\theta)-F$, this term lowers the reverse divergence.

\subsection{Relation to other methods}
\label{sec:related}

The closest related method is the underdamped diffusion bridge sampler~\cite{blessing2025underdamped}. It also learns forward and reverse processes in an extended space with momenta, weights paths by the ratio of forward and reverse path probabilities, and is trained by lowering the Kullback--Leibler divergence between the two. Hamiltonian variational inference also learns multistep Hamiltonian transitions together with a reverse auxiliary model~\cite{pmlr-v37-salimans15}. Hamiltonian annealed importance sampling and its differentiable versions refresh momenta between neighboring stages and learn global parameters such as step sizes and mass matrices~\cite{Neal:2001hwj,geffner2021mcmc,zhang2021differentiable}. On the HMC side, Riemann manifold HMC takes the momentum covariance from a configuration-dependent metric~\cite{girolami2011riemann}. L2HMC and its lattice extensions modify each leapfrog step with networks~\cite{Levy:2017bji,Foreman:2021ixr,Foreman:2021rhs}. NeuTra changes coordinates with a learned transformation before running HMC~\cite{hoffman2019neutralizingbadgeometryhamiltonian}. Compared with these methods, NHMC redraws a configuration-dependent momentum in every stage and drives the evolution with a learned force. The main difference is that NHMC runs in both directions. These methods usually generate paths only forward from the initial distribution and use the reverse process only to compute path weights. NHMC also runs reverse paths from target configurations and uses them in the reverse divergence term of training, in the round-trip update, and in the bidirectional free-energy estimate (Secs.~\ref{sec:nhmc_pathlaws}, \ref{sec:nhmc_correction} and \ref{sec:free_energy}). Latent-space methods change coordinates with a deterministic invertible map and then run HMC in the new coordinates~\cite{Luscher:2009eq,Li:2018nrg}. Neural canonical transformations extend this map to the phase space of fields and momenta~\cite{Li:2019ncf}. The round-trip update is closer in structure to tempered transitions~\cite{Neal_1996}. The difference is that every step is a learned nonequilibrium transition. The old and new configurations also share the intermediate field and differ only in the random part after it.

\section{Equilibrium observables of a lattice scalar field}
\label{sec:native}
\label{sec:proposals_equilibrium}

This section tests the observables given by the corrected chains in two-dimensional lattice $\phi^4$ theory. The models are trained directly at the target size, and we call them native models.

\subsection{Model and setup}

On an $L\times L$ lattice $\Lambda$ with periodic boundaries and volume $V=L^2$, the action is
\begin{equation}
 S[\phi]=\sum_{\boldsymbol r\in\Lambda}\left[
 -2\kappa\sum_{\mu=1}^{2}\phi_{\boldsymbol r}\phi_{\boldsymbol r+\hat\mu}
 +(1-2\lambda)\phi_{\boldsymbol r}^2+\lambda\phi_{\boldsymbol r}^4\right].
 \label{eq:phi4_action}
\end{equation}
Here $\boldsymbol r$ labels lattice sites, $\kappa$ is the nearest-neighbor coupling, $\lambda$ is the quartic
coupling, and $\hat\mu$ is a positive lattice unit vector. The field $\phi$ is the configuration $x$ of Sec.~\ref{sec:nhmc}.
We take $\lambda=0.022$ and zero external field. Square lattices have $L=4,8,12,16$, with $13$ couplings in $0.20\leq\kappa\leq0.30$ (Appendix~\ref{app:implementation}).
This range contains the critical coupling $\kappa_c=0.2705\pm0.0007$, obtained from the Binder-cumulant crossing~\cite{Binder:1981sa} at $L=32$ and $64$~\cite{Singha:2026slc}. We call $\kappa<\kappa_c$ the symmetric phase and $\kappa>\kappa_c$ the broken phase. We take $\kappa=0.27$ as our near-critical point.
We train a separate NHMC model at each $(L,\kappa)$.
At a given target point, the weighted forward paths and the three corrected chains use the same learned parameters.
Each path has $T=16$ stages. Full round trips redraw every stage, and truncated round trips return to $s=12$. The full scan over the return stage is given in Supplemental Sec.~\ref{sm:native_extra}.
Independent HMC and Wolff cluster~\cite{Wolff:1988uh} calculations provide the reference values of the observables. Thermodynamic integration (TI) provides the partition-function reference. Appendix~\ref{app:implementation} specifies the training and reference calculations.

The volume-averaged magnetization is
\begin{equation}
 m(\phi)=\frac{1}{V}\sum_{\boldsymbol r\in\Lambda}\phi_{\boldsymbol r}.
 \label{eq:phi4_magnetization}
\end{equation}
We measure the absolute magnetization $\langle|m|\rangle$, the susceptibility
$\chi$ (the fluctuation of the magnetization amplitude), and the Binder cumulant $U_4$, using
\begin{equation}
 \chi=V\left(\langle m^2\rangle-\langle|m|\rangle^2\right),\qquad
 U_4=1-\frac{\langle m^4\rangle}{3\langle m^2\rangle^2}.
 \label{eq:phi4_susceptibility_binder}
\end{equation}
$U_4$ probes the shape of the magnetization distribution.
All three are unchanged by a global sign flip. Motion between positive and
negative regions is described by autocorrelations of the sign of $m$, $\operatorname{sign}m$ (Supplemental Sec.~\ref{sm:autocorr}).

\begin{widefigure}[tb]
 \centering
 \includegraphics[width=0.9\textwidth,keepaspectratio]{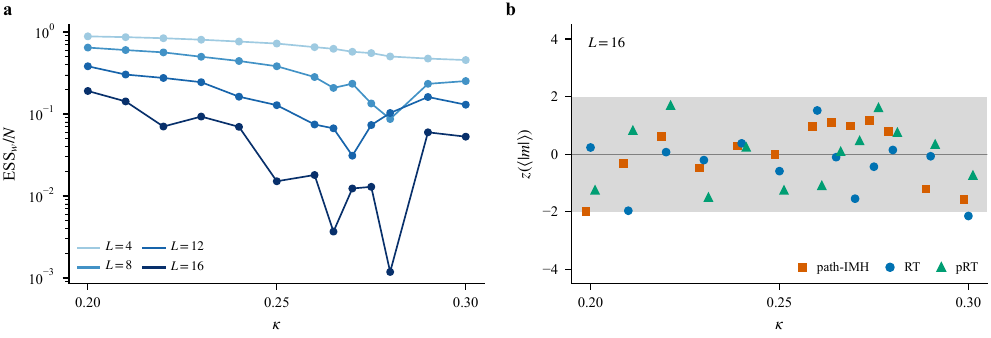}
 \caption{Equilibrium observables from the corrected chains on native models (a separate model trained at each $(L,\kappa)$). (a) Effective sample fraction of $65536$ forward paths, Eq.~\eqref{eq:phi4_snis}, with color indicating $L$. (b) Standardized deviation $z$ of $\langle|m|\rangle$ from the HMC and Wolff references at $L=16$. Squares, circles and triangles are path-IMH, RT and pRT, and the gray band marks $|z|\le2$. Statistics are described in Appendix~\ref{app:implementation}.}
 \label{fig:sampling}
\end{widefigure}

\subsection{Equilibrium observables of native models}

For $N$ independent forward paths with weights $w_i=e^{-W_i}$, let $\mathrm{ESS}_w$ denote the effective sample size (ESS) computed from the weights. The effective sample fraction is
\begin{equation}
 \frac{\mathrm{ESS}_w}{N}
 =\frac{(\sum_iw_i)^2}{N\sum_iw_i^2}.
 \label{eq:phi4_snis}
\end{equation}
Throughout, $w=e^{-W}$ denotes the unnormalized path weight, and $\operatorname{Var}(\log w)=\operatorname{Var}(W)$.

The effective sample fraction decreases with the volume. At $L=16$ it is smallest near criticality, about $10^{-3}$ to $10^{-2}$ (Fig.~\ref{fig:sampling}(a)). Unweighted endpoint averages systematically underestimate $\langle|m|\rangle$ for $\kappa\le0.26$ (Supplemental Sec.~\ref{app:phi4_numerics}). Equilibrium measurements therefore need correction.

We denote by $X_n$ the configuration of a corrected chain after each acceptance or rejection. Equilibrium averages use the equally weighted measurements after thermalization,
\begin{equation}
 \widehat{\langle O\rangle}_{\mathrm{chain}}
 =\frac{1}{N_{\rm MC}}\sum_{n=1}^{N_{\rm MC}} O(X_n).
 \label{eq:phi4_chain_mean}
\end{equation}
We compare a chain estimate $\widehat O$ with the reference value $O_{\rm ref}$ through the standardized deviation $z=(\widehat O-O_{\rm ref})/(\mathrm{SE}_O^2+\mathrm{SE}_{\rm ref}^2)^{1/2}$, where $\mathrm{SE}_O$ and $\mathrm{SE}_{\rm ref}$ are the two standard errors. The three corrected chains give $\langle|m|\rangle$, $\chi$, and $U_4$ in agreement with the references. At $L=16$ every $\langle|m|\rangle$ has $|z|\le2.2$ (Fig.~\ref{fig:sampling}(b)). Over all sizes and the three observables, $95\%$ of the values have $|z|\le2$, and the largest deviation is $3.3$. The complete observable curves are given in Supplemental Sec.~\ref{app:phi4_numerics}.

The three corrected chains therefore differ in how fast the configuration changes rather than in the observables. The next section describes this difference through acceptance and autocorrelation.

\section{Acceptance and autocorrelation}
\label{sec:scaling}

On native models the acceptance and the autocorrelation time are both set by the conditional work variance at the shared intermediate field. Section~\ref{sec:cost_size} uses the scaling of this variance with the volume to predict how frozen models perform on larger lattices.

\begin{widefigure}[tb]
\centering
\includegraphics[width=0.82\textwidth,keepaspectratio]{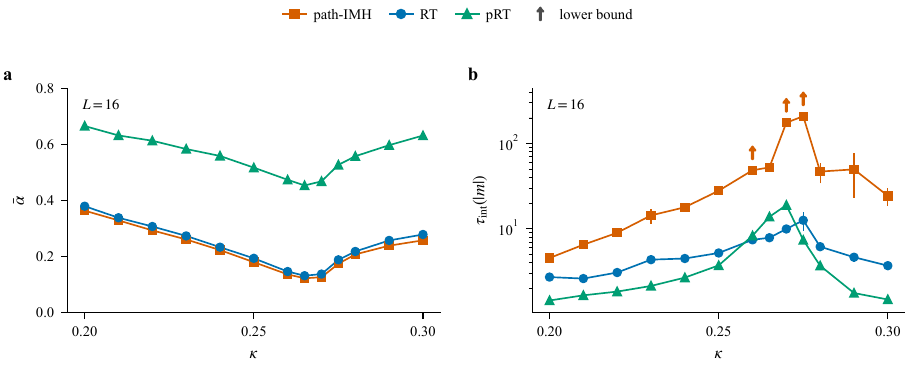}
\caption{Acceptance and autocorrelation measured on the native $L=16$ models. (a) Mean acceptance $\bar\alpha$. (b) $\tau_{\rm int}(|m|)$, with upward arrows for lower bounds. Squares, circles, and triangles are path-IMH, RT, and pRT.}
\label{fig:measured}
\end{widefigure}

\begin{widefigure}[tb]
\centering
\includegraphics[width=0.9\textwidth,keepaspectratio]{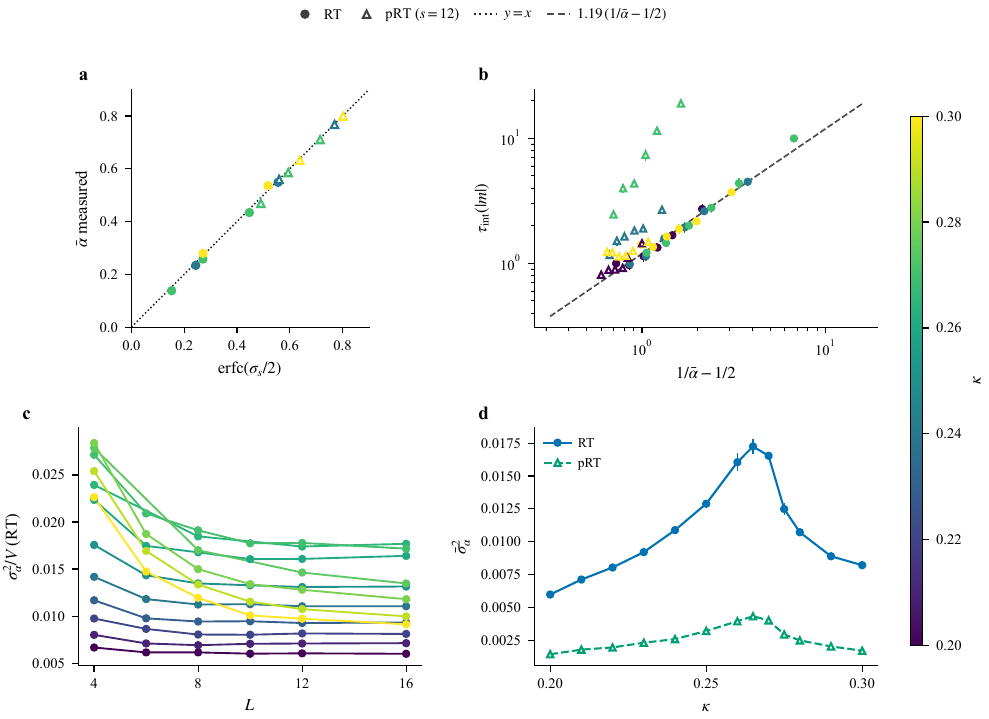}
\caption{Tests of the prediction on native models. (a) Measured mean acceptance versus the prediction $\operatorname{erfc}(\sigma_s/2)$, where $\sigma_s^2$ is the directly measured conditional work variance at a shared intermediate field, for seven models, $\kappa=0.24,0.27,0.30$ at $L=8$ and $16$ and $\kappa=0.27$ at $L=12$, with the dotted line $y=x$. (b) $\tau_{\rm int}(|m|)$ versus $1/\bar\alpha-1/2$, with the dashed line $1.19\,(1/\bar\alpha-1/2)$. (c) $\sigma_\alpha^2=4[\operatorname{erfc}^{-1}(\bar\alpha)]^2$ inferred from the mean full round-trip acceptance, divided by the volume, where $L=6,10$ have only the $11$ couplings spaced by $0.01$. (d) Per-site slope $\bar\sigma_\alpha^2$ versus $\kappa$, solid for full and dashed for truncated round trips ($s=12$). In (a) and (b) circles are full and triangles truncated round trips, and colors in (a) to (c) indicate $\kappa$. Appendices~\ref{app:implementation} and \ref{app:scaling} describe the error bars.}
\label{fig:predict}
\end{widefigure}

\subsection{Measured acceptance and autocorrelation}
\label{sec:native_tau}

For the states $X_n$ of a corrected chain we measure $O_n=O(X_n)$.
With one complete update as the unit of Monte Carlo time, the stationary autocorrelation function and the integrated autocorrelation time are
\begin{equation}
 \rho_O(j)=\frac{\langle O_0O_j\rangle_\pi-\langle O\rangle_\pi^2}
 {\operatorname{Var}_\pi(O)},\qquad
 \tau_{\mathrm{int}}(O)=\frac12+\sum_{j\geq1}\rho_O(j),
 \label{eq:nhmc_corr_convention}
\end{equation}
where $j$ is the separation in updates. The integrated autocorrelation time sets the asymptotic statistical error of stationary means~\cite{Sokal_1997,Wolff:2003sm}.

For each chain we take $\tau_{\mathrm{int}}$ as the larger of the Sokal automatic-window estimate~\cite{Madras:1988ei} and the batch-means estimate~\cite{Flyvbjerg_1989}. At each parameter point we take the mean over four chains, with the standard error of the mean as its error. Chains that do not meet the convergence criterion give only lower bounds, marked by upward arrows in the figures. Appendix~\ref{app:implementation} gives the estimators, the criterion, and the error bars.

At $L=16$ the mean acceptance of all three updates is lowest near criticality (Fig.~\ref{fig:measured}(a)). Truncated round trips redraw only the last four stages and have the highest acceptance. Full round trips and path-IMH redraw the whole path and have similar acceptances.
With the same model, $\tau_{\rm int}(|m|)$ of full round trips is shorter than that of path-IMH. At $L=16$ and $\kappa=0.27$ it is $10.0$ for full round trips and at least $176$ for path-IMH (Fig.~\ref{fig:measured}(b)). The two mean acceptances are nevertheless similar, $0.14$ and $0.13$. The long autocorrelation of path-IMH comes from long stays in states of large weight (Supplemental Sec.~\ref{sec:field_relaxation}). The acceptance probability of path-IMH depends on the weight of the current state. Once the chain reaches a state of large weight, it rejects candidates many times in a row. The acceptance probability of a round-trip update depends on the work difference of the old and new segments at a fixed intermediate field. The old record is regenerated in every update, so the path weight of the current state is not carried along.

The three updates pass through different numbers of stages per update (Sec.~\ref{sec:nhmc_correction}). To compare them we multiply $\tau_{\rm int}$ by the number of stages per update and call the result $\tau_{\rm stage}$, the integrated autocorrelation time with one stage as the unit of time. In terms of $\tau_{\rm stage}$, the advantage of full round trips over path-IMH is largest near criticality, more than a factor of eight at $\kappa=0.27$. At $\kappa=0.20$ the two are similar. Truncated round trips have the shortest $\tau_{\rm stage}$ at all couplings, about half that of full round trips at $\kappa=0.27$ (Supplemental Sec.~\ref{sm:tau_stage}).
Results at other sizes are given in Supplemental Sec.~\ref{sm:autocorr}.

\subsection{Acceptance formula}
\label{sec:acceptance_law}
\label{sec:acceptance_test}

Equation~\eqref{eq:nhmc_rt_acceptance} can be written as $\alpha_s=\min\{1,e^{W_{\theta,s}(\Gamma)-W_{\theta,s}(\Gamma^{\prime})}\}$.
Here the truncated path work $W_{\theta,s}(\Gamma)=S(x)+\log f_{\theta,s}(a_{\ge s}\mid u_s)-\log r_{\theta,s}(b_{\ge s}\mid x)$ contains only the redrawn stages $s,\ldots,T-1$.
For $s=0$ it differs from Eq.~\eqref{eq:nhmc_work} only by the intermediate-field term $\log q_0(u_0)$.
$\Gamma$ is the old record obtained backward from the current configuration, and $\Gamma^{\prime}$ is the new record generated forward from the same intermediate field $u_s$.
The mismatch between the learned conditional proposal and the target conditional distribution is described by the conditional work variance,
\begin{equation}
 \sigma_s^2=\mathbb E_{u_s\sim q_s}\!\big[\operatorname{Var}(W_{\theta,s}\mid u_s)\big],
 \label{eq:cond_work_var}
\end{equation}
where $q_s$ is the field distribution of forward paths at stage $s$. In the chain, $u_s$ follows the reverse marginal $\pi_s$ of Appendix~\ref{app:path_correction}. We average over $q_s$ because it can be sampled before any chain is run. Figure~\ref{fig:predict}(a) shows that this average predicts the acceptance measured on the chains.

At fixed $u_s$, the work difference $\Delta=W_{\theta,s}(\Gamma^{\prime})-W_{\theta,s}(\Gamma)$ between new and old records obeys the Crooks relation~\cite{Crooks:1999ttd}. At fixed $u_s$ the old and new records are drawn independently, so $\operatorname{Var}(\Delta)\approx2\sigma_s^2$. In the Gaussian approximation of $W_{\theta,s}$, $\langle\Delta\rangle=\operatorname{Var}(\Delta)/2\approx\sigma_s^2$ (Appendix~\ref{app:acceptance}). This gives the prediction for the mean acceptance
\begin{equation}
 \bar\alpha_s\equiv\mathbb E[\alpha_s]=\mathbb E\big[\min(1,e^{-\Delta})\big]\simeq\operatorname{erfc}(\sigma_s/2).
 \label{eq:erfc_law}
\end{equation}
For large $\sigma_s$, $\bar\alpha_s\approx\frac{2}{\sigma_s\sqrt\pi}e^{-\sigma_s^2/4}$.
Equation~\eqref{eq:erfc_law} is the classic Gaussian analysis of the HMC acceptance~\cite{Creutz:1988wv,Gupta:1988js,Kennedy:1990bv} applied to the conditional work. For independently generated whole paths the same analysis gives $\operatorname{erfc}(\sqrt{\operatorname{Var}(\log w)}/2)$~\cite{Chen:2026bqw}. In round-trip updates what matters is instead the conditional work variance after the intermediate field is fixed.

We measure $\sigma_s^2$ directly on the forward paths of native models. From each intermediate field we draw several continuation paths and average the within-group variance. For the seven native models, the full and truncated round-trip acceptances predicted by Eq.~\eqref{eq:erfc_law} differ from those measured on long chains by at most $0.022$ (Fig.~\ref{fig:predict}(a)). Part of this difference comes from the departure of the conditional work distribution from a Gaussian (Appendix~\ref{app:scaling}).

\subsection{Autocorrelation time}
\label{sec:tau_law}

In terms of the integrated autocorrelation time of Eq.~\eqref{eq:nhmc_corr_convention}, if every acceptance fully decorrelates the observable and the number of consecutive stays is geometric, then $\tau_{\rm int}=1/\bar\alpha-1/2$. The old and new configurations of a round-trip update share the intermediate field. This field retains part of the information of the old configuration, so one acceptance need not fully decorrelate (Supplemental Sec.~\ref{sec:conditional_flow}). We therefore write
\begin{equation}
 \tau_{\rm int}(O)\approx C\Big(\frac1{\bar\alpha}-\frac12\Big),
 \label{eq:tau_law}
\end{equation}
where $C$ is an empirical constant. The value $C=1$ corresponds to full decorrelation at every acceptance with an acceptance probability that does not depend on the configuration. Combined with Eq.~\eqref{eq:erfc_law}, the autocorrelation time can also be estimated in advance from $\sigma_s^2$.

On all $74$ native model points (Appendix~\ref{app:scaling}), we take the exponential of the median of the logarithms of the ratios $\tau_{\rm int}(|m|)/(1/\bar\alpha-1/2)$ and obtain $C=1.19\pm0.01$ for full round trips (Fig.~\ref{fig:predict}(b)). An accepted full round trip changes $|m|$ by about as much as the difference of two independent draws (Supplemental Sec.~\ref{sm:native_extra}). Each acceptance therefore nearly decorrelates the chain. That $C$ is slightly above $1$ is consistent with an acceptance probability that varies with the configuration. If $\alpha(x)$ is the mean acceptance from configuration $x$ and every acceptance gives an independent sample, then $\tau_{\rm int}=\mathbb E_\pi[1/\alpha(x)]-1/2$. This value is not smaller than $1/\bar\alpha-1/2$.

Away from criticality ($\kappa=0.20,0.24,0.30$) truncated round trips also follow the form of Eq.~\eqref{eq:tau_law}, with ratios between $1.2$ and $2.1$. At $\kappa=0.27$ the ratio grows from $3.5$ at $L=4$ to $12$ at $L=16$. When only the final stages are redrawn, the intermediate field retains most of the amplitude information. An accepted update then changes $|m|$ by only a fraction of its standard deviation, so the autocorrelation is markedly longer than the value given by the acceptance.

\subsection{Scaling of the conditional work variance}
\label{sec:var_scaling}

To apply the prediction at other sizes we need to know how the conditional work variance changes with the volume. A direct measurement of $\sigma_s^2$ needs repeated continuation paths from every intermediate field. For all native models we therefore infer $\sigma_\alpha^2\equiv4[\operatorname{erfc}^{-1}(\bar\alpha_s)]^2$ from the mean acceptance of long chains, separately for full and truncated round trips. Below we write $\sigma_s^2$ for the directly measured conditional work variance and $\sigma_\alpha^2$ for the value inferred from the acceptance. A bar marks the per-site value of either. In addition to the square lattices with $L=4$ to $16$, we use lattices with $L=6$ and $10$, each at the $11$ couplings spaced by $0.01$. At each coupling we fit
\begin{equation}
 \sigma_\alpha^2\approx\bar\sigma_\alpha^2V+c_\alpha ,
 \label{eq:var_volume}
\end{equation}
where the volume-independent constant $c_\alpha$ describes finite-size corrections. The directly measured $\sigma_s^2/V$ is close to the fitted slope $\bar\sigma_\alpha^2$ (Appendix~\ref{app:scaling}).
For $\kappa\le0.24$, $\sigma_\alpha^2/V$ hardly changes from $L=8$ to $16$ (Fig.~\ref{fig:predict}(c)).
$\bar\sigma_\alpha^2$ varies smoothly with the coupling. It peaks at $\kappa\approx0.265$, at about three times the symmetric-phase value at $\kappa=0.20$ (Fig.~\ref{fig:predict}(d)). The coefficient of truncated round trips ($s=12$) is about one quarter of that of full round trips. The reason is that $\sigma_s^2$ grows linearly with the number of redrawn stages $d=T-s$ and each stage contributes about equally (Supplemental Sec.~\ref{sm:scaling}).

Suppose the conditional work variance grows linearly with the volume, $\sigma_s^2\approx\bar\sigma_s^2V$. Equation~\eqref{eq:erfc_law} then gives, before any chain is run, the volume at which the acceptance falls to a given value $\bar\alpha_\ast$,
\begin{equation}
 V_{\bar\alpha_\ast}=\frac{4\,[\operatorname{erfc}^{-1}(\bar\alpha_\ast)]^2}{\bar\sigma_s^2}.
 \label{eq:vmax}
\end{equation}
This is the largest volume on which a frozen model can be used before its acceptance drops to $\bar\alpha_\ast$.

\section{Transfer of frozen models}
\label{sec:cost_size}
\label{sec:size_transfer}
\label{sec:boundary}

We apply frozen models to larger lattices and compare their observables with references and their acceptance with the predictions of Sec.~\ref{sec:scaling}.

\subsection{Equilibrium observables of frozen models}
\label{sec:lattice_size}

We fix all learned parameters, including the network weights, the stage step sizes, and the force combination coefficients. The coupling stays unchanged. We write $L_{\rm src}$ for the linear size of the square training lattice and $L$ for the target size. Each coupling uses its own trained model.

We use the $L_{\rm src}=16$ models at $L=24,32$, and for truncated round trips also at $L=48,64$. At $11$ couplings in $0.20\leq\kappa\leq0.30$ we compare $\langle|m|\rangle$, $\chi$, and $U_4$ with independent HMC and Wolff references at the target size. At $L=48$ and $64$ we run only truncated round trips, because the acceptance of the other two updates is too low there (Sec.~\ref{sec:transfer_accept}).
The coupling dependence of $\langle|m|\rangle$ follows the reference curves (Fig.~\ref{fig:transfer_obs}(a)). From $L=24$ to $64$ the full and truncated round trips reproduce the reference values of $\langle|m|\rangle$, $\chi$, and $U_4$. Of the $198$ values, $92\%$ have $|z|\le2$ (Fig.~\ref{fig:transfer_obs}(b) for $\langle|m|\rangle$). The largest deviation, $z=-4.4$, occurs for $\chi$ at $L=64$ and $\kappa=0.27$. Path-IMH deviates more at $L=32$. The results for $\chi$, $U_4$, and path-IMH are given in Supplemental Sec.~\ref{sm:transfer}.

\begin{widefigure}[tb]
 \centering
 \includegraphics[width=0.9\textwidth,keepaspectratio]{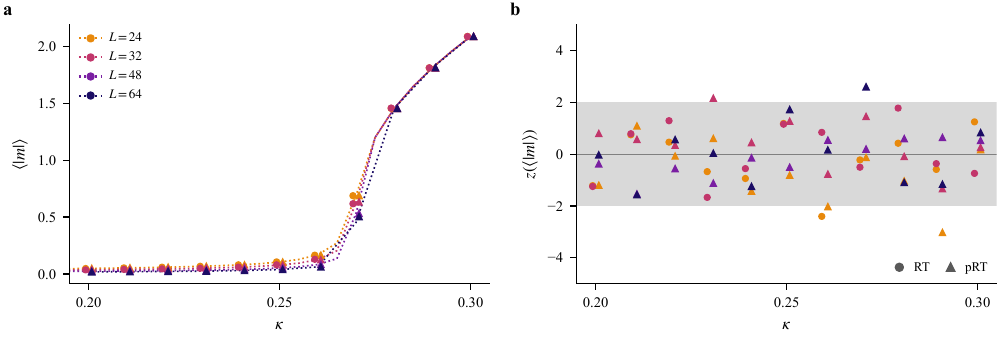}
 \caption{Equilibrium observables after transferring frozen $L_{\rm src}=16$ models (a separate source model trained at each coupling) to $L=24$ to $64$. (a) $\langle|m|\rangle$, with dotted lines for the HMC and Wolff references at the target size. (b) The corresponding standardized deviation $z$, with the gray band marking $|z|\le2$. Circles are full and triangles truncated round trips, color indicates the target size, and $L=48,64$ have truncated round trips only.}
 \label{fig:transfer_obs}
\end{widefigure}

\begin{widefigure}[tb]
\centering
\includegraphics[width=0.9\textwidth,keepaspectratio]{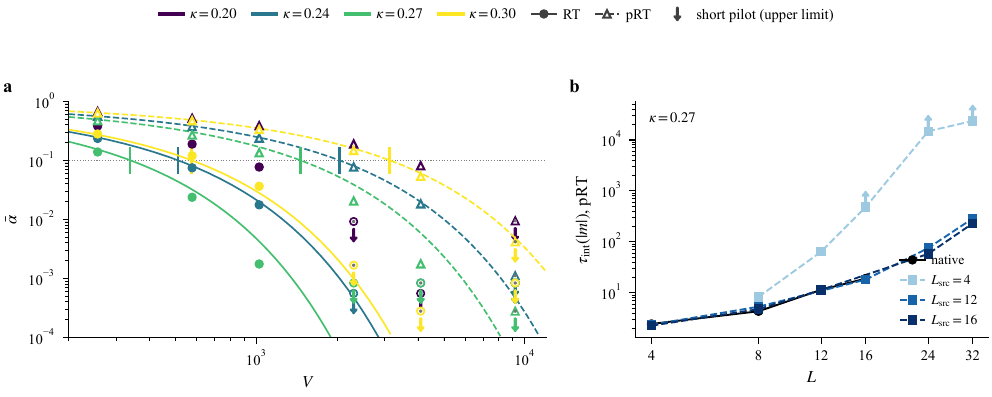}
\caption{Efficiency of frozen models after transfer. (a) Mean acceptance of frozen $L_{\rm src}=16$ models versus the volume $V$. Lines are the prediction $\operatorname{erfc}(\sqrt{\bar\sigma_s^2V}/2)$ from the per-site conditional work variance $\bar\sigma_s^2$ of the native $L=16$ models alone, solid for full and dashed for truncated round trips ($s=12$). Short vertical bars mark $V_{\bar\alpha_\ast}$ of Eq.~\eqref{eq:vmax} ($\bar\alpha_\ast=0.1$). Colors indicate $\kappa$, and $\kappa=0.20$ has measured points only. Points with downward arrows are upper limits from short pilot runs of $1000$ steps per chain, including $L=96$. (b) $\tau_{\rm int}(|m|)$ of truncated round trips at $\kappa=0.27$ versus the target size. Black is the native model, and blue from light to dark are models transferred from $L_{\rm src}=4,12,16$. Upward arrows mark lower bounds.}
\label{fig:transfer_eff}
\end{widefigure}

\subsection{Acceptance after transfer}
\label{sec:transfer_accept}

For independent proposals, the effective sample fraction of a fixed local model decreases exponentially with the volume after transfer, provided the lattice is much larger than the correlation length~\cite{Finkenrath:2022ogg,Abbott:2022zsh}. For round-trip updates the per-site conditional work variance plays this role. We predict the acceptance of frozen models from the per-site conditional work variance of the native $L=16$ model alone. We multiply it by the target volume and insert it in Eq.~\eqref{eq:erfc_law}, without fitted parameters. The prediction assumes that the per-site conditional work variance of a frozen model does not change with the lattice size. Section~\ref{sec:boundary_var} tests this assumption. Away from criticality, at $\kappa=0.24$ and $0.30$, we measure the acceptance on long chains, for full round trips at $L=24,32$ and for truncated round trips at $L=24$ to $64$. It differs from the prediction by at most $30\%$ in relative terms (Fig.~\ref{fig:transfer_eff}(a)). For full round trips at $\kappa=0.30$ this agreement comes from the cancellation of two opposite effects (Sec.~\ref{sec:boundary_var}). The acceptance of truncated round trips falls much more slowly with the volume. The ratio of the volume at which the measured acceptance falls to $0.1$ to the prediction of Eq.~\eqref{eq:vmax} is $0.92$ to $1.09$ at $\kappa=0.24$ and $0.30$. At $\kappa=0.27$ the acceptance after transfer falls below the prediction. The gap grows with size. At $L=32$ the measured acceptance of full round trips is about $0.4$ times the prediction. At $L=64$ that of truncated round trips is about $0.3$ times the prediction. Section~\ref{sec:boundary_var} discusses this deviation.
Away from criticality the autocorrelation times of full round trips transferred to $L=24$ and $32$ also follow Eq.~\eqref{eq:tau_law}. Those at $\kappa=0.27$ lie clearly above it. Supplemental Sec.~\ref{sm:transfer} gives these autocorrelation times.

\begin{widefigure}[tb]
\centering
\includegraphics[width=0.98\textwidth,keepaspectratio]{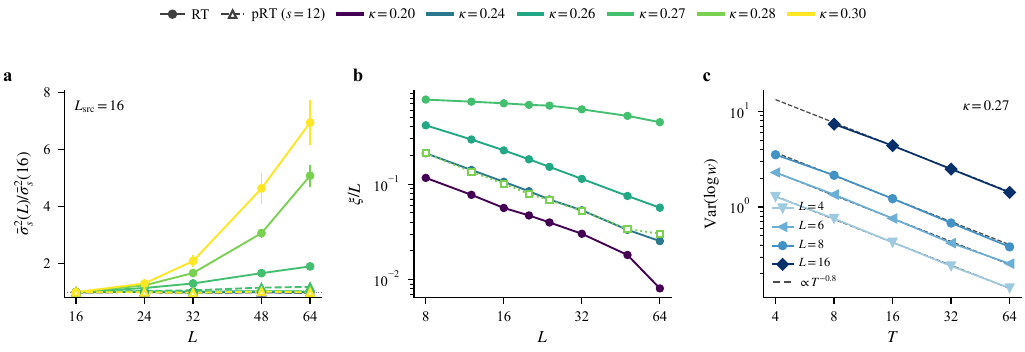}
\caption{Limits of transfer. (a) Per-site conditional work variance measured on forward paths for frozen $L_{\rm src}=16$ models on $L\times L$ lattices, relative to $L=16$, solid for full and dashed for truncated round trips ($s=12$). (b) Correlation length over lattice size. Solid lines are the second-moment correlation length $\xi_2$ in the symmetric phase and dotted lines the correlation length from the connected susceptibility in the broken phase. Correlation lengths are computed up to $\kappa=0.29$, and the broken phase is shown only for $\kappa=0.28$. (c) $\operatorname{Var}(\log w)$ of the whole path at $\kappa=0.27$ versus the number of stages $T$, with a separate model trained for each $(L,T)$ and symbol shapes indicating $L$. Dashed lines are $T^{-0.80}$ normalized to the $T=16$ point of each $L$. Colors in (a) and (b) indicate $\kappa$.}
\label{fig:boundary}
\end{widefigure}

\subsection{Size dependence of the per-site conditional work variance}
\label{sec:boundary_var}

We put the frozen $L_{\rm src}=16$ models directly on lattices with $L=16$ to $64$ and measure the conditional work variance at the shared intermediate field on forward paths (Fig.~\ref{fig:boundary}(a)). For $\kappa\le0.26$ the per-site variance of full and truncated round trips changes by at most $4\%$ from $L=16$ to $64$. At $\kappa=0.27$ the per-site variance of full round trips grows by a factor of about $1.9$ at $L=64$. That of truncated round trips grows by about $18\%$. In the broken phase the full round-trip variance grows faster, by factors of about $5$ and $7$ at $L=64$ for $\kappa=0.28$ and $0.30$. That of truncated round trips hardly changes.

Figure~\ref{fig:boundary}(b) gives the correlation lengths at the same couplings. In the symmetric phase we use the second-moment correlation length $\xi_2$, and in the broken phase the connected susceptibility with the mean magnetization subtracted (Appendix~\ref{app:implementation}). For $\kappa\le0.26$ the correlation length no longer changes with $L$ for $L\ge16$, with a maximum of about $3.6$ lattice spacings. When the lattice is much larger than the correlation length, the frozen model meets the same local fluctuations on larger lattices as in training. The per-site variance then stays the same. At $\kappa=0.27$, $\xi_2$ grows from about $11$ at $L=16$ to about $29$ at $L=64$. The ratio $\xi_2/L$ stays above $0.4$. Larger lattices contain long-wavelength fluctuations that are absent at the training size, and the per-site variance grows with them.

The broken phase is different. Here the connected correlation length is only one to two lattice spacings. It cannot explain the growth of the per-site variance. Measurements at $s=4$ and $8$ further show that the growth in the broken phase comes mainly from the first eight stages. At $\kappa=0.30$ the per-site variance no longer changes with the size from $s=8$ on. At $\kappa=0.28$ it still grows by about $40\%$ at $s=8$ on $L=48$ and $64$, and it stays constant at $s=12$ (Supplemental Sec.~\ref{sm:scaling}). The early stages set the size and sign of the magnetization. A frozen model learns the magnetization distribution of the early stages at $L=16$. The two peaks of the target distribution become narrower on larger lattices. The mismatch of the early stages therefore grows with the size and concentrates on a few intermediate fields. 

In the broken phase truncated round trips redraw only the late stages, so their per-site variance changes by at most $3\%$ from $L=16$ to $64$. Full round trips include the early stages, and their per-site variance grows with the size. At $\kappa=0.30$ it is already about $2.1$ times the $L=16$ value at $L=32$. Inserting the grown variance in Eq.~\eqref{eq:erfc_law} gives a full round-trip acceptance of only about $0.004$ at $L=32$, whereas long chains give $0.036$. The reason is that the work distribution at a fixed intermediate field departs from a Gaussian. At $L=32$ and $\kappa=0.30$ its excess kurtosis (the fourth cumulant divided by the squared variance) is about $50$. The $1\%$ of intermediate fields with the largest variance contribute about $15\%$ of the conditional work variance. The mean variance therefore overstates the drop of the acceptance. Taking intermediate fields from the reverse paths of chain samples and computing the round-trip acceptance directly gives $0.037$, in agreement with the measurement (Supplemental Sec.~\ref{sm:scaling}). The growth of the variance and the heavy tail act in opposite directions and roughly cancel at $L=32$. The prediction from $L=16$ alone therefore agrees with the measurement. For truncated round trips the prediction follows from the per-site variance, which does not change with the size.

\subsection{Autocorrelation time and source size}
\label{sec:transfer_dynamics}

The effect of transfer also depends on the source size. Figure~\ref{fig:transfer_eff}(b) shows $\tau_{\rm int}(|m|)$ of truncated round trips with $s=12$ at $\kappa=0.27$. It compares models transferred from $L_{\rm src}=4,12,16$ with native models trained directly at the target size. Models with $L_{\rm src}=12$ and $16$ come close to the native results at the same size and are nearly identical to each other. Increasing the source size from $12$ to $16$ therefore no longer changes the autocorrelation after transfer. When the source size is too small ($L_{\rm src}=4$), the path weights after transfer concentrate on a few paths. The autocorrelation time is then much longer. Results at other couplings are given in Supplemental Sec.~\ref{sm:transfer}.

\subsection{Number of stages and critical slowing down}

At a fixed number of stages $T$, Eqs.~\eqref{eq:erfc_law} and \eqref{eq:tau_law} give
$\tau_{\rm int}\approx \tfrac C2\sqrt{\pi\bar\sigma_s^2V}\,e^{\bar\sigma_s^2V/4}$ at larger volume, so the autocorrelation time grows exponentially with the volume (Supplemental Sec.~\ref{sm:scaling}). More stages push the limit further. At $\kappa=0.27$ we take $T=4$ to $64$ on $L=4,6,8$ and $T=8$ to $64$ on $L=16$, with a separate model trained for each $(L,T)$. The $\operatorname{Var}(\log w)$ of the whole path falls roughly as $T^{-0.80}$ (Fig.~\ref{fig:boundary}(c)). For full round trips, $\sigma_s^2$ and $\operatorname{Var}(\log w)$ differ only by the variance of the mean path work between different starting points $u_0$. The directly measured conditional work variance and $\operatorname{Var}(\log w)$ have similar per-site slopes (Appendix~\ref{app:scaling}). This part is therefore small, and we write
\begin{equation}
 \begin{split}
 \sigma_s^2&\simeq\frac dT\,\operatorname{Var}(\log w),\\
 \operatorname{Var}(\log w)&\approx\operatorname{Var}(\log w)\big|_{T=16}\,(T/16)^{-0.80},
 \end{split}
 \label{eq:sigma_scaling}
\end{equation}
where $\operatorname{Var}(\log w)$ at $T=16$ grows linearly with the volume with a finite-size constant (Appendix~\ref{app:scaling}). At large volume, keeping the acceptance of full round trips requires $T\propto V^{1/0.80}=V^{1.25}$. This relation is obtained at $\kappa=0.27$ for $L\le16$ and extrapolated to larger sizes. For annealed importance sampling, the number of steps grows linearly with the number of degrees of freedom under three conditions~\cite{Syed:2024oasmc}. The samples at every step follow the corresponding intermediate distribution exactly, the incremental weights of different steps are independent, and the target consists of independent copies. This scaling has been observed for nonequilibrium methods and stochastic normalizing flows on the lattice~\cite{Caselle:2022acb,Bonanno:2024udh,Bulgarelli:2024brv}. The growth found here is faster. Standard nonequilibrium HMC without learning (NE-HMC) approximately follows $\operatorname{Var}(\log w)\propto V/T$ at $\kappa=0.24$, in agreement with the literature. This relation fails near criticality and in the broken phase (Supplemental Sec.~\ref{sm:scaling}). Truncated round trips redraw only the last $d$ stages, so $\sigma_s^2$ is about $d/T$ of the full round-trip value. The number of stages they need grows more slowly. Near criticality, however, their autocorrelation is no longer set by the acceptance (Sec.~\ref{sec:tau_law}). On the native models with $L\le16$ their $\tau_{\rm stage}$ is still the shortest of the three updates (Sec.~\ref{sec:native_tau}).

\begin{widefigure}[tb]
\centering
\includegraphics[width=0.9\textwidth,keepaspectratio]{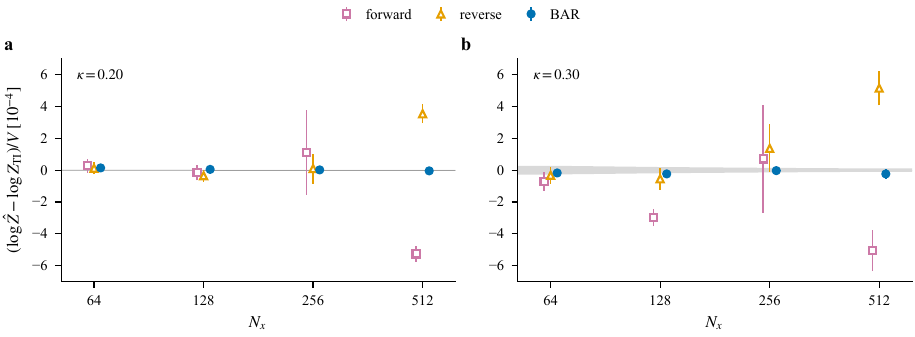}
\caption{Per-site residual of the log partition function $(\log\hat Z-\log Z_{\rm TI})/V$ for frozen $64\times8$ models transferred to $N_x\times8$. Filled circles are the Bennett acceptance ratio (BAR) estimate, open squares the one-sided forward estimate, and open triangles the one-sided reverse estimate. The gray band is the error of the TI reference and is narrower than the line width at $\kappa=0.20$. (a) $\kappa=0.20$. (b) $\kappa=0.30$.}
\label{fig:rect_transfer}
\end{widefigure}

The autocorrelation of local updates hardly depends on the volume when the correlation length is much smaller than the lattice. Near the critical point it grows as $\xi^{z_{\rm dyn}}$~\cite{Sokal_1997,Schaefer:2010hu}, where $z_{\rm dyn}$ is the dynamic critical exponent and $\xi$ is cut off by $L$ at finite volume. The autocorrelation of full round trips is instead set by the acceptance, and its volume dependence comes from the conditional work variance. The per-site conditional work variance is largest near the critical coupling. At fixed model and number of stages the autocorrelation grows exponentially with the volume. Keeping the acceptance then requires increasing the number of stages $T$ according to Eq.~\eqref{eq:sigma_scaling}. Effective exponents are given in Appendix~\ref{app:scaling}.

\section{Partition function and free energy}
\label{sec:free_energy}

\subsection{Square lattices}

Forward paths give an estimate of the partition function through Eq.~\eqref{eq:nhmc_exponential_identity}. Complete reverse records drawn from the target distribution give a reverse estimate of the same partition function. The Bennett acceptance ratio (BAR) combines the two sets of paths into a bidirectional estimate~\cite{Bennett_1976,Shirts_2003}, which is the optimal bridge-sampling form~\cite{Meng:1996ss}. Thermalized corrected chains supply the reverse records. BAR counts them as the number of effectively independent records given the chain autocorrelation (Appendix~\ref{app:bar_free_energy}). On the native square lattices, BAR from all three chains agrees with TI to within about $0.02$ in the total $\log Z$. The one-sided forward estimate deviates most at $L=16$ (Supplemental Sec.~\ref{sm:logz}).

\subsection{Transfer on rectangular lattices}
\label{sec:rectangular_free_energy}

Square lattices only reach $L=16$. To test the bidirectional estimate at larger volumes, we use periodic $N_x\times N_t$ lattices with fixed $N_t=8$. The volume $V=N_xN_t$ changes with $N_x$ and can reach $4096$ sites, while the TI reference remains easy to obtain. Near criticality the acceptance after transfer drops too quickly (Sec.~\ref{sec:transfer_accept}), so we take one coupling in each phase. We train one $64\times8$ model at each of $\kappa=0.20$ and $0.30$ and freeze its parameters. We apply it at larger $N_x$ and at each volume compare the per-site residual of the log partition function with respect to the TI reference, $(\log\widehat Z-\log Z_{\rm TI})/V$. The reverse records come from truncated round-trip chains ($s=12$). The acceptance of these chains falls with the volume to $0.08$ ($\kappa=0.20$) and $0.05$ ($\kappa=0.30$) at $512\times8$. This is consistent with the extrapolation from the acceptance at $128\times8$ with the scaling of Sec.~\ref{sec:transfer_accept}.

Figure~\ref{fig:rect_transfer} shows the results. At both couplings, BAR agrees with the reference at every size from $N_x=64$ to $512$, with per-site residuals of at most $0.25\times10^{-4}$. The one-sided forward and reverse exponential averages drift apart with the volume. At $512\times8$ the forward estimate lies clearly below and the reverse estimate clearly above the reference. The bidirectional estimate nevertheless remains accurate. BAR only needs the forward and reverse work distributions to overlap near $F$ (Appendix~\ref{app:bar_free_energy}). At these volumes the round-trip chains still supply new reverse records (Appendix~\ref{app:chain_free_energy}).

\section{Discussion and conclusions}
\label{sec:discussion_conclusions}

We use round-trip updates on the learned paths of NHMC. The current configuration walks back along a reverse path to an intermediate field, and a new configuration is generated from it. The acceptance is given by the probability ratio of the new and old paths. The target distribution is therefore preserved exactly for any model parameters, and shortcomings of the model only lower the acceptance. In two-dimensional $\phi^4$ theory the corrected chains give observables in agreement with the references at the training size. The mismatch between model and target makes chains of independent proposals stay in states of large weight for long periods. The autocorrelation time of round-trip updates is instead still set by the acceptance. Counted in stages, the autocorrelation time of full round trips near criticality is less than one eighth of that of independent proposals from the same model.

The efficiency of round-trip updates is set by the conditional work variance at the shared intermediate field. This variance can be measured on forward paths before any chain is run. The mean acceptance is approximately $\operatorname{erfc}(\sigma_s/2)$. The acceptance in turn sets the autocorrelation time of full round trips. Away from criticality the per-site variance of truncated round trips does not change with the lattice size. For $\kappa\le0.26$ the same holds for full round trips. The value measured at $L=16$ therefore predicts the acceptance of a frozen model on larger lattices. The measured acceptances agree with the prediction, and $92\%$ of the observable values agree with the references within two standard errors. A model trained on $64\times8$ and used on $512\times8$ gives accurate free energies through bidirectional estimates with the reverse records of round-trip chains. At these sizes the one-sided estimates fail.

The limit of transfer appears near criticality and for full round trips in the broken phase. Near criticality the correlation length grows with the lattice, and the per-site conditional work variance of a frozen model grows with it. In the broken phase the per-site variance of full round trips grows with the size. Its agreement with the prediction at the measured sizes comes from the cancellation by the heavy tail. The per-site variance of truncated round trips does not change with the size. One then needs more stages or retraining at the target size. The number of stages needed to keep the acceptance of full round trips grows roughly as $V^{1.25}$. For round-trip updates, critical slowing down thus appears as a requirement on path length.

We see three directions for further work. Round-trip chains give target samples for any model parameters. Our training already uses these samples as a buffer. This buffer could be used systematically to improve several models or to train small correction networks, also for other bidirectional path models. Equally important is a lower cost of the paths, through fewer stages at fixed acceptance and less computation per stage. A comparison with HMC at equal numbers of force evaluations on lines of constant $\xi/L$ should go with it. Round-trip updates can also be applied to gauge theories.

\section*{Data availability}
The code, trained models, plotting data and plotting scripts of this work will be made publicly available on Zenodo (DOI \href{https://doi.org/10.5281/zenodo.24163330}{10.5281/zenodo.24163330}) and GitHub (\url{https://github.com/qxxmax/rtnhmc}).\ A nontechnical overview is available at \url{https://qxxmax.github.io/assets/blog/nhmc/}.

\begin{acknowledgments}
We thank Gert Aarts, Elia Cellini, Shiyang Chen, Yuanqi Du, Gurtej Kanwar, Chuan Liu, Miaoxin Liu, Biagio Lucini, Alessandro Nada, Jingkang Ouyang, Lukas Stelzl, Hartmut Wittig, Jianhui Zhang, Rui Zhang, and Kai Zhou for useful discussions. During this work we became aware of the related concurrent work MAD-Path. MAD-Path builds Metropolis-adjusted round trips along diffusion paths~\cite{chen2026markovchainmontecarlo}.
Large language model assistants, GPT-5.6 Sol (OpenAI) and Claude Opus 5.5 (Anthropic), were used to optimize and organize analysis and job scripts built on our original code. Claude Opus 5.5 also helped polish the language of the manuscript.
We checked all code, data, conclusions and text and take full responsibility for them.
Computations were performed on the HIMster III high-performance computing cluster at the Helmholtz Institute Mainz (HIM) and on GPU resources of the Institute for Quantitative and Computational Biosciences (IQCB) at Johannes Gutenberg University Mainz.
We thank HIM and IQCB for providing computing resources.
\end{acknowledgments}

\appendix
\makeatletter
\renewenvironment{widefigure}[1][]{%
  \par\onecolumngrid\vskip10pt\noindent\begin{minipage}{\textwidth}%
  \def\@captype{figure}\centering
}{%
  \end{minipage}\par\vskip8pt\twocolumngrid\global\@ignoretrue\@endpetrue
}
\makeatother
\section{Path probabilities and full or truncated round trips}
\label{app:path_correction}

The same path probabilities are used for training, importance weighting,
and Metropolis updates. Below we specify their common measure and
show how a conditional update of the path record induces a reversible
update of the field. The construction combines extended-space path densities~\cite{pmlr-v37-salimans15,NEURIPS2020_41d80bfc}
with the involutive MCMC formulation~\cite{pmlr-v119-neklyudov20a}.

\subsection{Complete records and path weights}

For a continuous NHMC stage, the kick and drift Jacobians are
\begin{equation}
 \partial\mathsf K_\epsilon=
 \begin{pmatrix}I&0\\\epsilon\,\partial G_{\theta,t}(x)&I\end{pmatrix},\qquad
 \partial\mathsf D_\epsilon=
 \begin{pmatrix}I&\epsilon I\\0&I\end{pmatrix}.
 \label{eq:app_shear_jacobians}
\end{equation}
Both have unit determinant. Momentum reversal changes each substep to its
inverse, and the symmetric leapfrog composition therefore obeys
Eq.~\eqref{eq:nhmc_reversal}. This needs only a force that depends on the position alone and a step size that does not depend on the state. The force need not be conservative. A global
$\mathbb Z_2$ transformation has unit absolute Jacobian and is undone in
reverse order. The Bernoulli label of the sign flip is retained alongside the momenta.
The full correspondence $(u,a)\leftrightarrow(x,b)$ consequently preserves
the product of Lebesgue measure for continuous variables and counting
measure for discrete labels.

The normalized forward law and reverse reference law are
$Q_\theta=q_0(u)f_\theta(a\mid u)$ and
$P_\theta=\pi(x)r_\theta(b\mid x)$ on these common coordinates, in which the same variables describe a path generated forward from $u$ or backward from $x$.
Their density ratio gives Eq.~\eqref{eq:nhmc_pathratio}. In particular,
for an integrable endpoint observable $O$, if $P_\theta$ is absolutely continuous with respect to $Q_\theta$ ($P_\theta\ll Q_\theta$),
\begin{equation}
 \mathbb E_{Q_\theta}\!\left[e^{-W_\theta}O(x)\right]
 =Z\,\mathbb E_{P_\theta}[O(x)]
 =Z\langle O\rangle_\pi.
 \label{eq:app_weighted_observable}
\end{equation}
Setting $O=1$ gives the first identity of Eq.~\eqref{eq:nhmc_exponential_identity}. The mean-work expression for
$D_{\mathrm{KL}}(Q_\theta\Vert P_\theta)$ additionally requires $Q_\theta\ll P_\theta$ and integrability of $W_\theta$. The
training loss also requires a finite second moment of the work when the
variance penalty is used.

For a continuous stage without an additional discrete transformation,
let $q_t$ be the forward-process marginal of $x_t$ ($q_0$ at $t=0$). The forward
map induces an exact conditional density of the reversed output momentum,
$ q^{{\rm R},*}_{\theta,t}$, satisfying
\begin{equation}
 \begin{aligned}
 q_t(x)q^{\rm F}_{\theta,t}(p\mid x)
 &=q_{t+1}(y)q^{{\rm R},*}_{\theta,t}(-\bar p\mid y),\\
 (y,\bar p)&=\Phi_{\theta,t}(x,p).
 \end{aligned}
 \label{eq:exact_reverse_continuous}
\end{equation}
This is a change-of-variables identity for the actual forward process.
The learned reverse law need not equal this conditional density. The
weights and acceptance probabilities use the learned density that is
actually evaluated, including any discrete-label probabilities.

\subsection{From the conditional record to a field update}

Fix a return stage $s$. In this subsection we drop the subscript and write $a,b$ for the suffix records $a_{\ge s},b_{\ge s}$ on the redrawn stages $s,\ldots,T-1$. A reverse record $b\sim r_s(\cdot\mid x)$
determines an intermediate field $u$ and the corresponding forward
record $a$. Conversely, $(u,a)$ determines $(x,b)$ through a measurable,
invertible correspondence $\Psi_s$ on the common support. In continuous
coordinates write
$J_s(u,a)=|\det\partial\Psi_s(u,a)|$. Discrete labels again use counting measure.
Under $P_\theta$, the joint law of $(u,a)$ has density
\begin{equation}
 \pi_s(u)h_s(a\mid u)
 =\pi(x(u,a))\,r_s(b(u,a)\mid x(u,a))\,J_s(u,a).
 \label{eq:rt_target_conditional}
\end{equation}
Here $\pi_s$ is the stage-$s$ intermediate-field marginal of $P_\theta$ ($\pi_T=\pi$) and $h_s$ the conditional
law of the old record. Neither requires explicit evaluation. A new record
$a'\sim f_s(\cdot\mid u)$ determines $y=x(u,a')$. Conditional on $u$,
independence Metropolis--Hastings uses
\begin{equation}
 \alpha_s(u,a,a')=\min\!\left\{1,
 \frac{h_s(a'\mid u)f_s(a\mid u)}
      {h_s(a\mid u)f_s(a'\mid u)}\right\}.
 \label{eq:rt_general_acceptance}
\end{equation}
Substitution of Eq.~\eqref{eq:rt_target_conditional} cancels both $Z$ and
$\pi_s(u)$. For $J_s=1$ this is Eq.~\eqref{eq:nhmc_rt_acceptance}, while a
non-unit Jacobian contributes the ratio $J_s(u,a')/J_s(u,a)$.

The accepted equilibrium density of the pair of records is
\begin{equation}
 \begin{split}
 \mathcal M_s(u,a,a')={}&\pi_s(u)
 \min\!\bigl\{h_s(a\mid u)f_s(a'\mid u),\\
 &\qquad\quad h_s(a'\mid u)f_s(a\mid u)\bigr\}.
 \end{split}
 \label{eq:rt_symmetric_record_flow}
\end{equation}
It is symmetric under $a\leftrightarrow a'$. To see that this is a
statement about configurations, take any two measurable field regions
$\mathcal X_1$ and $\mathcal X_2$. The accepted flow is the equilibrium probability that one update moves from $\mathcal X_1$ to $\mathcal X_2$ and is accepted. It is
\begin{equation}
 \begin{aligned}
 &\int_{\mathcal X_1}\pi(dx)K_s^{\mathrm{acc}}(x,\mathcal X_2)\\
 &\quad=\int du\,da\,da'\,
 \mathbf1_{\mathcal X_1}(x(u,a))\mathbf1_{\mathcal X_2}(x(u,a'))\\
 &\qquad\qquad{}\times\mathcal M_s(u,a,a').
 \end{aligned}
 \label{eq:rt_projected_balance}
\end{equation}
Exchanging $a$ and $a'$ interchanges $\mathcal X_1$ and $\mathcal X_2$ without changing the
integral. Rejection retains $x$ and contributes only a diagonal term.
Thus the complete field kernel satisfies
\begin{equation}
 \pi(dx)K_s(x,dy)=\pi(dy)K_s(y,dx),\qquad \pi K_s=\pi.
 \label{eq:rt_endpoint_balance}
\end{equation}
The reverse record is redrawn from the current field at the next update.
It is not a retained state of the chain. Mixing additionally depends on
support and on communication between field
regions~\cite{pmlr-v119-neklyudov20a}.

\subsection{Truncation and other bidirectional paths}

For NHMC, $J_s=1$ for the full path and every fixed suffix. Stages $s,\ldots,T-1$ keep their original parameters and time indices.
Truncation therefore changes the update without changing its stationary measure.
The proof uses the shared intermediate field and does not assume that
this field follows $q_0$. The old and new paths start from the same
intermediate field, so the starting density cancels in the RT acceptance.
The truncated path work does not contain the marginal density of the
intermediate field, so its exponential average cannot estimate $Z$.

The same argument applies to a stochastic configuration-space path when
its complete states and random choices have evaluable forward and reverse
probabilities. In that case the old and new paths are compared with the
same intermediate state held fixed, and conditional-record balance,
Eq.~\eqref{eq:rt_symmetric_record_flow}, is unchanged. The round-trip
proof needs only the path correspondence. Hamiltonian dynamics is one
realization. Throughout the comparisons in this
paper the return stage and learned parameters are fixed during sampling.

\subsection{Acceptance and autocorrelation}
\label{app:acceptance}

At a fixed intermediate field the old and new records are independent, each with a conditional work variance of about $\sigma_s^2$, so $\operatorname{Var}\Delta\approx2\sigma_s^2$. If the work difference $\Delta$ is approximately Gaussian, the Crooks relation $\mathcal P(\Delta)=e^{\Delta}\mathcal P(-\Delta)$ requires its mean to be half its variance, $\langle\Delta\rangle=\operatorname{Var}\Delta/2\approx\sigma_s^2$. This gives $\mathbb E[\min(1,e^{-\Delta})]=\operatorname{erfc}(\sigma_s/2)$, as in the Gaussian analysis of the HMC acceptance~\cite{Gupta:1988js}. If each accepted update decorrelates the observable and the stay lengths are geometric, the integrated autocorrelation time is $1/\bar\alpha-1/2$. The old and new configurations of a round trip share the intermediate field, so one acceptance need not fully decorrelate the observable. We therefore measure the decorrelation by the ratio $\tau_{\rm int}/(1/\bar\alpha-1/2)$, which equals one for full decorrelation. For full round trips this ratio is close to a constant $C$, Eq.~\eqref{eq:tau_law}. The ratio is clearly above one for small return depth $d$ (Sec.~\ref{sec:scaling} and Appendix~\ref{app:other_samplers}).

\section{Learned paths, numerical setup and equilibrium observables}
\label{app:implementation}

The main lattice calculation uses the two-dimensional periodic
$\phi^4$ action in Eq.~\eqref{eq:phi4_action} with zero external field.
For the main square-lattice comparison we train a separate NHMC model at each $(L,\kappa)$.
At each target point, independent path generation, self-normalized importance sampling (SNIS), path-IMH, full RT and pRT use the same trained parameters.
Table~\ref{tab:settings} collects the numerical settings.

\begin{table}[tb]
\centering\small
\caption{Numerical settings. $g=2\tau_{\rm int}$, $M$ is the Sokal window, and $n_{\rm post}$ is the number of records kept after thermalization.}
\label{tab:settings}
\renewcommand{\arraystretch}{1.2}
\begin{tabular}{lll}\toprule
Group & Item & \parbox[t]{0.36\linewidth}{\raggedright Value}\\ \midrule
Target & Couplings $\kappa$ & \parbox[t]{0.36\linewidth}{\raggedright $0.20$–$0.30$ in steps of $0.01$, plus $0.265$ and $0.275$}\\
 & Lattice sizes & \parbox[t]{0.36\linewidth}{\raggedright $L=4,8,12,16$, plus $L=6,10$ at the $11$ couplings spaced by $0.01$}\\
Path & Stages $T$ & \parbox[t]{0.36\linewidth}{\raggedright $16$}\\
 & Substeps per stage & \parbox[t]{0.36\linewidth}{\raggedright $5$}\\
 & Return stage of pRT & \parbox[t]{0.36\linewidth}{\raggedright $s=12$ ($d=4$)}\\
 & Flip probability $p_{\rm flip}$ & \parbox[t]{0.36\linewidth}{\raggedright $1/2$}\\
Networks & Architecture & \parbox[t]{0.36\linewidth}{\raggedright periodic $3\times3$ convolutions, two hidden layers of width $32$}\\
 & Norm., activation & \parbox[t]{0.36\linewidth}{\raggedright group normalization, SiLU}\\
Training & Variance penalty $\nu$ & \parbox[t]{0.36\linewidth}{\raggedright $0.01$ (forward and reverse terms)}\\
 & Epochs, batch size & \parbox[t]{0.36\linewidth}{\raggedright $20000$ epochs, $48$ paths per batch}\\
 & Optimizer & \parbox[t]{0.36\linewidth}{\raggedright Adam~\cite{Kingma:2014vow}, learning rate $10^{-3}\to10^{-5}$ (cosine)}\\
 & Gradient-norm clipping & \parbox[t]{0.36\linewidth}{\raggedright $5$}\\
RT buffer & First epoch & \parbox[t]{0.36\linewidth}{\raggedright $1000$}\\
 & Chains, updates & \parbox[t]{0.36\linewidth}{\raggedright $64$ RT chains, which first make $20$ RT updates and are then updated once every two training epochs}\\
 & Capacity & \parbox[t]{0.36\linewidth}{\raggedright latest $20000$ configurations}\\
Chains, errors & Chains per point & \parbox[t]{0.36\linewidth}{\raggedright $4$}\\
 & Thermalization & \parbox[t]{0.36\linewidth}{\raggedright first $10\%$ discarded}\\
 & Observable errors & \parbox[t]{0.36\linewidth}{\raggedright jackknife with $20$ blocks per chain}\\
 & $\tau_{\rm int}$ estimate & \parbox[t]{0.36\linewidth}{\raggedright larger of Sokal window ($M\ge5g$) and batch-means estimate with $50$ batches}\\
 & Convergence criteria & \parbox[t]{0.36\linewidth}{\raggedright $n_{\rm post}\ge1000\,g$, $n_{\rm post}\ge20\times$ longest rejection run, $M<n_{\rm post}/50$}\\
References & HMC & \parbox[t]{0.36\linewidth}{\raggedright trajectory length $1$, $10$ substeps of size $0.1$}\\
 & Wolff & \parbox[t]{0.36\linewidth}{\raggedright one local Metropolis sweep and two single-cluster flips}\\
 & Records & \parbox[t]{0.36\linewidth}{\raggedright $10^6$ per coupling for $|m|$, $2\times10^5$ for $\operatorname{sign}m$}\\
\bottomrule\end{tabular}
\end{table}

The initial field follows a normalized Gaussian distribution whose width is trained.
Four networks, shared by all stages, give the forward and reverse means and log widths of the conditional Gaussian momentum laws.
A fifth network gives the force residual used in the reversible leapfrog integration (Fig.~\ref{fig:network}).
\begin{figure}[tb]
\centering
\includegraphics[width=\linewidth,keepaspectratio]{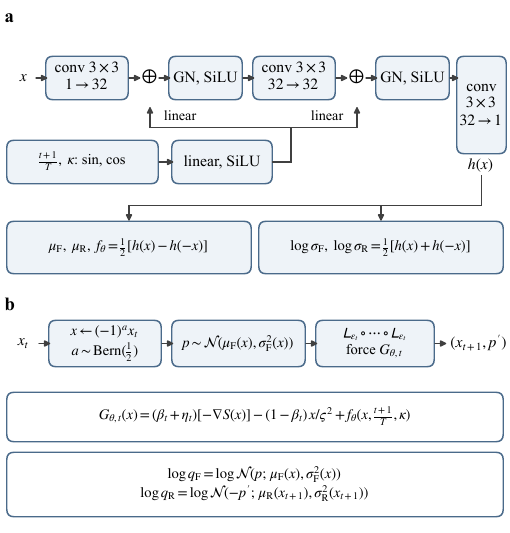}
\caption{Networks and one stage of NHMC. (a) Each of the five networks maps the field $x$ to $h(x)$ with periodic $3\times3$ convolutions, group normalization (GN) and SiLU. The stage position and the coupling enter through a sinusoidal embedding, which is added to every channel before each normalization. The odd part of $h$ gives the momentum means and the force residual, and the even part gives the log widths. (b) One stage $t$ with the evolution force of Eq.~\eqref{eq:force_param} and the recorded forward and reverse momentum densities. $\mathsf L_{\epsilon_t}$ is one leapfrog substep.}
\label{fig:network}
\end{figure}
The evolution force combines the analytic force from the target action, the force from the Gaussian initial distribution and the network residual.
The coefficient of the analytic target force, the coefficient of the Gaussian force at each stage and the step size at each stage are also trained.
The evolution force of stage $t$ is
\begin{equation}
 G_{\theta,t}(x)=(\beta_t+\eta_t)\,[-\nabla S(x)]-(1-\beta_t)\,\frac{x}{\varsigma^2}+f_\theta(x,(t+1)/T,\kappa),
 \label{eq:force_param}
\end{equation}
where $t=0,\ldots,T-1$, $\beta_t=\sum_{j=0}^{t}\operatorname{softmax}(\boldsymbol a)_j$ is given by a learnable vector $\boldsymbol a$ of length $T$, increases monotonically with the stage and reaches $\beta_{T-1}=1$ in the last stage, $\eta_t$ is an additional coefficient of the analytic force at each stage (initialized at zero), $\varsigma$ is the width of the initial distribution $q_0=\mathcal N(0,\varsigma^2I)$, and $f_\theta$ is the force-residual network (with zero initial output). The substep length of stage $t$ is $\epsilon_t=e^{\ell_t}$, where $\ell_t$ is also learnable.
Stage position and the coupling are conditioning inputs. For the models trained separately at each $(L,\kappa)$ the coupling is fixed during training.
The force residual and forward and reverse momentum means are odd functions of the field, while the momentum widths are even functions. For a network output $\psi(x)$ the odd part is $[\psi(x)-\psi(-x)]/2$ and the even part $[\psi(x)+\psi(-x)]/2$, and the log widths are made even.
Global $\mathbb Z_2$ flips are recorded as Bernoulli labels of the sign flips in the complete forward and reverse path probabilities.

The forward-path training term is given by Eq.~\eqref{eq:nhmc_loss}.
We optimize it by differentiating through the reparameterized random variables along the path.
From epoch $1000$ on we also maintain a configuration buffer generated by RT chains.
These chains start from forward endpoints of the model being trained.
Configurations drawn from the buffer provide starting points for complete reverse paths.
For these reverse paths we compute the work with Eq.~\eqref{eq:nhmc_work} as well, treating each path as starting at $u$ and ending at $x$.
We add the negative mean of these work values and the work-variance penalty to the loss, with unit relative weight for the reverse term.
The buffer configurations are held fixed during differentiation, while gradients propagate through the reverse records generated at that training epoch and their probabilities.
Training and sampling times are listed in Supplemental Sec.~\ref{app:timing}.

All learned parameters remain fixed during production sampling.
Full and truncated round trips reuse the same networks and learned stage indices.
The baseline truncated round trip in the main text retains only the last four stages, with return stage $s=12$.
The return-stage scan uses chains generated by the same model at different $s$.

The square-lattice reference means are obtained from independent HMC and Wolff calculations and combined using inverse-variance weights.
Standardized deviations combine the standard errors of both sides.
Autocorrelation calculations retain the repeated states after each rejection.
The HMC values of $\tau_{\rm int}(|m|)$ at $\kappa=0.27$ agree with Ref.~\cite{Singha:2026slc}. At $L=8$ and $16$ we find $23.8$ and $66.9$, and Ref.~\cite{Singha:2026slc} gives $22.9$ and $64.6$ at $\kappa_c=0.2705$.

\textbf{Correlation length.} We measure the correlation lengths of Sec.~\ref{sec:boundary} in separate Wolff simulations, where each update is again one local Metropolis sweep and two single-cluster flips. The lattices are $L=8$ to $64$ and the couplings $0.20\le\kappa\le0.29$ in steps of $0.0025$. Each parameter point has $10^6$ measurements ($5\times10^5$ for $L\ge48$) after $2\times10^4$ thermalization updates. The structure function is $\mathcal F=(|\tilde\phi(\boldsymbol p_1)|^2+|\tilde\phi(\boldsymbol p_2)|^2)/(2V)$ with $\tilde\phi(\boldsymbol p)=\sum_{\boldsymbol r}\phi_{\boldsymbol r}e^{i\boldsymbol p\cdot\boldsymbol r}$, $\boldsymbol p_1=(2\pi/L,0)$ and $\boldsymbol p_2=(0,2\pi/L)$. The second-moment correlation length is $\xi=[\chi_F/\mathcal F-1]^{1/2}/[2\sin(\pi/L)]$. In the symmetric phase we take $\chi_F=V\langle m^2\rangle$ and write $\xi_2$, and in the broken phase we take the connected susceptibility $\chi$ of Sec.~\ref{sec:native}, with the mean magnetization subtracted. Errors come from a jackknife over $50$ blocks.

\textbf{Autocorrelation estimates and convergence criteria.} Both estimates target $g=2\tau_{\mathrm{int}}$. The Sokal automatic-window estimate~\cite{Madras:1988ei} takes $g=1+2\sum_{j=1}^{M}\rho_O(j)$ at the smallest $M$ with $M\ge5g$. The batch-means estimate uses the variance of batch means~\cite{Flyvbjerg_1989}. Half the larger of the two gives the chain value of $\tau_{\mathrm{int}}$.
The convergence criteria must hold for both $|m|$ and $\operatorname{sign}m$. The condition on the longest rejection run requires the chain to be much longer than its longest stay. When any chain fails one condition, the quoted $\tau_{\mathrm{int}}$ is a lower bound. At $L=16$, $\kappa=0.26$, $0.27$ and $0.275$, the longest run of rejections of the path-IMH chains grows with chain length, and we also report their $\tau_{\mathrm{int}}$ as lower bounds.

\textbf{Means and error bars.} Each NHMC value of $\tau_{\mathrm{int}}$ in the figures and tables is the mean of the four chain values. Its error bar is the standard error of the mean, the sample standard deviation of the four chain values divided by $2$. Lower bounds carry no error bar. For HMC and Wolff we run one long reference chain per coupling and split it into four equal segments. We estimate $\tau_{\mathrm{int}}$ on each segment, plot the median, and take the jackknife standard error over the four segments as the error bar. Errors of acceptances and observables come from a block jackknife. We cut each chain into $20$ blocks and drop one block at a time.

\section{Numerical tests of the scaling relations and of transfer}
\label{app:scaling}

\textbf{Measuring the conditional work variance.} For each native model we generate intermediate fields $u_s$ from forward paths, continue $16$ independent paths from each $u_s$ to the endpoint, and take the mean over $u_s$ of the work variance at fixed $u_s$ as $\sigma_s^2$ ($16384$ paths in total, $s=0,2,4,8,12,14$ at $T=16$). For the frozen models on larger lattices (Table~\ref{tab:condvar_stages}) we use $4096$ paths. The truncated work $W_{\theta,s}$ does not contain the density of the intermediate field, and its variance between different intermediate fields is much larger than at a fixed intermediate field. The unconditional variance of $W_{\theta,s}$ therefore cannot replace $\sigma_s^2$. The work of the whole path also contains $\log q_0(u_0)$, which removes the differences between starting fields. For full round trips, $\sigma_0^2$ and $\operatorname{Var}(\log w)$ are therefore close. At $\kappa=0.27$ they are $1.16$ and $1.23$ at $L=8$, and $4.09$ and $4.37$ at $L=16$. At $\kappa=0.27$ the three per-site slopes are $0.0153$ from the direct measurement ($L=8$ to $16$), $\bar\sigma_\alpha^2=0.0165$, and the per-site slope of $\operatorname{Var}(\log w)$, $\bar\sigma_w^2=0.0165$ (Table~\ref{tab:fits}).

\begin{table*}[tb]
\centering\small
\caption{Fits of the scaling relations and the autocorrelation. Residuals are differences of acceptances or of logarithms. The Result column names the statistic. $A_w$ and $c_w$ are fitted constants, $V_e$ is a fitted volume scale, and $\bar\sigma_w^2$ is the per-site slope of $\operatorname{Var}(\log w)$.}
\label{tab:fits}
\renewcommand{\arraystretch}{1.2}
\begin{tabular}{llll}\toprule
Quantity & Form & Data & \parbox[t]{0.28\textwidth}{\raggedright Result}\\ \midrule
Full round-trip acceptance & $\operatorname{erfc}(\sqrt{\bar\sigma_\alpha^2V+c_\alpha}/2)$ & $13$ couplings, $L=4$–$16$ & \parbox[t]{0.28\textwidth}{\raggedright median of per-coupling maximum residuals $0.009$, overall maximum $0.024$}\\
 & $\propto e^{-V/V_e}$ (control) & $13$ couplings, $L=4$–$16$ & \parbox[t]{0.28\textwidth}{\raggedright median residual $0.032$}\\
$\operatorname{Var}(\log w)$ & $A_w\,V^{0.74}T^{-0.80}$ & $\kappa=0.27$, $L=4$–$8$, $T=4$–$64$ & \parbox[t]{0.28\textwidth}{\raggedright rms log residual $2.4\%$, $6\%$ with $L=16$}\\
 & $(\bar\sigma_w^2V+c_w)(T/16)^{-0.80}$ & $\kappa=0.27$, $L=4$–$16$ & \parbox[t]{0.28\textwidth}{\raggedright $\bar\sigma_w^2=0.0165$, $c_w=0.16$, deviation at $T=4$ of $5\%$ at $L=8$ and smaller at $L=4,6$}\\
$\tau_{\rm int}(|m|)$ and $\bar\alpha$ & $C(1/\bar\alpha-1/2)$ & $74$ native $(L,\kappa)$ values & \parbox[t]{0.28\textwidth}{\raggedright $C=1.19\pm0.01$ from a bootstrap over the points, rms log residual $16\%$}\\
\bottomrule\end{tabular}
\end{table*}

\textbf{Deviations of the acceptance prediction.} For path-IMH, the acceptance can be computed exactly from the path weights as $\mathbb E[\min(w,w')]/\mathbb E[w]$. Across the $41$ parameter points with long path-IMH chains (all $13$ couplings at $L=8$, $12$, $16$ and two couplings at $L=6$), this value differs from the measurement by $0.004$ in the median. The Gaussian prediction $\operatorname{erfc}(\sqrt{\operatorname{Var}(\log w)}/2)$ for the same paths differs by $0.009$ in the median. For round trips we compare the Gaussian prediction with a direct estimate. The direct estimate averages $\min(1,e^{-\Delta})$ over the same continuation paths used to measure $\sigma_s^2$, without the Gaussian approximation. At $L=8$ and $12$ there are eight comparisons, four models ($L=8$ at $\kappa=0.24,0.27,0.30$ and $L=12$ at $\kappa=0.27$) times two round-trip types. In all eight the direct estimate is closer to the measurement. At $L=16$ there are six comparisons, three couplings times two round-trip types. There the direct estimate is closer in two of the six.

\textbf{Fits.} The per-coupling erfc form describes the acceptance better than a pure exponential (Table~\ref{tab:fits}). For $L\le16$, $\bar\sigma_\alpha^2$ and $c_\alpha$ are strongly correlated. The error bars of $\bar\sigma_\alpha^2$ in Fig.~\ref{fig:predict}(d) add two parts in quadrature. One is the statistical error of the acceptances propagated to the slope. The other is the spread of the slope when the fit is repeated with one lattice size left out.

\textbf{Number of stages and volume.} At $\kappa=0.27$, the exponent of $\operatorname{Var}(\log w)$ in the stage number is the same at all sizes, about $-0.80$ (Fig.~\ref{fig:stage_scan}). We train models with $T=4$ to $64$ at $L=4,6,8$ and with $T=8$ to $64$ at $L=16$. On small volumes the effective volume exponent is below one because of a finite-size constant. We therefore use the form $(\bar\sigma_w^2V+c_w)(T/16)^{-0.80}$ with a constant (Table~\ref{tab:fits}). With this form, where the constant dominates, the variance is about $c_w(T/16)^{-0.80}$ and depends on $T$ only. At small volume the data are also described by $A_wV^{0.74}T^{-0.80}$ (Table~\ref{tab:fits}). The exponents of $V$ and $T$ in this form are similar, so the points for different $L$ fall approximately on one curve in $T/V$. At larger volume the variance is proportional to $V$ (Sec.~\ref{sec:scaling}).

What matters is the number of momentum refreshes, that is the number of stages, not the substeps per stage. Raising the substeps per stage from $5$ to $10$ leaves the variance unchanged for models trained at $L=8$ and $16$. Refining the substeps only at sampling time does not lower the variance either. When the substeps of the same model are raised to $10$ and $20$, the step size shrinks so that the integration time per stage is unchanged. $\operatorname{Var}(\log w)$ then rises for all six models at $L=8,16$ and $\kappa=0.24,0.27,0.30$, by at most $5\%$.

For $\operatorname{Var}(\log w)$ above about one, the measured $\mathrm{ESS}_w/N$ falls well below the Gaussian prediction $e^{-\operatorname{Var}(\log w)}$. At $L=8$ and $T=4$ the two are $0.001$ and $0.030$. We therefore use the conditional work variance and the acceptance as the primary measures.

\begin{figure}[tb]
\centering
\includegraphics[width=0.98\linewidth,keepaspectratio]{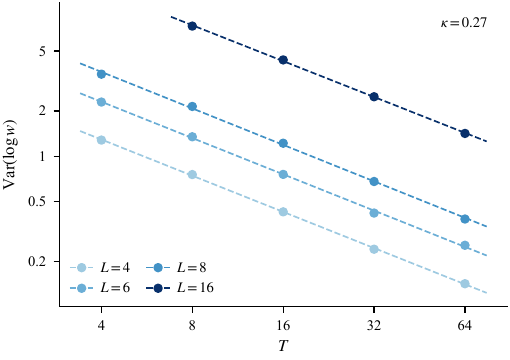}
\caption{$\operatorname{Var}(\log w)$ of separately trained $T$-stage models at $\kappa=0.27$, $L=4,6,8,16$. Dashed lines are power-law fits per $L$, with $T$ exponents between $-0.79$ and $-0.81$ (errors at most $0.03$).}
\label{fig:stage_scan}
\end{figure}

\textbf{Growth of the autocorrelation with volume and coupling.} The full round trips of the native models at all $74$ values of $(L,\kappa)$ follow $\tau_{\mathrm{int}}(|m|)\approx C(1/\bar\alpha-1/2)$. These values cover $L=4$ to $16$ at $T=16$, with $L=6,10$ covering only the $11$ couplings spaced by $0.01$. Here $\bar\alpha$ is the measured mean acceptance, and $C$ is the exponential of the median of the log ratios $\tau_{\rm int}/(1/\bar\alpha-1/2)$ (Table~\ref{tab:fits}, Fig.~\ref{fig:tau_fit}).

The per-site coefficient $\bar\sigma_\alpha^2(\kappa)$ peaks near criticality and is smaller in both phases (Fig.~\ref{fig:predict}(d)). The constant $c_\alpha$ is clearly larger in the broken phase than in the symmetric phase. Chains transferred to $L=24,32$ follow the same relation with their measured acceptance at most points. The exception is $\kappa=0.27$ for the models transferred to $L=24$ and $32$, where the measurement is at least about $2.4$ times the prediction.

At the measured sizes the HMC $\tau_{\mathrm{int}}(|m|)$ grows markedly with volume near criticality and hardly changes at $\kappa=0.24$ (Table~\ref{tab:tau_per_size}).

An accepted full round trip nearly decorrelates the chain. At $\kappa=0.24,0.27,0.30$ and $L=4$ to $16$ the mean absolute change in $|m|$ is $1.02$ to $1.13$ standard deviations, close to the value $2/\sqrt\pi\approx1.13$ for two independent standard normal draws. Truncated round trips ($s=12$) are accepted $1.2$ to $3.4$ times as often as full round trips. At $\kappa=0.27$, however, their $\tau_{\rm int}(|m|)$ is $1.9$ to $2.6$ times that of full round trips (Table~\ref{tab:tau_per_size}).

\begin{figure}[tb]
\centering
\includegraphics[width=0.98\linewidth,keepaspectratio]{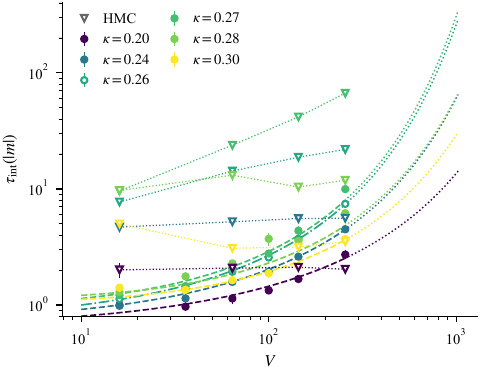}
\caption{Autocorrelation of full round trips versus volume and coupling (native models, $T=16$). $\tau_{\mathrm{int}}(|m|)$ versus $V$ for 6 of the 13 couplings. Points are measurements, dashed lines $C[1/\bar\alpha(V)-1/2]$ with $\bar\alpha(V)$ the erfc fit at that coupling, and thin dotted lines the extrapolation to $V>256$. The coupling $\kappa=0.26$ is drawn with open symbols and a dash-dotted line, and open down-triangles in the same color are HMC at that coupling ($L=4$ to $16$), joined by dotted lines. Statistics as in Appendix~\ref{app:implementation}.}
\label{fig:tau_fit}
\end{figure}

\textbf{Controls.} The control without learning is standard nonequilibrium HMC with standard normal momenta, no learned force, and linear or power-law annealing from the initial Gaussian to the target. We run it at $L=8,16$ for $\kappa=0.24,0.27,0.30$ and at $L=12$ for $\kappa=0.27$. We run two scans. One scans the step size ($T=8$ to $256$). The other is a grid over the initial width, annealing exponent, leapfrog substeps and step size ($T=8$ to $128$, $L=8,16$). The numbers below are the smallest variance of the two scans at each $T$ (only the step-size scan at $T=256$). At $\kappa=0.24$ the variance follows $\operatorname{Var}(\log w)\propto V/T$. Each doubling of $T$ multiplies it by $0.50$ to $0.61$, and the ratio between $L=16$ and $L=8$ is about $4$. Near criticality this scaling fails. At $\kappa=0.27$ and $L=8$ the variance falls only by a factor of $2.9$ from $T=16$ to $T=256$. At $\kappa=0.30$ it falls by only about $30\%$ from $T=8$ to $T=128$ at both $L=8$ and $16$. At the same coupling the NHMC variance at $T=16$ is $0.92$ at $L=8$ and $2.71$ at $L=16$, which is $7$ to $9$ times smaller than the NE-HMC value at the same $T$. Table~\ref{tab:tau_per_size} lists the integrated autocorrelation times of HMC, Wolff and the three path updates with a separately trained model at each size.

\begin{table*}[!tp]
\centering\small
\caption{Integrated autocorrelation times $\tau_{\rm int}$ with a separately trained model at each size. For each size the two columns give $|m|$ and $\operatorname{sign}m$ in units of one complete update (defined in Eq.~\eqref{eq:nhmc_corr_convention}). HMC and Wolff values for $|m|$ are medians over four equal segments of reference chains with $10^6$ records per point, those for $\operatorname{sign}m$ from reference chains that recorded the signed magnetization with $2\times10^5$ records per point. NHMC values are means over four chains. $\ge$ marks lower bounds, and ``no flip'' means that $\operatorname{sign}m$ did not change within the record. One update computes $-\nabla S$ about $10$ times for HMC, $6T=96$ times for path-IMH, $12T=192$ times for full round trips, and $12d=48$ times for truncated round trips, and NHMC updates need the corresponding network evaluations in addition (Supplemental Sec.~\ref{app:timing}).}
\label{tab:tau_per_size}\begin{tabular}{llrr@{\qquad}rr@{\qquad}rr@{\qquad}rr}\toprule
$\kappa$ & update & \multicolumn{2}{c}{$L=4$} & \multicolumn{2}{c}{$L=8$} & \multicolumn{2}{c}{$L=12$} & \multicolumn{2}{c}{$L=16$} \\
\cmidrule(lr){3-4}\cmidrule(lr){5-6}\cmidrule(lr){7-8}\cmidrule(lr){9-10}
 & & $|m|$ & sign & $|m|$ & sign & $|m|$ & sign & $|m|$ & sign \\ \midrule
0.24 & HMC & 4.7 & 8.0 & 5.2 & 9.5 & 5.6 & 9.9 & 5.6 & 13.2 \\
 & Wolff & 6.6 & 0.8 & 3.0 & 1.0 & 2.6 & 1.6 & 2.4 & 3.0 \\
 & RT & 1.0 & 0.9 & 1.6 & 1.3 & 2.6 & 2.2 & 4.5 & 4.2 \\
 & pRT & 1.2 & 0.7 & 1.6 & 1.0 & 1.9 & 1.0 & 2.7 & 1.3 \\
 & path-IMH & 1.2 & 1.0 & 2.4 & 2.1 & 5.9 & 6.5 & 17.9 & 15.8 \\
\midrule
0.27 & HMC & 9.6 & 58.0 & 23.8 & $\ge$194 & 41.7 & $\ge$452 & 66.9 & $\ge$1026 \\
 & Wolff & 20.4 & 4.0 & 23.5 & 1.5 & 30.4 & 0.9 & 31.1 & 0.8 \\
 & RT & 1.2 & 1.1 & 2.0 & 1.9 & 4.4 & 4.0 & 10.0 & 8.4 \\
 & pRT & 2.5 & 0.7 & 4.3 & 1.0 & 11.5 & 1.4 & 19.0 & 1.8 \\
 & path-IMH & 1.6 & 1.6 & 3.5 & 4.0 & 23.5 & 21.3 & $\ge$176 & $\ge$181 \\
\midrule
0.30 & HMC & 5.0 & $\ge$3884 & 3.1 & no flip & 3.2 & no flip & 3.5 & no flip \\
 & Wolff & 10.2 & $\ge$119 & 10.2 & 78.5 & 9.4 & 81.9 & 10.3 & 58.4 \\
 & RT & 1.4 & 1.0 & 1.6 & 1.3 & 2.2 & 2.5 & 3.7 & 3.3 \\
 & pRT & 1.2 & 0.7 & 1.1 & 0.9 & 1.3 & 1.1 & 1.5 & 1.3 \\
 & path-IMH & 3.9 & 1.8 & 2.9 & 3.1 & 5.0 & 5.3 & 24.4 & 19.2 \\
\bottomrule\end{tabular}\end{table*}

\section{Bidirectional free-energy estimation and supplementary checks}
\label{app:bar_free_energy}

Section~\ref{sec:free_energy} estimates the partition function using complete-path work.
After transfer of the frozen model to $512\times8$, BAR still agrees with TI, while the one-sided estimates deviate (Table~\ref{tab:bar_transfer}).

\subsection{Rectangular lattices and forward and reverse work}

The rectangular calculations use periodic $N_x\times8$ lattices with $N_x=64,128,256,512$, $\lambda=0.022$, and $\kappa=0.20$ or $0.30$.
For each coupling a single NHMC source model is trained on $64\times8$. Its parameters are frozen and applied directly to the larger $N_x$.
At every size the forward sample consists of $65536$ freshly generated complete paths.
Reverse records come from truncated round-trip chains ($s=12$) started at forward-path endpoints,
with $16$ chains for $N_x\ge128$ and $4$ chains for $64\times8$, each of $2\times10^4$ updates, with the first $10\%$ discarded.
Block errors and BAR weights for the reverse records follow the conventions of Sec.~\ref{app:chain_free_energy}.

\subsection{Forward and reverse estimators}

With the dimensionless free energy $F=-\log Z$ of Sec.~\ref{sec:nhmc_pathlaws}, the normalized initial law $q_0$ has reference free energy zero. Equation~\eqref{eq:nhmc_pathratio} gives $\log(Q_\theta/P_\theta)=W-F$, so under common support $\mathbb E_Q[e^{-W}]=Z$ and $\mathbb E_P[e^{W}]=Z^{-1}$.

Let $W_i^{\rm F}=W(\Gamma_i)$ for paths sampled from $Q_\theta$, and
$W_j^{\rm R}=-W(\Gamma_j)$ for reverse paths sampled from $P_\theta$.
The superscript R denotes work oriented along the reverse traversal.
Its negative is placed on the forward-work axis when the two distributions
are compared, as in the Crooks relation~\cite{Crooks:1999ttd}.
The forward and reverse exponential identities yield
\begin{align}
 \widehat F_{\rm F}
 &=-\log\!\left[\frac{1}{n_{\rm F}}\sum_i e^{-W_i^{\rm F}}\right],
 \nonumber\\
 \widehat F_{\rm R}
 &=\phantom{-}\log\!\left[\frac{1}{n_{\rm R}}\sum_j e^{-W_j^{\rm R}}\right].
 \label{eq:bar_oneway_estimators}
\end{align}
Finite-sample one-sided estimates are dominated by rare work values and
drift at large volume (Table~\ref{tab:bar_transfer}).

The bidirectional estimate is BAR~\cite{Bennett_1976,Shirts_2003}, written in the form of Sec.~\ref{app:chain_free_energy} with the forward and reverse samples weighted by the effective number of reverse records.
Reverse paths start from target configurations of the corrected chains, which are themselves started at forward-path endpoints. Drawing the remaining variables from the learned reverse momentum laws then samples $P_\theta$.

We characterize weight concentration in each one-sided exponential average by
\begin{equation}
 \mathrm{ESS}_w^{\rm F}=
 \frac{(\sum_i e^{-W_i^{\rm F}})^2}{\sum_i e^{-2W_i^{\rm F}}}.
 \label{eq:bar_ess}
\end{equation}
This is $\mathrm{ESS}_w$ of Eq.~\eqref{eq:phi4_snis}, and $\mathrm{ESS}_w^{\rm R}$ for reverse records is defined in the same way with weights $e^{-W_j^{\rm R}}$.

\subsection{Correlations in reverse records from corrected chains}
\label{app:chain_free_energy}

Thermalized corrected chains provide the reverse starting points.
For path-IMH and pRT we draw a new complete reverse path from each recorded configuration. This path returns from the endpoint to stage zero.
In configuration updates pRT walks back only to stage $s$. For the free energy we record the path work of all stages separately.
A full round trip already produces two paired complete paths. When the candidate is accepted, we record the path of the new endpoint.
When it is rejected, we record the reverse path of the current endpoint.

In the square-lattice calculations of Sec.~\ref{sec:free_energy}, path-IMH, RT and pRT each provide complete reverse records from four corrected chains.
The first $10\%$ of the records from each chain are discarded, and the remaining records are divided into $20$ blocks.
The forward samples are also divided into $20$ blocks.
Each block-jackknife replicate removes one time block from each chain and one block of forward samples simultaneously,
then recomputes the one-sided and bidirectional estimates.

To specify the BAR weights, let $n_{\mathrm F}$ and $n_{\mathrm R}$ denote the number of forward paths and the total number of reverse records from the chains used in the estimate,
and let $k_{\mathrm R}$ denote the number of configuration updates between consecutive reverse records.
Write $\widehat\tau_W^{\mathrm{rec}}$ for the integrated autocorrelation time of work estimated along the reverse-record sequence,
and $\widehat\tau_{|m|}^{\mathrm{step}}$ for the amplitude integrated autocorrelation time estimated along the per-update configuration sequence.
Both use Sokal automatic-window estimates and the median of the per-chain estimates,
with the definition in Eq.~\eqref{eq:nhmc_corr_convention}.
BAR uses
\begin{equation}
 n_{\mathrm{R,eff}}=
 \frac{n_{\mathrm R}}{\max\{2\widehat\tau_W^{\mathrm{rec}},
 2\widehat\tau_{|m|}^{\mathrm{step}}/k_{\mathrm R},1\}}.
 \label{eq:bar_chain_effective_records}
\end{equation}
We call $n_{\mathrm{R,eff}}$ the effective number of reverse records. It sets the relative weights of the forward and reverse samples in BAR.
It describes temporal correlations, while the effective sample size in Eq.~\eqref{eq:bar_ess}, defined from one-sided work weights, describes weight concentration.

Let $\omega_{\mathrm F}=n_{\mathrm F}/(n_{\mathrm F}+n_{\mathrm{R,eff}})$
and $\omega_{\mathrm R}=1-\omega_{\mathrm F}$.
Writing the path weight as $w(\Gamma)=e^{-W(\Gamma)}$ under the same forward-work convention,
BAR solves the equation~\cite{Meng:1996ss}
\begin{equation}
 \widehat Z=
 \frac{\displaystyle\frac{1}{n_{\mathrm F}}
  \sum_{i\in\mathrm F}\frac{w_i}{\omega_{\mathrm R}w_i+\omega_{\mathrm F}\widehat Z}}
 {\displaystyle\frac{1}{n_{\mathrm R}}
  \sum_{j\in\mathrm R}\frac{1}{\omega_{\mathrm R}w_j+\omega_{\mathrm F}\widehat Z}}.
 \label{eq:bar_chain_bridge}
\end{equation}
For reverse records, $w_j=e^{W_j^{\mathrm R}}$, with the sign convention for $W_j^{\mathrm R}$ given in Eq.~\eqref{eq:bar_oneway_estimators}.
The bidirectional results on both the square and the rectangular lattices use Eq.~\eqref{eq:bar_chain_bridge}, with the weights adjusted by the effective number of reverse records in Eq.~\eqref{eq:bar_chain_effective_records}.
The block-based uncertainty calculation holds $\omega_{\mathrm F}$ and $\omega_{\mathrm R}$ fixed at their full-sample values.

\subsection{Thermodynamic-integration reference}

The independent free-energy reference for the rectangular lattices is obtained by thermodynamic integration,
\begin{equation}
 \begin{aligned}
 F(\kappa)&=F(0)+\int_0^\kappa d\kappa'\,
 \left\langle\frac{\partial S}{\partial\kappa'}\right\rangle_{\kappa'},\\
 \frac{\partial S}{\partial\kappa}&=-2\sum_{\boldsymbol r,\mu}\phi_{\boldsymbol r}\phi_{\boldsymbol r+\hat\mu}.
 \end{aligned}
 \label{eq:bar_ti}
\end{equation}
At $\kappa=0$, the action factorizes, fixing the normalization through
\begin{equation}
 \frac{F(0)}{V}=-\log\int_{-\infty}^{\infty}d\phi\,
 e^{-(1-2\lambda)\phi^2-\lambda\phi^4}.
 \label{eq:bar_reference_anchor}
\end{equation}
The reference value is $F(0)/V=-0.578278$ for $\lambda=0.022$.
Equation~\eqref{eq:phi4_action} differs from the common form $\sum_{\boldsymbol r}[\phi_{\boldsymbol r}^2+\lambda(\phi_{\boldsymbol r}^2-1)^2]$ (plus the hopping term) only by the constant $\lambda V$. The work and the reference calculations both use Eq.~\eqref{eq:phi4_action} as written.
The integrand $\langle\partial S/\partial\kappa\rangle$ is measured point by point with Wolff-type updates.
Each update is one local Metropolis sweep followed by two single-cluster flips, with clusters built on the embedded Ising variables~\cite{Wolff:1988uh,Brower:1989mt},
and each coupling point uses $2\times10^6$ updates.
The coupling grid has spacing $0.005$ for $\kappa\le0.25$ and $0.001$ for $0.25<\kappa\le0.30$, $101$ points up to $\kappa=0.30$.
The integral uses the composite Simpson rule, and its discretization error is taken as $1/15$ of the change on doubling the spacing (Richardson estimate)
and is combined in quadrature with the statistical error.
The resulting per-site reference error is $1\times10^{-6}$ to $3\times10^{-6}$ at $\kappa=0.20$ and $1.0\times10^{-5}$ to $2.8\times10^{-5}$ at $\kappa=0.30$,
decreasing with volume.

For the square lattices ($L=4$ to $16$), the TI reference combines the Wolff and HMC measurements of $\langle\partial S/\partial\kappa\rangle$ ($10^6$ records each) by inverse-variance weighting. The coupling grid runs from $0$ to $0.30$ with spacing $0.005$, refined to $0.001$ on $0.25$ to $0.29$ for $L=12,16$. We integrate with the composite Simpson rule and estimate the discretization error as $|I(\delta\kappa)-I(2\delta\kappa)|/15$, where $\delta\kappa$ is the coupling step. For $\kappa=0.265$ and $0.275$ we take the upper limit of the integral at these values.

\subsection{Residuals in $\log Z$ per site for the frozen model}

Table~\ref{tab:bar_transfer} lists the per-site residuals $(\log\widehat Z-\log Z_{\rm TI})/V$ of the forward exponential average, the reverse exponential average and BAR.
At both couplings BAR agrees with the reference at every size from $N_x=64$ to $512$, with residuals of at most $0.25\times10^{-4}$,
all within about $1.3$ combined standard errors.
The one-sided estimates drift away with volume. At $512\times8$ the forward estimate lies low and the reverse estimate high,
in the directions expected for finite-sample exponential averages and far outside their errors.
Over the same range the truncated round-trip acceptance falls and $\tau_{\rm int}(|m|)$ grows (last two columns).

\begin{widetable}
\centering\small
\caption{Residuals in $\log Z$ per site, $(\log\widehat Z-\log Z_{\rm TI})/V$, in units of $10^{-4}$, after transferring the frozen $64\times8$ source model to $N_x\times8$.
Errors add in quadrature the block-jackknife error over forward and reverse samples and the error of the TI reference. The BAR column uses Eq.~\eqref{eq:bar_chain_bridge} with reverse samples weighted by $n_{\mathrm{R,eff}}$.
$\bar\alpha$ and $\tau_{\rm int}(|m|)$ are the mean acceptance and integrated autocorrelation time of the truncated round-trip chains ($s=12$) that supply the reverse records. For each chain $\tau_{\rm int}(|m|)$ is the larger of the Sokal automatic-window and $50$-batch-means estimates, and the table gives the mean over $16$ chains for $N_x\ge128$ and over $4$ chains for $64\times8$. $\ge$ marks lower bounds.}
\label{tab:bar_transfer}
\begin{tabular}{llrrrrr}
\toprule
$\kappa$ & lattice & forward & reverse & BAR & $\bar\alpha$ & $\tau_{\rm int}(|m|)$\\
\midrule
$0.20$ & $64\times8$ & $+0.27(45)$ & $+0.12(36)$ & $+0.14(11)$ & $0.54$ & $1.9$\\
       & $128\times8$ & $-0.15(49)$ & $-0.35(38)$ & $+0.05(8)$ & $0.38$ & $3.2$\\
       & $256\times8$ & $+1.1(2.7)$ & $+0.10(94)$ & $+0.02(9)$ & $0.22$ & $6.6$\\
       & $512\times8$ & $-5.28(50)$ & $+3.55(59)$ & $-0.05(11)$ & $0.08$ & $\ge21.0$\\
\midrule
$0.30$ & $64\times8$ & $-0.71(66)$ & $-0.32(59)$ & $-0.18(30)$ & $0.49$ & $2.3$\\
       & $128\times8$ & $-2.98(58)$ & $-0.54(71)$ & $-0.24(21)$ & $0.33$ & $3.3$\\
       & $256\times8$ & $+0.7(3.4)$ & $+1.4(1.5)$ & $-0.03(19)$ & $0.17$ & $8.8$\\
       & $512\times8$ & $-5.1(1.3)$ & $+5.2(1.1)$ & $-0.24(31)$ & $0.05$ & $\ge40.0$\\
\bottomrule
\end{tabular}
\end{widetable}

Figure~\ref{fig:bar_workdist} shows the distributions of the forward work and of the reverse-record work, both expressed as $W(\Gamma)$ (that is, the reverse records enter as $-W^{\rm R}$).
By the Crooks relation the two distributions cross at $W=F$~\cite{Crooks:1999ttd}.
At $64\times8$ they overlap substantially, while at $512\times8$ they meet only in their tails near $F_{\rm TI}=-\log Z_{\rm TI}$,
so the one-sided exponential averages are set by a few tail samples, whereas BAR needs only this overlap.
At $\kappa=0.20$ the per-site forward work variance is about $0.0063$ at every size. The work variance is therefore proportional to the volume and transfer itself does not raise the per-site variance.
At $\kappa=0.30$ it grows from $0.009$ at $64\times8$ to $0.028$ at $512\times8$ and the forward distribution develops a right tail,
while the per-site variance of the reverse records changes much less. This is the variance of the complete forward work. The truncated round-trip acceptance is set by the conditional variance of the redrawn stages.

\begin{widefigure}
\centering
\includegraphics[width=0.98\textwidth,keepaspectratio]{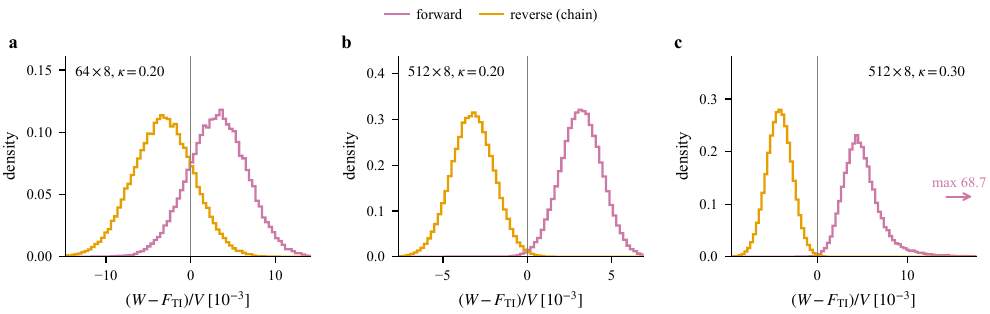}
\caption{Distributions of work evaluated on forward and reverse paths of the frozen model. The horizontal axis shows the per-site work relative to the TI reference, $(W-F_{\rm TI})/V$. The vertical line marks $W=F_{\rm TI}$. The vertical axis is the probability density of each distribution.
Reverse records enter as $-W^{\rm R}=W(\Gamma)$, taken from truncated round-trip chains with $s=12$.
(a) $64\times8$, $\kappa=0.20$. (b) $512\times8$, $\kappa=0.20$. (c) $512\times8$, $\kappa=0.30$. The arrow in (c) marks the forward tail beyond the plotted range.}
\label{fig:bar_workdist}
\end{widefigure}

\section{Other learned samplers, CMCD and DBS}
\label{app:other_samplers}

The round-trip construction needs only forward and reverse path probabilities and therefore applies to the overdamped CMCD~\cite{vargas2024transport} and DBS~\cite{richter2024improved} samplers. We train both with the official code of the underdamped diffusion bridges (UDB) work~\cite{blessing2025underdamped} and its $\phi^4$ target. We symmetrize the networks under $\mathbb Z_2$ and use $80$ steps at $L=8$ and $160$ at $L=16$, with forward kernel $\mathcal K^{\rm F}_k=\mathcal N(x+\epsilon_k(\mathrm{drift}_k+v^{\rm F}),2\epsilon_k)$ and reverse kernel $\mathcal K^{\rm R}_k=\mathcal N(x'+\epsilon_k(\mathrm{drift}_{k+1}+v^{\rm R}),2\epsilon_k)$. Here $\epsilon_k$ is the step size of step $k$, $\mathrm{drift}_k=\beta_k\nabla\log\gamma+(1-\beta_k)\nabla\log q_0$ the score of the annealed density at step $k$ with the unnormalized target density $\gamma=e^{-S}$, $v^{\rm F}$ and $v^{\rm R}$ the learned forward and reverse control drifts, and $x'$ the configuration after the step. CMCD sets $v^{\rm R}=-v^{\rm F}$, while DBS learns $v^{\rm R}$ with a separate network. The annealing schedule $\beta_k$ and the Gaussian initial distribution $q_0$ are also trained. For these samplers $T$ is the number of discretization steps, and the paths contain no momentum refreshes. CMCD and DBS use $10^4$ Adam training epochs with $500$ paths each ($5\times10^6$ training paths, learning rate $5\times10^{-4}$). NHMC uses $2\times10^4$ training epochs with $48$ paths each ($9.6\times10^5$ paths) plus the round-trip buffer of Appendix~\ref{app:implementation}. The settings differ in two more respects. The training of CMCD and DBS has no reverse-path term from round-trip chains and no explicit sign flip in every stage, and their $\mathbb Z_2$ symmetry enters only through the symmetrized networks.

Table~\ref{tab:udb_compare} compares all couplings at $L=8$. With the conditional work variance measured directly as in Appendix~\ref{app:scaling} ($2048$ intermediate fields per coupling, each continued by $16$ paths), the full round-trip acceptance predicted by $\operatorname{erfc}(\sigma_s/2)$ agrees with the measured chain acceptance to within $0.01$, except at $\kappa=0.30$ ($0.03$ for CMCD and $0.07$ for DBS). For truncated round trips ($s=60$) the difference is at most $0.02$. Full round trips of CMCD and DBS approximately follow Eq.~\eqref{eq:tau_law}, with the ratio $\tau_{\rm int}/(1/\bar\alpha-1/2)$ between $0.9$ and $1.8$ at the individual couplings. For CMCD and DBS the per-site variance $\sigma_\alpha^2/V$ implied by the acceptance is similar at $L=8$ and $16$ for $\kappa=0.24,0.27,0.30$, except for CMCD at $\kappa=0.30$, where $L=16$ is about $32\%$ higher. Truncated round trips resample the last quarter of the path. Even accepted updates change the configuration only slightly, and the ratio reaches about $10$ to $100$ (lower bounds at $\kappa=0.28$ and $0.29$).

\begin{widetable}
\centering\small
\renewcommand{\arraystretch}{0.96}
\setlength{\tabcolsep}{4pt}
\caption{Comparison of NHMC with CMCD and DBS (UDB code, $\mathbb Z_2$-symmetrized networks) at $L=8$. $\operatorname{Var}(\log w)$ is from forward paths ($262144$ for NHMC, $102400$ for CMCD/DBS). We report the median $\mathrm{ESS}_w/N$ over four batches of $N=16\,384$ forward paths each. For $\hat k$ we report the median of the four tail shape estimates from generalized Pareto fits~\cite{JMLR:v25:19-556} to the largest weights. The path-IMH, RT and pRT columns give acceptance / $\tau_{\rm int}(|m|)$. pRT redraws the last quarter of the path in all cases (NHMC $d=4$ of $T=16$, CMCD/DBS $d=20$ of $T=80$). $\ge$ marks lower bounds.}
\label{tab:udb_compare}\begin{tabular}{llrrrrrr}\toprule $\kappa$ & sampler & $\operatorname{Var}(\log w)$ & $\mathrm{ESS}_w/N$ & $\hat k$ & path-IMH & RT & pRT \\
\midrule
0.20 & NHMC & 0.41 & 0.651 & 0.22 & 0.65 / 1.5 & 0.66 / 1.1 & 0.82 / 0.9 \\
 & CMCD & 0.19 & 0.826 & 0.03 & 0.76 / 0.9 & 0.78 / 1.0 & 0.94 / 10.0 \\
 & DBS & 0.20 & 0.814 & 0.10 & 0.75 / 1.0 & 0.76 / 1.0 & 0.95 / 6.9 \\
0.21 & NHMC & 0.47 & 0.603 & 0.18 & 0.62 / 1.8 & 0.64 / 1.4 & 0.81 / 1.0 \\
 & CMCD & 0.21 & 0.807 & 0.05 & 0.74 / 1.0 & 0.76 / 1.0 & 0.94 / 12.2 \\
 & DBS & 0.22 & 0.800 & 0.06 & 0.73 / 1.0 & 0.74 / 0.9 & 0.95 / 8.7 \\
0.22 & NHMC & 0.54 & 0.567 & 0.20 & 0.59 / 1.6 & 0.61 / 1.2 & 0.80 / 1.1 \\
 & CMCD & 0.24 & 0.779 & 0.10 & 0.73 / 1.2 & 0.74 / 1.0 & 0.93 / 12.8 \\
 & DBS & 0.25 & 0.775 & 0.11 & 0.71 / 1.3 & 0.73 / 1.0 & 0.94 / 10.4 \\
0.23 & NHMC & 0.63 & 0.500 & 0.19 & 0.57 / 1.9 & 0.58 / 1.4 & 0.78 / 1.3 \\
 & CMCD & 0.28 & 0.753 & 0.12 & 0.71 / 1.3 & 0.72 / 1.1 & 0.93 / 14.4 \\
 & DBS & 0.29 & 0.742 & 0.15 & 0.70 / 1.2 & 0.71 / 1.1 & 0.94 / 11.0 \\
0.24 & NHMC & 0.74 & 0.444 & 0.21 & 0.53 / 2.4 & 0.55 / 1.6 & 0.77 / 1.6 \\
 & CMCD & 0.33 & 0.712 & 0.16 & 0.69 / 1.6 & 0.70 / 1.2 & 0.92 / 21.1 \\
 & DBS & 0.34 & 0.704 & 0.17 & 0.68 / 1.3 & 0.69 / 1.1 & 0.94 / 15.1 \\
0.25 & NHMC & 0.87 & 0.383 & 0.22 & 0.49 / 3.2 & 0.51 / 1.6 & 0.75 / 2.6 \\
 & CMCD & 0.42 & 0.644 & 0.20 & 0.64 / 2.4 & 0.66 / 1.3 & 0.93 / 29.0 \\
 & DBS & 0.40 & 0.656 & 0.14 & 0.65 / 1.6 & 0.66 / 1.3 & 0.93 / 27.6 \\
0.26 & NHMC & 1.06 & 0.288 & 0.38 & 0.46 / 2.8 & 0.46 / 1.9 & 0.72 / 3.7 \\
 & CMCD & 0.55 & 0.533 & 0.31 & 0.59 / 2.2 & 0.60 / 1.9 & 0.92 / 45.9 \\
 & DBS & 0.49 & 0.599 & 0.17 & 0.62 / 1.8 & 0.63 / 1.3 & 0.92 / 37.9 \\
0.27 & NHMC & 1.23 & 0.240 & 0.35 & 0.43 / 3.5 & 0.43 / 2.0 & 0.71 / 4.3 \\
 & CMCD & 0.88 & 0.392 & 0.21 & 0.51 / 2.2 & 0.51 / 1.3 & 0.91 / 50.0 \\
 & DBS & 0.64 & 0.507 & 0.24 & 0.57 / 1.9 & 0.58 / 1.5 & 0.91 / 63.0 \\
0.28 & NHMC & 1.02 & 0.210 & 0.44 & 0.46 / 10.0 & 0.49 / 2.3 & 0.75 / 2.9 \\
 & CMCD & 1.49 & 0.249 & 0.29 & 0.39 / 3.8 & 0.41 / 1.9 & 0.87 / $\ge$27.3 \\
 & DBS & 0.88 & 0.343 & 0.38 & 0.50 / 4.5 & 0.51 / 2.6 & 0.89 / $\ge$34.1 \\
0.29 & NHMC & 0.91 & 0.263 & 0.49 & 0.50 / 3.5 & 0.51 / 1.7 & 0.78 / 1.4 \\
 & CMCD & 2.59 & 0.069 & 0.50 & 0.26 / 46.7 & 0.27 / 3.4 & 0.82 / $\ge$11.3 \\
 & DBS & 1.69 & 0.130 & 0.48 & 0.36 / 14.8 & 0.37 / 3.3 & 0.80 / $\ge$11.4 \\
0.30 & NHMC & 0.92 & 0.334 & 0.47 & 0.52 / 2.9 & 0.54 / 1.6 & 0.80 / 1.1 \\
 & CMCD & 4.17 & 0.018 & 0.62 & 0.17 / 78.3 & 0.18 / 5.4 & 0.77 / 7.8 \\
 & DBS & 3.51 & 0.081 & 0.52 & 0.25 / 14.1 & 0.25 / 4.1 & 0.74 / 8.9 \\ \bottomrule\end{tabular}
\end{widetable}

\textbf{Partition functions.} We estimate $\log Z$ using the Bennett acceptance ratio. Reverse samples come from full round trips, except for NHMC at $L=16$, where we use truncated round trips ($s=12$). For each estimate we use four chains of $2\times10^4$ updates and $65536$ forward paths. At $L=8$ all three samplers agree with TI to within $0.011$ at all 11 couplings (Fig.~\ref{fig:udb_bar}(a)). At $L=16$ ($\kappa=0.24,0.27,0.30$) the difference appears in the broken phase. At $\kappa=0.30$ the full round-trip acceptance is $0.28$ for NHMC, $0.03$ for DBS and $0.002$ for CMCD. The CMCD chain rarely accepts an update. The DBS estimate deviates by about two standard errors, and NHMC truncated round trips agree with the TI reference (Fig.~\ref{fig:udb_bar}).

\textbf{Work variance and weights.} In the symmetric phase and near criticality the CMCD and DBS work variances are about $0.44$ to $0.72$ times the NHMC value (Table~\ref{tab:udb_compare}). At $\kappa=0.29$ and $0.30$ both become $1.9$ to $4.5$ times the NHMC value at $L=8$. The effective sample fractions show the same pattern at these two couplings. A value $\hat k>0.5$ indicates that the variance of the weights may diverge~\cite{JMLR:v25:19-556}. Where $\hat k$ approaches or exceeds $0.5$, $\mathrm{ESS}_w$ is dominated by a few paths and varies strongly between batches. At $\kappa=0.28$ the CMCD median falls within the range of the four NHMC estimates, while at $\kappa=0.29$ to $0.30$ NHMC is clearly higher. From the full round-trip acceptance, at $\kappa=0.30$ the per-site variances $\sigma_\alpha^2/V$ of CMCD and DBS are about $4.6$ and $3.4$ times that of NHMC at $L=8$, growing to about $8$ and $4$ times at $L=16$.

\begin{widefigure}
\centering
\includegraphics[width=0.91\textwidth,keepaspectratio]{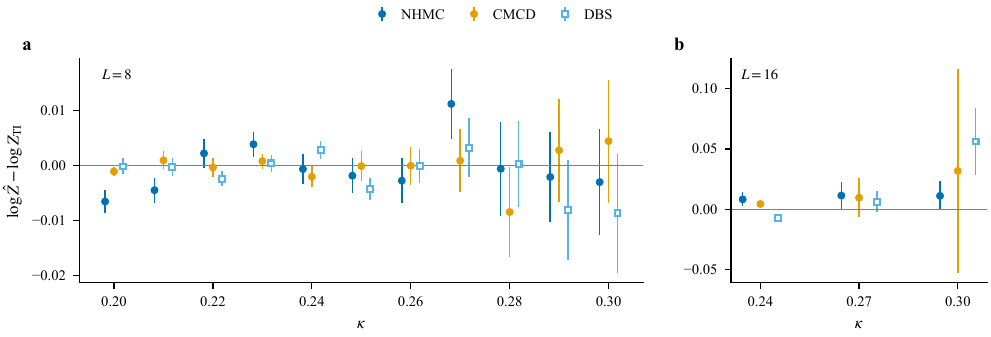}
\caption{Deviation of the total $\log\hat Z$ from $\log Z_{\rm TI}$ (BAR), not per site as in Supplemental Fig.~\ref{fig:nhmc_square_logz}. (a) $L=8$, full round trips for all three samplers. (b) $L=16$, full round trips for CMCD and DBS, truncated round trips ($s=12$) for NHMC. NHMC and CMCD are filled circles of different colors, and DBS is drawn with open squares. The horizontal line is the TI reference. Error bars add two parts in quadrature, the block-jackknife error of BAR and the error of the TI reference. The TI error includes the statistical error and the discretization error estimated by doubling the integration step.}
\label{fig:udb_bar}
\end{widefigure}

\clearpage
\bibliography{nhmcrt}

@article{Metropolis:1953am,
    author = "Metropolis, N. and Rosenbluth, A. W. and Rosenbluth, M. N. and Teller, A. H. and Teller, E.",
    title = "{Equation of state calculations by fast computing machines}",
    doi = "10.1063/1.1699114",
    journal = "J. Chem. Phys.",
    volume = "21",
    pages = "1087--1092",
    year = "1953"
}

@article{Hastings:1970aa,
    author = "Hastings, W. K.",
    title = "{Monte Carlo Sampling Methods Using Markov Chains and Their Applications}",
    doi = "10.1093/biomet/57.1.97",
    journal = "Biometrika",
    volume = "57",
    pages = "97--109",
    year = "1970"
}

@article{Duane:1987de,
    author = "Duane, S. and Kennedy, A. D. and Pendleton, B. J. and Roweth, D.",
    title = "{Hybrid Monte Carlo}",
    doi = "10.1016/0370-2693(87)91197-X",
    journal = "Phys. Lett. B",
    volume = "195",
    pages = "216--222",
    year = "1987"
}

@inbook{Neal_2011, title={{MCMC} Using {H}amiltonian Dynamics}, ISBN={9780429138508}, url={http://dx.doi.org/10.1201/b10905-6}, DOI={10.1201/b10905-6}, booktitle={Handbook of Markov Chain Monte Carlo}, publisher={Chapman and Hall/CRC}, author={Neal, Radford M.}, year={2011}, month=May, pages={113--162} }

@Inbook{Sokal_1997,
author="Sokal, A.",
editor="DeWitt-Morette, Cecile
and Cartier, Pierre
and Folacci, Antoine",
title="{M}onte {C}arlo Methods in Statistical Mechanics: Foundations and New Algorithms",
bookTitle="Functional Integration: Basics and Applications",
year="1997",
publisher="Springer US",
address="Boston, MA",
pages="131--192",
isbn="978-1-4899-0319-8",
doi="10.1007/978-1-4899-0319-8_6",
url="https://doi.org/10.1007/978-1-4899-0319-8_6"
}

@article{Schaefer:2010hu,
    author = "Schaefer, Stefan and Sommer, Rainer and Virotta, Francesco",
    collaboration = "ALPHA",
    title = "{Critical slowing down and error analysis in lattice QCD simulations}",
    eprint = "1009.5228",
    archivePrefix = "arXiv",
    primaryClass = "hep-lat",
    reportNumber = "DESY-10-151, SFB-CPP-10-81, HU-EP-10-55",
    doi = "10.1016/j.nuclphysb.2010.11.020",
    journal = "Nucl. Phys. B",
    volume = "845",
    pages = "93--119",
    year = "2011"
}

@article{Albergo:2019eim,
    author = "Albergo, M. S. and Kanwar, G. and Shanahan, P. E.",
    title = "{Flow-based generative models for Markov chain Monte Carlo in lattice field theory}",
    eprint = "1904.12072",
    archivePrefix = "arXiv",
    primaryClass = "hep-lat",
    reportNumber = "MIT-CTP/5114",
    doi = "10.1103/PhysRevD.100.034515",
    journal = "Phys. Rev. D",
    volume = "100",
    number = "3",
    pages = "034515",
    year = "2019"
}

@article{Singha:2026slc,
    author = "Singha, Ankur and Kauffmann, Jacob and Cellini, Elia and Jansen, Karl and Nakajima, Shinichi",
    title = "{Scalable generative sampling and multilevel estimation for lattice field theories near criticality}",
    eprint = "2604.10209",
    archivePrefix = "arXiv",
    primaryClass = "hep-lat",
    doi = "10.1103/jjss-lwzv",
    journal = "Phys. Rev. D",
    volume = "114",
    number = "7",
    pages = "074515",
    year = "2026"
}

@article{Wolff:1988uh,
    author = "Wolff, Ulli",
    title = "{Collective Monte Carlo Updating for Spin Systems}",
    reportNumber = "DESY-88-144",
    doi = "10.1103/PhysRevLett.62.361",
    journal = "Phys. Rev. Lett.",
    volume = "62",
    pages = "361",
    year = "1989"
}

@InProceedings{pmlr-v37-rezende15,
  title = {Variational Inference with Normalizing Flows},
  author = {Rezende, Danilo and Mohamed, Shakir},
  booktitle = {Proceedings of the 32nd International Conference on Machine Learning},
  pages = {1530--1538},
  year = {2015},
  editor = {Bach, Francis and Blei, David},
  volume = {37},
  series = {Proceedings of Machine Learning Research},
  address = {Lille, France},
  publisher = {PMLR},
  url = {https://proceedings.mlr.press/v37/rezende15.html}
}

@article{JMLR:v22:19-1028,
  author  = {George Papamakarios and Eric Nalisnick and Danilo Jimenez Rezende and Shakir Mohamed and Balaji Lakshminarayanan},
  title   = {Normalizing Flows for Probabilistic Modeling and Inference},
  journal = {J. Mach. Learn. Res.},
  year    = {2021},
  volume  = {22},
  number  = {57},
  pages   = {1--64},
  url     = {http://jmlr.org/papers/v22/19-1028.html}
}

@article{Luscher:2009eq,
    author = "L{\"u}scher, Martin",
    title = "{Trivializing maps, the Wilson flow and the HMC algorithm}",
    eprint = "0907.5491",
    archivePrefix = "arXiv",
    primaryClass = "hep-lat",
    reportNumber = "CERN-PH-TH-2009-118",
    doi = "10.1007/s00220-009-0953-7",
    journal = "Commun. Math. Phys.",
    volume = "293",
    pages = "899--919",
    year = "2010"
}

@article{Bacchio:2022vje,
    author = "Bacchio, Simone and Kessel, Pan and Schaefer, Stefan and Vaitl, Lorenz",
    title = "{Learning trivializing gradient flows for lattice gauge theories}",
    eprint = "2212.08469",
    archivePrefix = "arXiv",
    primaryClass = "hep-lat",
    doi = "10.1103/PhysRevD.107.L051504",
    journal = "Phys. Rev. D",
    volume = "107",
    number = "5",
    pages = "L051504",
    year = "2023"
}

@article{Nicoli:2020njz,
    author = "Nicoli, Kim A. and Anders, Christopher J. and Funcke, Lena and Hartung, Tobias and Jansen, Karl and Kessel, Pan and Nakajima, Shinichi and Stornati, Paolo",
    title = "{Estimation of Thermodynamic Observables in Lattice Field Theories with Deep Generative Models}",
    eprint = "2007.07115",
    archivePrefix = "arXiv",
    primaryClass = "hep-lat",
    reportNumber = "DESY-20-148",
    doi = "10.1103/PhysRevLett.126.032001",
    journal = "Phys. Rev. Lett.",
    volume = "126",
    number = "3",
    pages = "032001",
    year = "2021"
}

@article{Kanwar:2020xzo,
    author = "Kanwar, Gurtej and Albergo, Michael S. and Boyda, Denis and Cranmer, Kyle and Hackett, Daniel C. and Racani{\`e}re, S{\'e}bastien and Rezende, Danilo Jimenez and Shanahan, Phiala E.",
    title = "{Equivariant flow-based sampling for lattice gauge theory}",
    eprint = "2003.06413",
    archivePrefix = "arXiv",
    primaryClass = "hep-lat",
    reportNumber = "MIT-CTP/5181",
    doi = "10.1103/PhysRevLett.125.121601",
    journal = "Phys. Rev. Lett.",
    volume = "125",
    number = "12",
    pages = "121601",
    year = "2020"
}

@article{Boyda:2020hsi,
    author = "Boyda, Denis and Kanwar, Gurtej and Racani{\`e}re, S{\'e}bastien and Rezende, Danilo Jimenez and Albergo, Michael S. and Cranmer, Kyle and Hackett, Daniel C. and Shanahan, Phiala E.",
    title = "{Sampling using $SU(N)$ gauge equivariant flows}",
    eprint = "2008.05456",
    archivePrefix = "arXiv",
    primaryClass = "hep-lat",
    reportNumber = "MIT-CTP/5228",
    doi = "10.1103/PhysRevD.103.074504",
    journal = "Phys. Rev. D",
    volume = "103",
    number = "7",
    pages = "074504",
    year = "2021"
}

@article{Singha:2022icw,
    author = "Singha, Ankur and Chakrabarti, Dipankar and Arora, Vipul",
    title = "{Conditional normalizing flow for Markov chain Monte~Carlo sampling in the critical region of lattice field theory}",
    eprint = "2207.00980",
    archivePrefix = "arXiv",
    primaryClass = "hep-lat",
    doi = "10.1103/PhysRevD.107.014512",
    journal = "Phys. Rev. D",
    volume = "107",
    number = "1",
    pages = "014512",
    year = "2023"
}

@article{Abbott:2022zsh,
    author = "Abbott, Ryan and others",
    title = "{Aspects of scaling and scalability for flow-based sampling of lattice QCD}",
    eprint = "2211.07541",
    archivePrefix = "arXiv",
    primaryClass = "hep-lat",
    reportNumber = "MIT-CTP/5496",
    doi = "10.1140/epja/s10050-023-01154-w",
    journal = "Eur. Phys. J. A",
    volume = "59",
    number = "11",
    pages = "257",
    year = "2023"
}

@article{Liu_1996, title={{M}etropolized independent sampling with comparisons to rejection sampling and importance sampling}, volume={6}, ISSN={1573-1375}, url={http://dx.doi.org/10.1007/BF00162521}, DOI={10.1007/bf00162521}, number={2}, journal={Stat. Comput.}, publisher={Springer Science and Business Media LLC}, author={Liu, Jun S.}, year={1996}, month=jun, pages={113--119} }

@article{Mengersen_1996, title={Rates of convergence of the {H}astings and {M}etropolis algorithms}, volume={24}, ISSN={2168-8966}, url={http://dx.doi.org/10.1214/aos/1033066201}, DOI={10.1214/aos/1033066201}, number={1}, journal={Ann. Statist.}, publisher={Institute of Mathematical Statistics}, author={Mengersen, K. L. and Tweedie, R. L.}, year={1996}, month=Feb, pages={101--121} }

@article{Gabrie:2021tlu,
    author = "Gabri{\'e}, Marylou and Rotskoff, Grant M. and Vanden-Eijnden, Eric",
    title = "{Adaptive Monte Carlo augmented with normalizing flows}",
    eprint = "2105.12603",
    archivePrefix = "arXiv",
    primaryClass = "physics.data-an",
    doi = "10.1073/pnas.2109420119",
    journal = "Proc. Natl. Acad. Sci. U.S.A.",
    volume = "119",
    number = "10",
    pages = "e2109420119",
    year = "2022"
}

@misc{Qian:2025uvx,
    author = "Qian, Moxian and Chen, Shiyang",
    title = "{Sequential Autoregressive Block Metropolis--Hastings for Two-Dimensional $\phi^4$ Lattice Fields}",
    eprint = "2512.19575",
    archivePrefix = "arXiv",
    primaryClass = "hep-lat",
    month = "12",
    year = "2025"
}

@inproceedings{Levy:2017bji,
    author = "Levy, Daniel and Hoffman, Matthew D. and Sohl-Dickstein, Jascha",
    title = "{Generalizing Hamiltonian Monte Carlo with Neural Networks}",
    eprint = "1711.09268",
    archivePrefix = "arXiv",
    primaryClass = "stat.ML",
    month = "11",
    year = "2018",
    booktitle = "International Conference on Learning Representations (ICLR)"
}

@misc{hoffman2019neutralizingbadgeometryhamiltonian,
      title={{NeuTra}-lizing Bad Geometry in {H}amiltonian {M}onte {C}arlo Using Neural Transport}, 
      author={Matthew Hoffman and Pavel Sountsov and Joshua V. Dillon and Ian Langmore and Dustin Tran and Srinivas Vasudevan},
      year={2019},
      eprint={1903.03704},
      archivePrefix={arXiv},
      primaryClass={stat.CO},
      url={https://arxiv.org/abs/1903.03704}, 
}

@InProceedings{pmlr-v37-salimans15,
  title = 	 {{M}arkov Chain {M}onte {C}arlo and Variational Inference: Bridging the Gap},
  author = 	 {Salimans, Tim and Kingma, Diederik and Welling, Max},
  booktitle = 	 {Proceedings of the 32nd International Conference on Machine Learning},
  pages = 	 {1218--1226},
  year = 	 {2015},
  editor = 	 {Bach, Francis and Blei, David},
  volume = 	 {37},
  series = 	 {Proceedings of Machine Learning Research},
  address = 	 {Lille, France},
  month = 	 {07--09 Jul},
  publisher =    {PMLR},
  url = 	 {https://proceedings.mlr.press/v37/salimans15.html},
}

@article{Jarzynski:1996oqb,
    author = "Jarzynski, C.",
    title = "{Nonequilibrium Equality for Free Energy Differences}",
    eprint = "cond-mat/9610209",
    archivePrefix = "arXiv",
    doi = "10.1103/PhysRevLett.78.2690",
    journal = "Phys. Rev. Lett.",
    volume = "78",
    number = "14",
    pages = "2690--2693",
    year = "1997"
}

@article{Crooks:1999ttd,
    author = "Crooks, Gavin E.",
    title = "{Entropy production fluctuation theorem and the nonequilibrium work relation for free energy differences}",
    doi = "10.1103/PhysRevE.60.2721",
    journal = "Phys. Rev. E",
    volume = "60",
    number = "3",
    pages = "2721",
    year = "1999"
}

@inproceedings{NEURIPS2020_41d80bfc,
 author = {Wu, Hao and K\"{o}hler, Jonas and No{\'e}, Frank},
 booktitle = {Advances in Neural Information Processing Systems},
 editor = {H. Larochelle and M. Ranzato and R. Hadsell and M.F. Balcan and H. Lin},
 pages = {5933--5944},
 publisher = {Curran Associates, Inc.},
 title = {Stochastic Normalizing Flows},
 url = {https://proceedings.neurips.cc/paper_files/paper/2020/file/41d80bfc327ef980528426fc810a6d7a-Paper.pdf},
 volume = {33},
 year = {2020}
}

@article{Caselle:2022acb,
    author = "Caselle, Michele and Cellini, Elia and Nada, Alessandro and Panero, Marco",
    title = "{Stochastic normalizing flows as non-equilibrium transformations}",
    eprint = "2201.08862",
    archivePrefix = "arXiv",
    primaryClass = "hep-lat",
    doi = "10.1007/JHEP07(2022)015",
    journal = "JHEP",
    volume = "2022",
    number = "07",
    pages = "015",
    year = "2022"
}

@article{Bonanno:2024udh,
    author = "Bonanno, Claudio and Nada, Alessandro and Vadacchino, Davide",
    title = "{Mitigating topological freezing using out-of-equilibrium simulations}",
    eprint = "2402.06561",
    archivePrefix = "arXiv",
    primaryClass = "hep-lat",
    doi = "10.1007/JHEP04(2024)126",
    journal = "JHEP",
    volume = "2024",
    number = "04",
    pages = "126",
    year = "2024"
}

@article{Bulgarelli:2024brv,
    author = "Bulgarelli, Andrea and Cellini, Elia and Nada, Alessandro",
    title = "{Scaling of stochastic normalizing flows in SU(3) lattice gauge theory}",
    eprint = "2412.00200",
    archivePrefix = "arXiv",
    primaryClass = "hep-lat",
    doi = "10.1103/PhysRevD.111.074517",
    journal = "Phys. Rev. D",
    volume = "111",
    number = "7",
    pages = "074517",
    year = "2025"
}

@article{Chen:2026bqw,
    author = "Chen, Shiyang and Qian, Moxian and Aarts, Gert and Lucini, Biagio and Zhou, Kai",
    title = "{Stochastic path sampler for lattice field theory}",
    eprint = "2606.13790",
    archivePrefix = "arXiv",
    primaryClass = "hep-lat",
    doi = "10.1007/JHEP11(2026)164",
    journal = "JHEP",
    volume = "2026",
    number = "11",
    pages = "164",
    year = "2026"
}

@article{Nilmeier_2011, title={Nonequilibrium candidate {M}onte {C}arlo is an efficient tool for equilibrium simulation}, volume={108}, ISSN={1091-6490}, url={http://dx.doi.org/10.1073/pnas.1106094108}, DOI={10.1073/pnas.1106094108}, number={45}, journal={Proc. Natl. Acad. Sci. U.S.A.}, publisher={National Academy of Sciences}, author={Nilmeier, Jerome P. and Crooks, Gavin E. and Minh, David D. L. and Chodera, John D.}, year={2011}, month=Oct, pages={E1009--E1018} }

@article{Neal_1996, title={Sampling from multimodal distributions using tempered transitions}, volume={6}, ISSN={1573-1375}, url={http://dx.doi.org/10.1007/BF00143556}, DOI={10.1007/bf00143556}, number={4}, journal={Stat. Comput.}, publisher={Springer Science and Business Media LLC}, author={Neal, Radford M.}, year={1996}, month=Dec, pages={353--366} }

@InProceedings{pmlr-v119-neklyudov20a,
  title = 	 {Involutive {MCMC}: a Unifying Framework},
  author =       {Neklyudov, Kirill and Welling, Max and Egorov, Evgenii and Vetrov, Dmitry},
  booktitle = 	 {Proceedings of the 37th International Conference on Machine Learning},
  pages = 	 {7273--7282},
  year = 	 {2020},
  editor = 	 {Daum{\'e} III, Hal and Singh, Aarti},
  volume = 	 {119},
  series = 	 {Proceedings of Machine Learning Research},
  month = 	 {13--18 Jul},
  publisher =    {PMLR},
  url = 	 {https://proceedings.mlr.press/v119/neklyudov20a.html},
}

@inproceedings{
vargas2024transport,
title={Transport meets Variational Inference: Controlled {M}onte {C}arlo Diffusions},
author={Francisco Vargas and Shreyas Padhy and Denis Blessing and Nikolas N{\"u}sken},
booktitle={The Twelfth International Conference on Learning Representations},
year={2024},
url={https://openreview.net/forum?id=PP1rudnxiW}
}

@inproceedings{
richter2024improved,
title={Improved sampling via learned diffusions},
author={Lorenz Richter and Julius Berner},
booktitle={The Twelfth International Conference on Learning Representations},
year={2024},
url={https://openreview.net/forum?id=h4pNROsO06}
}

@article{Binder:1981sa,
    author = "Binder, K.",
    title = "{Finite size scaling analysis of Ising model block distribution functions}",
    doi = "10.1007/BF01293604",
    journal = "Z. Phys. B",
    volume = "43",
    pages = "119--140",
    year = "1981"
}

@article{Wolff:2003sm,
    author = "Wolff, Ulli",
    collaboration = "ALPHA",
    title = "{Monte Carlo errors with less errors}",
    eprint = "hep-lat/0306017",
    archivePrefix = "arXiv",
    reportNumber = "HU-EP-03-32, SFB-CPP-03-12",
    doi = "10.1016/S0010-4655(03)00467-3",
    journal = "Comput. Phys. Commun.",
    volume = "156",
    pages = "143--153",
    year = "2004",
    note = "[Erratum: Comput.Phys.Commun. 176, 383 (2007)]"
}

@article{Madras:1988ei,
    author = "Madras, Neal and Sokal, Alan D.",
    title = "{The pivot algorithm: A highly efficient Monte Carlo method for the self-avoiding walk}",
    doi = "10.1007/BF01022990",
    journal = "J. Statist. Phys.",
    volume = "50",
    pages = "109--186",
    year = "1988"
}

@article{Flyvbjerg_1989,
    author = {Flyvbjerg, H. and Petersen, H. G.},
    title = {Error estimates on averages of correlated data},
    journal = {J. Chem. Phys.},
    volume = {91},
    number = {1},
    pages = {461--466},
    year = {1989},
    month = {07},
    issn = {0021-9606},
    doi = {10.1063/1.457480},
    url = {https://doi.org/10.1063/1.457480}
}

@article{Brower:1989mt,
    author = "Brower, R. C. and Tamayo, P.",
    title = "{Embedded Dynamics for $\phi^4$ Theory}",
    doi = "10.1103/PhysRevLett.62.1087",
    journal = "Phys. Rev. Lett.",
    volume = "62",
    pages = "1087--1090",
    year = "1989"
}

@article{Andrieu_2009,
author = {Christophe Andrieu and Gareth O. Roberts},
title = {{The pseudo-marginal approach for efficient Monte Carlo computations}},
volume = {37},
journal = {Ann. Statist.},
number = {2},
publisher = {Institute of Mathematical Statistics},
pages = {697--725},
year = {2009},
doi = {10.1214/07-AOS574},
URL = {https://doi.org/10.1214/07-AOS574}
}

@article{Doucet_2015, title={Efficient implementation of {M}arkov chain {M}onte {C}arlo when using an unbiased likelihood estimator}, volume={102}, ISSN={1464-3510}, url={http://dx.doi.org/10.1093/biomet/asu075}, DOI={10.1093/biomet/asu075}, number={2}, journal={Biometrika}, publisher={Oxford University Press (OUP)}, author={Doucet, A. and Pitt, M. K. and Deligiannidis, G. and Kohn, R.}, year={2015}, month=Mar, pages={295--313} }

@article{Creutz:1988wv,
    author = "Creutz, M.",
    title = "{Global Monte Carlo algorithms for many-fermion systems}",
    doi = "10.1103/PhysRevD.38.1228",
    journal = "Phys. Rev. D",
    volume = "38",
    pages = "1228--1238",
    year = "1988"
}

@article{Gupta:1988js,
    author = "Gupta, Rajan and Kilcup, Gregory W. and Sharpe, Stephen R.",
    title = "{Tuning the Hybrid Monte Carlo Algorithm}",
    reportNumber = "SLAC-PUB-4587, LA-UR-88-886",
    doi = "10.1103/PhysRevD.38.1278",
    journal = "Phys. Rev. D",
    volume = "38",
    pages = "1278",
    year = "1988"
}

@article{Bennett_1976, title={Efficient estimation of free energy differences from {M}onte {C}arlo data}, volume={22}, ISSN={0021-9991}, url={http://dx.doi.org/10.1016/0021-9991(76)90078-4}, DOI={10.1016/0021-9991(76)90078-4}, number={2}, journal={J. Comput. Phys.}, publisher={Elsevier BV}, author={Bennett, Charles H}, year={1976}, month=Oct, pages={245--268} }

@article{Shirts_2003, title={Equilibrium Free Energies from Nonequilibrium Measurements Using Maximum-Likelihood Methods}, volume={91}, ISSN={1079-7114}, url={http://dx.doi.org/10.1103/PhysRevLett.91.140601}, DOI={10.1103/physrevlett.91.140601}, number={14}, journal={Phys. Rev. Lett.}, publisher={American Physical Society (APS)}, author={Shirts, Michael R. and Bair, Eric and Hooker, Giles and Pande, Vijay S.}, year={2003}, month=Oct, pages={140601} }

@article{Meng:1996ss,
  author  = {Meng, Xiao-Li and Wong, Wing Hung},
  title   = {Simulating ratios of normalizing constants via a simple identity: A theoretical exploration},
  journal = {Stat. Sin.},
  volume  = {6},
  number  = {4},
  pages   = {831--860},
  year    = {1996},
  url     = {https://www3.stat.sinica.edu.tw/statistica/j6n4/j6n43/j6n43.htm}
}

@article{Neal:2001hwj,
    author = "Neal, Radford M.",
    title = "{Annealed importance sampling}",
    doi = "10.1023/A:1008923215028",
    journal = "Stat. Comput.",
    volume = "11",
    number = "2",
    pages = "125--139",
    year = "2001"
}

@misc{chen2026markovchainmontecarlo,
      title={{M}arkov Chain {M}onte {C}arlo with Diffusion Paths}, 
      author={Han Chen and Sifan Liu and Jun Yang},
      year={2026},
      eprint={2607.11631},
      archivePrefix={arXiv},
      primaryClass={stat.CO},
      url={https://arxiv.org/abs/2607.11631}, 
}

@inproceedings{Kingma:2014vow,
    author = "Kingma, Diederik P. and Ba, Jimmy",
    title = "{Adam: A Method for Stochastic Optimization}",
    booktitle = "{International Conference on Learning Representations}",
    eprint = "1412.6980",
    archivePrefix = "arXiv",
    primaryClass = "cs.LG",
    month = "12",
    year = "2015"
}

@inproceedings{
blessing2025underdamped,
title={Underdamped Diffusion Bridges with Applications to Sampling},
author={Denis Blessing and Julius Berner and Lorenz Richter and Gerhard Neumann},
booktitle={The Thirteenth International Conference on Learning Representations},
year={2025},
url={https://openreview.net/forum?id=Q1QTxFm0Is}
}

@article{JMLR:v25:19-556,
  author  = {Aki Vehtari and Daniel Simpson and Andrew Gelman and Yuling Yao and Jonah Gabry},
  title   = {{P}areto Smoothed Importance Sampling},
  journal = {J. Mach. Learn. Res.},
  year    = {2024},
  volume  = {25},
  number  = {72},
  pages   = {1--58},
  url     = {http://jmlr.org/papers/v25/19-556.html}
}

@article{Li:2018nrg,
    author = "Li, Shuo-Hui and Wang, Lei",
    title = "{Neural Network Renormalization Group}",
    eprint = "1802.02840",
    archivePrefix = "arXiv",
    journal = "Phys. Rev. Lett.",
    volume = "121",
    pages = "260601",
    doi = "10.1103/PhysRevLett.121.260601",
    year = "2018"
}

@article{Li:2019ncf,
    author = "Li, Shuo-Hui and Dong, Chen-Xiao and Zhang, Linfeng and Wang, Lei",
    title = "{Neural Canonical Transformation with Symplectic Flows}",
    eprint = "1910.00024",
    archivePrefix = "arXiv",
    journal = "Phys. Rev. X",
    volume = "10",
    pages = "021020",
    doi = "10.1103/PhysRevX.10.021020",
    year = "2020"
}

@misc{Finkenrath:2022ogg,
    author = "Finkenrath, Jacob",
    title = "{Tackling critical slowing down using global correction steps with equivariant flows: the case of the Schwinger model}",
    eprint = "2201.02216",
    archivePrefix = "arXiv",
    primaryClass = "hep-lat",
    year = "2022"
}

@article{girolami2011riemann,
title={Riemann manifold {L}angevin and {H}amiltonian {M}onte {C}arlo methods},
author={Mark Girolami and Ben Calderhead},
journal={J. R. Stat. Soc. B},
volume={73},
pages={123--214},
year={2011},
doi={10.1111/j.1467-9868.2010.00765.x}
}

@inproceedings{Foreman:2021ixr,
    author = "Foreman, Sam and Jin, Xiao-Yong and Osborn, James C.",
    title = "{Deep Learning Hamiltonian Monte Carlo}",
    booktitle = "{ICLR 2021 Workshop on Science and Engineering of Deep Learning (SimDL)}",
    eprint = "2105.03418",
    archivePrefix = "arXiv",
    primaryClass = "hep-lat",
    month = "5",
    year = "2021"
}

@article{Foreman:2021rhs,
    author = "Foreman, Sam and Jin, Xiao-Yong and Osborn, James C.",
    title = "{LeapfrogLayers: A Trainable Framework for Effective Topological Sampling}",
    eprint = "2112.01582",
    archivePrefix = "arXiv",
    primaryClass = "hep-lat",
    doi = "10.22323/1.396.0508",
    journal = "PoS",
    volume = "LATTICE2021",
    pages = "508",
    year = "2022"
}

@inproceedings{geffner2021mcmc,
title={{MCMC} Variational Inference via Uncorrected {H}amiltonian Annealing},
author={Tomas Geffner and Justin Domke},
booktitle={Advances in Neural Information Processing Systems},
volume={34},
pages={639--651},
year={2021},
eprint={2107.04150},
archivePrefix={arXiv}
}

@inproceedings{zhang2021differentiable,
title={Differentiable Annealed Importance Sampling and the Perils of Gradient Noise},
author={Guodong Zhang and Kyle Hsu and Jianing Li and Chelsea Finn and Roger Grosse},
booktitle={Advances in Neural Information Processing Systems},
volume={34},
pages={19398--19410},
year={2021},
eprint={2107.10211},
archivePrefix={arXiv}
}

@article{Kennedy:1990bv,
    author = "Kennedy, A. D. and Pendleton, Brian",
    title = "{Acceptances and autocorrelations in hybrid Monte Carlo}",
    journal = "Nucl. Phys. B Proc. Suppl.",
    volume = "20",
    pages = "118--121",
    year = "1991",
    doi = "10.1016/0920-5632(91)90893-J"
}

@article{Syed:2024oasmc,
    author = "Syed, Saifuddin and Bouchard-C{\^o}t{\'e}, Alexandre and Chern, Kevin and Doucet, Arnaud",
    title = "{Optimized annealed sequential Monte Carlo samplers}",
    eprint = "2408.12057",
    archivePrefix = "arXiv",
    primaryClass = "stat.CO",
    doi = "10.1093/jrsssb/qkag082",
    journal = "J. R. Stat. Soc. B",
    volume = "{\unskip}",
    pages = "qkag082",
    year = "2026"
}
\clearpage
\onecolumngrid
\renewenvironment{widefigure}[1][]{\begin{figure}[!htbp]\centering}{\end{figure}}
\renewenvironment{widetable}[1][]{\begin{table}[H]\centering}{\end{table}}
\setcounter{section}{0}\setcounter{figure}{0}\setcounter{table}{0}\setcounter{equation}{0}
\renewcommand{\thesection}{S\arabic{section}}
\renewcommand{\thesubsection}{\thesection.\arabic{subsection}}
\makeatletter\renewcommand{\p@subsection}{}\renewcommand{\p@subsubsection}{}\makeatother
\renewcommand{\thefigure}{S\arabic{figure}}
\renewcommand{\thetable}{S\arabic{table}}
\renewcommand{\theequation}{S\arabic{equation}}
\renewcommand{\appendixname}{}
\begin{center}\textbf{\large Supplemental Material}\end{center}
\label{sm:start}
\suppressfloats[t]

Sections S1 to S8 give the training and sampling times and the equilibrium observables, followed by further results on the native models, the scaling relations, the magnetization autocorrelations, model transfer, partition functions, CMCD and DBS, and the flow between the two magnetization-sign regions.

\section{Learned paths and equilibrium observables}
\label{sm:paths}

\subsection{Training and sampling time}
\label{app:timing}
All times in Table~\ref{tab:timing} were measured on one device, a HIMster III compute node (Intel Xeon 6972P), with PyTorch 2.13 and four threads. The models are at $\kappa=0.27$ with five leapfrog substeps per stage. The rows with $T=16$ use the native models, and at $L=8$ and $16$ further models with $T=4$ to $64$ and $T=8$ to $64$ are included. Training time is one epoch at batch size $48$, given before and after the buffer reverse-path training term (Sec.~\ref{sec:nhmc_correction}) is switched on at epoch $1000$. Sampling time is one update of four chains processed as one batch. The forward-path column gives the average time to generate one complete forward path in batches of $1024$, the computation needed for SNIS and for path-IMH proposals.

The time per update grows in proportion to $T$ for path-IMH and full round trips. A full round trip takes about twice as long as a path-IMH update and requires exactly twice as many $-\nabla S$ evaluations. For truncated round trips with fixed return depth $d=4$, the time per update remains essentially constant as $T$ increases. The number of $-\nabla S$ evaluations per chain and update does not depend on the hardware. Each stage has five drifts and six kicks, so path-IMH needs $6T$, the full round trip $12T$ and the truncated round trip $12d$. With only four chains, fixed overhead dominates the time per update. Increasing the lattice size from $L=4$ to $16$ increases this time by only a factor of about $1.5$. When forward paths are generated in batches, the time per path increases with lattice volume.

\begin{table}[H]
 \centering
 \scriptsize
 \setlength{\tabcolsep}{4pt}
 \renewcommand{\arraystretch}{1.17}
 \caption{Training and sampling time on one device ($\kappa=0.27$). The training columns give seconds per training epoch before and after the buffer is switched on. After that the round-trip chains are updated every two epochs, and we average over two epochs. Sampling time is given in milliseconds per update of four chains, median of $200$ updates after $20$ warm-up updates. pRT is given for $d=T/4$ (the setting of the main text) and $d=4$. The two coincide at $T=16$, and at $T=4$ the setting $d=4$ is the full round trip. Forward-path time is given in milliseconds per path in batches of $1024$.}
 \label{tab:timing}
 \begin{tabular}{@{}rrrrrrrrr@{}}
 \toprule
 $L$ & $T$ & \multicolumn{2}{c}{training (s/epoch)} & \multicolumn{4}{c}{sampling (ms/update)} & forward path \\
 \cmidrule(lr){3-4}\cmidrule(lr){5-8}
 & & before buffer & buffer phase & path-IMH & RT & pRT $d=T/4$ & pRT $d=4$ & (ms/path) \\
 \midrule
4 & 16 & 1.04 & 2.40 & 141 & 282 & 70 & 70 & 1.94 \\
\addlinespace
8 & 4 & 0.34 & 0.78 & 40 & 79 & 20 & 79 & 1.02 \\
8 & 8 & 0.69 & 1.64 & 83 & 165 & 42 & 82 & 1.99 \\
8 & 16 & 1.44 & 3.43 & 165 & 325 & 81 & 81 & 4.21 \\
8 & 32 & 3.13 & 7.77 & 313 & 616 & 155 & 77 & 8.43 \\
8 & 64 & 7.30 & 18.52 & 651 & 1281 & 321 & 81 & 15.18 \\
\addlinespace
12 & 16 & 2.19 & 5.14 & 183 & 363 & 91 & 91 & 8.29 \\
\addlinespace
16 & 8 & 1.37 & 3.24 & 112 & 223 & 56 & 111 & 5.56 \\
16 & 16 & 2.92 & 7.26 & 213 & 427 & 106 & 106 & 11.64 \\
16 & 32 & 6.78 & 17.76 & 436 & 873 & 218 & 109 & 21.30 \\
16 & 64 & 16.48 & 45.97 & 848 & 1616 & 404 & 102 & 42.28 \\
 \bottomrule
 \end{tabular}
\end{table}

\subsection{Equilibrium observables}
\label{app:phi4_numerics}

Figures~\ref{fig:phi4_raw_snis} and \ref{fig:phi4_equilibrium} extend Fig.~\ref{fig:sampling}(a,b). The first figure shows the raw and SNIS deviations of $\langle|m|\rangle$, $\chi$ and $U_4$ for the same forward paths together with the effective sample fraction. The second figure shows the coupling and size dependence of the three observables for the three corrected chains and their standardized deviations. The susceptibility peaks where the magnetization changes fastest, more strongly on larger lattices, and $U_4$ rises over a similar coupling interval. Nearly all standardized deviations of the three updates satisfy $|z|\le2$. The largest, about $3.3$, comes from path-IMH at $L=8$, $\kappa=0.22$.

\begin{widefigure}
 \centering
 \includegraphics[width=0.98\textwidth,keepaspectratio]{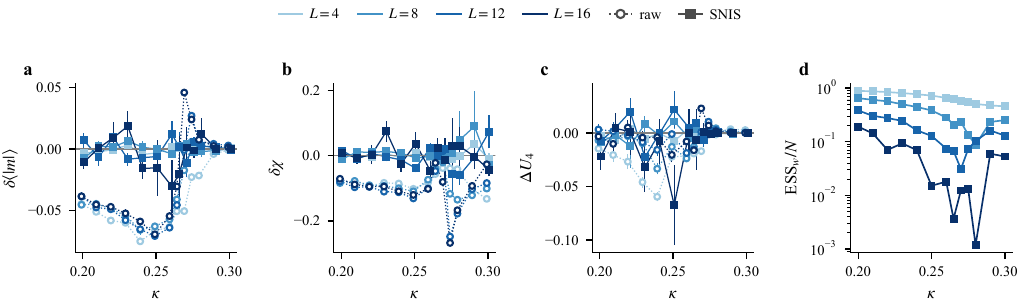}
 \caption{Independent-path estimates from models trained separately at each $(L,\kappa)$.
 At each parameter point, the same $65536$ forward paths are used to compare equally
 weighted endpoint averages (raw, open circles) with SNIS (filled squares).
 (a), (b) show the relative deviations $\delta O=(O-O_{\rm ref})/O_{\rm ref}$ of $\langle|m|\rangle$ and $\chi$ from the
 HMC and Wolff reference, (c) the difference $\Delta U_4=U_4-U_{4,\mathrm{ref}}$, and
 (d) the effective sample fraction of Eq.~\eqref{eq:phi4_snis}.
 We split the $65536$ paths into $20$ blocks for the jackknife.
 Error bars do not include the error of the reference.}
 \label{fig:phi4_raw_snis}
\end{widefigure}

\begin{widefigure}
 \centering
 \includegraphics[width=0.96\textwidth,keepaspectratio]{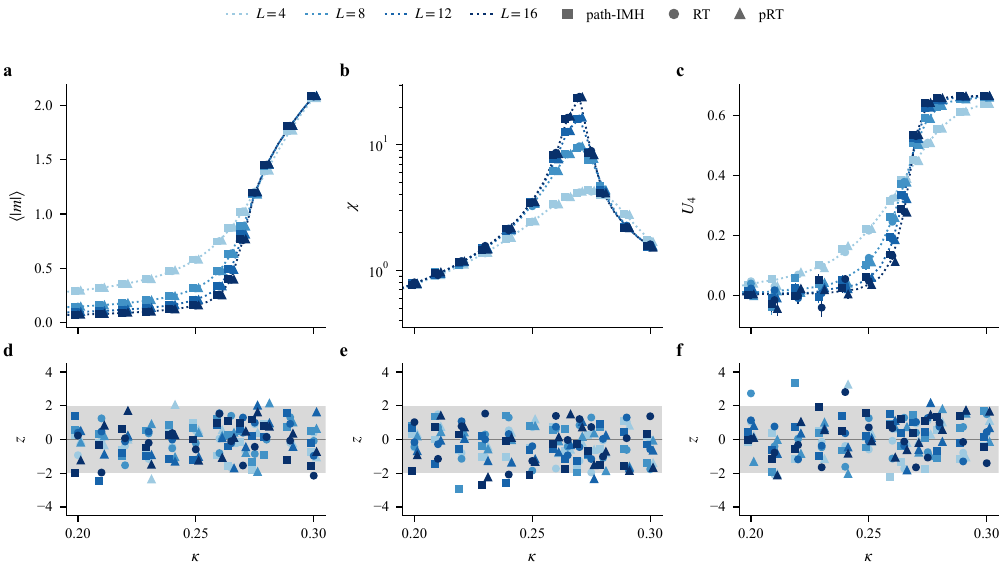}
 \caption{Coupling dependence of equilibrium observables in two-dimensional
 $\phi^4$ theory. (a) to (c) show $\langle|m|\rangle$, $\chi$, and $U_4$, respectively.
 Each $(L,\kappa)$ uses a separately trained model.
 Error bars are jackknife standard errors with each chain divided into $20$ blocks.
 Dotted curves (colored by $L$) are independent HMC and Wolff references combined by inverse-variance
 weighting. Point colors also indicate $L$ and shapes the method.
 (d) to (f) show $z=(\widehat O-O_{\rm ref})/(\mathrm{SE}_O^2+\mathrm{SE}_{\rm ref}^2)^{1/2}$,
 and gray bands indicate $|z|\leq2$.}
 \label{fig:phi4_equilibrium}
\end{widefigure}

\section{Supplementary results for native models}
\label{sm:native_extra}

\subsection{Acceptance, accepted amplitude changes and the magnetization distribution}

\begin{widefigure}[tb]
\centering
\includegraphics[width=0.98\textwidth,keepaspectratio]{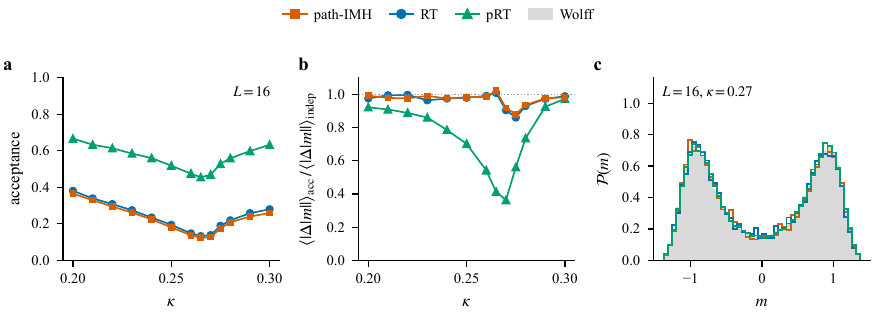}
\caption{Acceptance, accepted changes in magnetization amplitude, and the magnetization distribution at $L=16$ (one model trained at each coupling, pRT with $s=12$).
(a) Acceptance rates of path-IMH (small squares), RT, and pRT. (b) The mean absolute change in magnetization amplitude over accepted updates,
$\langle\bigl||m_{n+1}|-|m_n|\bigr|\rangle_{\mathrm{acc}}$,
divided by a baseline labeled indep on the axis. For the baseline we pool the $|m|$ records of the four chains, shuffle them, and take the mean absolute difference between the original and shuffled sequences. A ratio of one means that an accepted update changes $|m|$ as much as two independent samples differ.
The dotted line marks a ratio of one.
(c) The signed magnetization distribution $\mathcal P(m)$ of the path-IMH, RT, and pRT chains at $\kappa=0.27$, with a Wolff reference.}
\label{fig:nhmc_acceptance}
\end{widefigure}

Figure~\ref{fig:nhmc_acceptance} compares the acceptance rates and the accepted changes in magnetization amplitude of the three updates at $L=16$.
It also compares the signed magnetization distributions of the three corrected chains at $\kappa=0.27$, with Wolff sampling as the reference.
At $L=16$ and $\kappa=0.27$, the effective sample fraction of the forward paths is small (Fig.~\ref{fig:sampling}(a)). The acceptance of path-IMH is nevertheless not low and is close to that of full round trips, $0.13$ and $0.14$, respectively. For both methods, the mean absolute change in $|m|$ per accepted update is close to the random baseline. Truncated round trips have higher acceptance. Near $\kappa=0.27$, however, their mean absolute change per accepted update is well below this baseline (Fig.~\ref{fig:nhmc_acceptance}(a,b)).

The observables $\langle|m|\rangle$, $\chi$ and $U_4$ are invariant under a sign flip. We therefore use the signed magnetization distribution to test whether the chains visit the two magnetization regions with the correct weights. The distributions of all three corrected chains agree with the Wolff reference (Fig.~\ref{fig:nhmc_acceptance}(c)).

\subsection{Return-stage scan}

\begin{widefigure}[tb]
 \centering
 \includegraphics[width=0.86\textwidth,keepaspectratio]{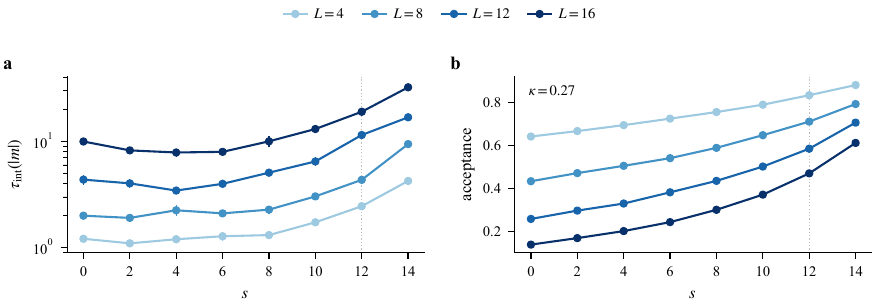}
 \caption{Scan of the return stage $s$ at $\kappa=0.27$ and total stage count $T=16$.
One model is trained at each $(L,\kappa)$. Full round trips have $s=0$, and $s=12$ redraws only the last four stages.
(a) $\tau_{\mathrm{int}}(|m|)$ in units of one round-trip update, (b) acceptance rate. Statistics as in Appendix~\ref{app:implementation}.
The light gray dotted vertical line marks $s=12$, used in the other comparisons.}
 \label{fig:nhmc_turning_scan}
\end{widefigure}

Figure~\ref{fig:nhmc_turning_scan} compares the dependence of the autocorrelation
time and acceptance rate on the return stage $s$ at $\kappa=0.27$.
The acceptance rate rises with $s$. The amplitude autocorrelation time is smallest
for small $s$ close to the start of the path (Fig.~\ref{fig:nhmc_turning_scan}). Supplemental Sec.~\ref{sm:autocorr} repeats the scan at $\kappa=0.24$ and $0.30$. There the autocorrelation time varies little at larger $s$, and at $L\ge12$ the value at $s=12$ is clearly shorter than that of full round trips.
The acceptance formula of Sec.~\ref{sec:scaling} explains the higher acceptance of truncated round trips. At $s=12$ a truncated round trip redraws only the last $d=4$ stages. Its conditional work variance is much smaller than that of a full round trip, so its acceptance is higher. Near criticality, accepted pRT updates also produce smaller changes in $|m|$.

\section{Supplementary results on scaling}
\label{sm:scaling}

\begin{table}[tb]
\caption{Per-site conditional work variance $\sigma_s^2/V$ of the frozen $L_{\rm src}=16$ model on $L\times L$ lattices, measured on forward paths. The third column gives the value at $L=16$ (in units of $10^{-3}$), and the other columns give the ratio to $L=16$. Each point uses $4096$ continuation paths, $16$ of which share one intermediate field.}
\label{tab:condvar_stages}
\centering
\begin{tabular}{rrrrrrr}
\toprule
$\kappa$ & $s$ & $L=16$ & $24$ & $32$ & $48$ & $64$\\
\midrule
$0.24$ & $0$ & $10.71$ & $1.01$ & $1.00$ & $1.00$ & $0.99$\\
$0.24$ & $4$ & $7.92$ & $1.03$ & $1.02$ & $1.03$ & $1.01$\\
$0.24$ & $8$ & $5.49$ & $0.97$ & $0.95$ & $0.98$ & $0.95$\\
$0.24$ & $12$ & $2.67$ & $1.00$ & $0.97$ & $1.00$ & $0.97$\\
$0.27$ & $0$ & $16.04$ & $1.15$ & $1.30$ & $1.67$ & $1.90$\\
$0.27$ & $4$ & $11.52$ & $1.15$ & $1.26$ & $1.44$ & $1.51$\\
$0.27$ & $8$ & $7.78$ & $1.05$ & $1.07$ & $1.20$ & $1.18$\\
$0.27$ & $12$ & $3.87$ & $1.04$ & $1.06$ & $1.16$ & $1.18$\\
$0.28$ & $0$ & $11.54$ & $1.24$ & $1.67$ & $3.07$ & $5.09$\\
$0.28$ & $4$ & $8.26$ & $1.23$ & $1.53$ & $2.44$ & $4.21$\\
$0.28$ & $8$ & $5.20$ & $0.99$ & $1.01$ & $1.42$ & $1.37$\\
$0.28$ & $12$ & $2.50$ & $1.01$ & $0.99$ & $1.03$ & $1.02$\\
$0.30$ & $0$ & $9.19$ & $1.31$ & $2.09$ & $4.64$ & $6.95$\\
$0.30$ & $4$ & $6.49$ & $1.49$ & $1.86$ & $2.40$ & $5.49$\\
$0.30$ & $8$ & $3.73$ & $0.96$ & $0.95$ & $1.01$ & $0.91$\\
$0.30$ & $12$ & $1.77$ & $0.99$ & $0.98$ & $1.02$ & $0.99$\\
\bottomrule
\end{tabular}
\end{table}

Table~\ref{tab:condvar_stages} gives the per-site conditional work variance of the frozen model at several return stages $s$ as a function of the target size. It is the stage-resolved version of Fig.~\ref{fig:boundary}(a) in the main text.

Table~\ref{tab:workshape} tests the shape of the conditional work distribution in the broken phase. We take $1024$ configurations from truncated round-trip chains ($s=12$, $64$ chains, one configuration saved every $10$ updates after $500$ thermalization updates), walk each back along a reverse path to $s=0$ to obtain the intermediate field and the old record, and generate $16$ new records forward from each intermediate field. The direct acceptance is the mean of $\min(1,e^{W_{\rm old}-W_{\rm new}})$. In the broken phase at $L=32$ the excess kurtosis is clearly above zero, a few intermediate fields carry a sizable part of the variance, and the Gaussian prediction from the mean variance lies far below the direct acceptance.

\begin{table}[tb]
\caption{Conditional work distribution of the frozen $L_{\rm src}=16$ model in full round trips ($s=0$). $\bar\alpha_{\rm dir}$ is the acceptance computed directly from chain samples. The columns mean and median give $\operatorname{erfc}(\sigma/2)$ with the mean and the median of the conditional variances over intermediate fields. $\gamma_2$ is the excess kurtosis of the within-field residuals, and the last column is the share of the variance carried by the $1\%$ of intermediate fields with the largest variance.}
\label{tab:workshape}
\centering
\begin{tabular}{rrrrrrr}
\toprule
$\kappa$ & $L$ & $\bar\alpha_{\rm dir}$ & mean & median & $\gamma_2$ & $1\%$ share\\
\midrule
$0.24$ & $16$ & $0.235(7)$ & $0.238$ & $0.250$ & $0.0$ & $0.02$\\
$0.24$ & $32$ & $0.019(2)$ & $0.019$ & $0.024$ & $0.1$ & $0.02$\\
$0.27$ & $16$ & $0.145(6)$ & $0.150$ & $0.163$ & $0.1$ & $0.02$\\
$0.27$ & $32$ & $0.0019(4)$ & $0.0013$ & $0.0017$ & $0.2$ & $0.02$\\
$0.28$ & $16$ & $0.229(7)$ & $0.226$ & $0.243$ & $1.0$ & $0.03$\\
$0.28$ & $32$ & $0.017(2)$ & $0.0027$ & $0.0069$ & $9.5$ & $0.08$\\
$0.30$ & $16$ & $0.290(8)$ & $0.277$ & $0.303$ & $3.0$ & $0.05$\\
$0.30$ & $32$ & $0.037(3)$ & $0.0045$ & $0.017$ & $51$ & $0.15$\\
\bottomrule
\end{tabular}
\end{table}

\begin{figure}[tb]
\centering
\includegraphics[width=0.5\linewidth,keepaspectratio]{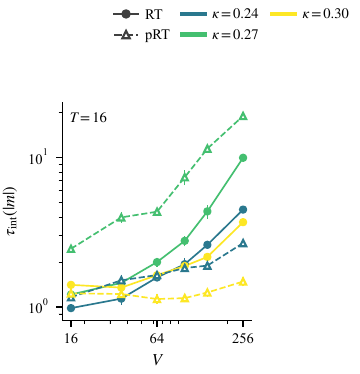}
\caption{$\tau_{\rm int}(|m|)$ of native models at $T=16$ versus volume $V$ (one model trained at each $(L,\kappa)$). Solid lines are full and dashed lines truncated round trips ($s=12$), with colors for $\kappa=0.24,0.27,0.30$.}
\label{fig:boundary_tau_sm}
\end{figure}

At a fixed number of stages the autocorrelation time of native models grows exponentially with volume (Fig.~\ref{fig:boundary_tau_sm}). Figure~\ref{fig:scaling_accept} adds frozen models to the comparison of predicted and measured acceptance. For the frozen models we take $\sigma_s^2$ as the per-site value of the native $L=16$ model times the target volume. Figure~\ref{fig:scaling_var} gives the full scaling of the conditional work variance. $\sigma_s^2$ grows linearly with the depth $d$, each stage contributes about equally, and the variance of a truncated round trip is therefore about $d/T$ times that of a full round trip (Fig.~\ref{fig:scaling_var}(b)). Figure~\ref{fig:scaling_var}(c) also shows standard nonequilibrium HMC without learning (NE-HMC). Figure~\ref{fig:scaling_tau} gives $\tau_{\rm int}(|m|)$ versus volume (with HMC at the same coupling), the acceptance after transfer, and $\tau_{\rm int}$ versus $1/\bar\alpha-1/2$ including frozen models.

\begin{widefigure}
\centering
\includegraphics[width=0.9\linewidth,keepaspectratio]{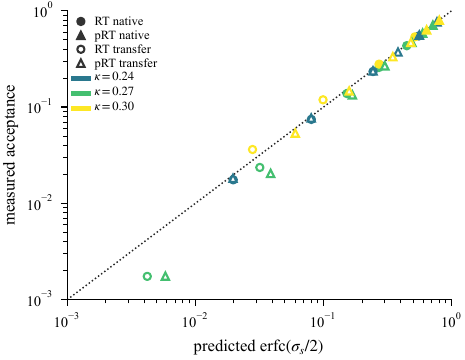}
\caption{Predicted versus measured round-trip acceptance. The horizontal axis is $\operatorname{erfc}(\sigma_s/2)$, where $\sigma_s^2$ is the directly measured conditional work variance at a shared intermediate field ($s=0$ for full and $s=12$ for truncated round trips). Filled symbols are the seven models trained directly at the target size (native models), $\kappa=0.24,0.27,0.30$ at $L=8$ and $16$ and $\kappa=0.27$ at $L=12$. Open symbols are frozen models, $L=16$ models used with frozen parameters at larger sizes (full round trips at $L=24,32$, truncated round trips at $L=24$ to $64$). For these $\sigma_s^2$ is the per-site value of the native $L=16$ model times the target volume $V$. Error bars are block-jackknife errors. Colors indicate the coupling. The dotted line is $y=x$.}
\label{fig:scaling_accept}
\end{widefigure}

\begin{widefigure}
\centering
\includegraphics[width=0.9\textwidth,keepaspectratio]{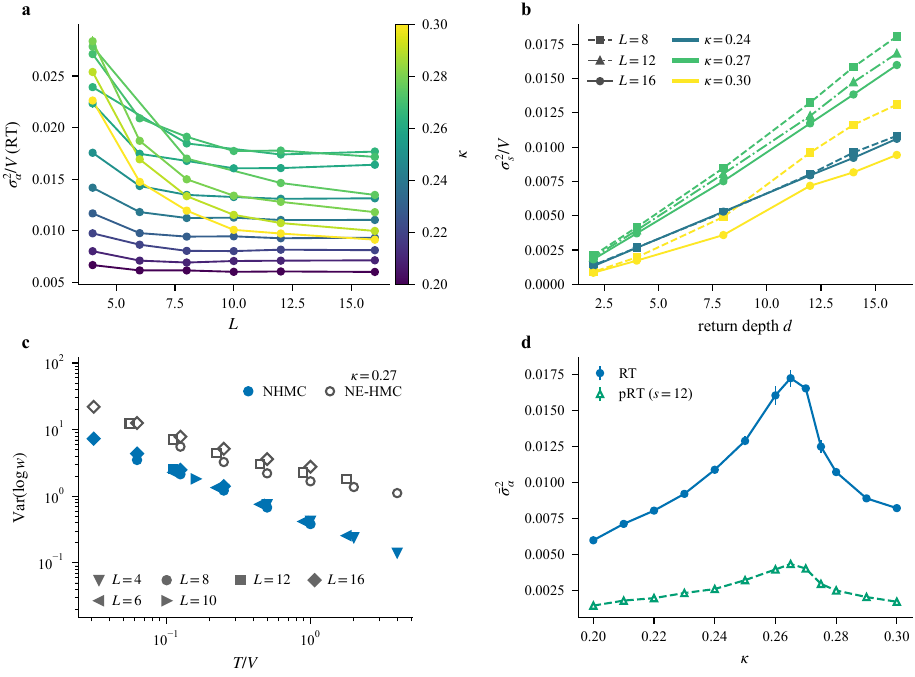}
\caption{Scaling of the conditional work variance. (a) $\sigma_\alpha^2/V$ versus $L$, with $\sigma_\alpha^2\equiv4[\operatorname{erfc}^{-1}(\bar\alpha)]^2$ inferred from the mean full round-trip acceptance, colors for $\kappa=0.20$ to $0.30$, and $L=6,10$ only at the $11$ couplings spaced by $0.01$.
(b) Directly measured $\sigma_s^2/V$ versus depth $d=T-s$ ($T=16$), line style indicating $L$ and color $\kappa$.
(c) $\operatorname{Var}(\log w)$ versus $T/V$ at $\kappa=0.27$ for NHMC and for standard nonequilibrium HMC without learning (labeled NE-HMC, at each $T$ the step size with the smallest variance in a scan), symbol shapes indicating $L$. The NHMC points at $L=4,6,8,16$ scan $T$, and those at $L=10,12$ are $T=16$ models only.
(d) Per-site slope $\bar\sigma_\alpha^2$ versus $\kappa$, from a fit in $L$ at each $\kappa$ of $\sigma_\alpha^2\approx\bar\sigma_\alpha^2V+c_\alpha$, where $c_\alpha$ is a volume-independent finite-size constant, solid for full and dashed for truncated round trips ($s=12$).
Error bars in (a) are propagated from the jackknife error of the acceptance. Appendix~\ref{app:scaling} describes the error bars in (d).}
\label{fig:scaling_var}
\end{widefigure}

\begin{widefigure}
\centering
\includegraphics[width=0.98\textwidth,keepaspectratio]{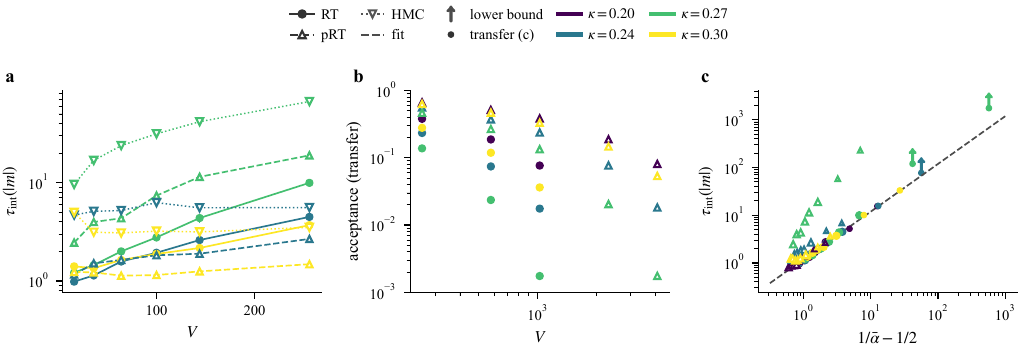}
\caption{Autocorrelation and transferred acceptance of full and truncated round trips. (a) $\tau_{\rm int}(|m|)$ versus volume $V$. Solid and dashed lines are full and truncated round trips (path-IMH in Fig.~\ref{fig:nhmc_mixing_all}). Colors give $\kappa=0.24,0.27,0.30$. Open down-triangles with dotted lines in the same color show HMC at that coupling.
(b) Mean acceptance of the frozen $L_{\rm src}=16$ model transferred to larger lattices, where $L_{\rm src}$ is the training size (Sec.~\ref{sec:cost_size}).
(c) $\tau_{\rm int}(|m|)$ versus $1/\bar\alpha-1/2$. Filled symbols are full and open symbols truncated round trips. Small symbols are models transferred to $L=24,32$. The dashed line is $1.19\,(1/\bar\alpha-1/2)$. Upward arrows mark lower bounds. Colors in (b,c) also include $\kappa=0.20$. Statistics as in Appendix~\ref{app:implementation}.}
\label{fig:scaling_tau}
\end{widefigure}

\section{Magnetization autocorrelations}
\label{sm:autocorr}

\subsection{Long stays of path-IMH}
\label{sec:field_relaxation}

The difference between path-IMH and full round trips comes from the dependence of the acceptance on the current state. At $L=16$ and $\kappa=0.27$, the effective sample fraction of the forward paths is small (Fig.~\ref{fig:sampling}(a)), yet path-IMH has an acceptance close to that of full round trips, $0.13$ and $0.14$, respectively (Supplemental Sec.~\ref{sm:native_extra}). The acceptance of path-IMH depends on the path weight carried by the current state. The larger the current weight, the lower the acceptance probability of subsequent proposals. With unbounded path weights, independence Metropolis is not geometrically ergodic~\cite{Liu_1996,Mengersen_1996}. Path-IMH can also be viewed as a pseudo-marginal chain in which the path weight is an unbiased estimate of the endpoint weight, and heavy-tailed fluctuations of such estimates likewise keep the chain in the same state for long periods~\cite{Andrieu_2009,Doucet_2015}. At similar mean acceptance, path-IMH therefore still spends long periods in a few states of large weight. A round trip instead draws a fresh reverse path for the current configuration at every update, so a large path weight of the current state does not carry over to the next update. Stay lengths show this difference directly. Let $\ell$ be the number of updates the chain stays on one configuration. For a randomly chosen time, $\sum\ell^2/\sum\ell$ is the mean length of the stay containing it. At $\kappa=0.27$ this length grows rapidly with $L$ for path-IMH. At $L=16$ it reaches $419$ for path-IMH and only $16$ for full round trips.

\subsection{Amplitude autocorrelations at four sizes}

Figure~\ref{fig:nhmc_mixing_all} extends Fig.~\ref{fig:measured}(b) to the four sizes $L=4,8,12,16$ and adds HMC and Wolff as references from conventional lattice algorithms. Each HMC update is one trajectory of unit length. One Wolff update here consists of one local Metropolis sweep and two single-cluster flips. The clusters are built on the embedded Ising variables~\cite{Wolff:1988uh,Brower:1989mt}. A flip only changes the sign of the field inside the cluster, so the field amplitude $|\phi_{\boldsymbol r}|$ at each site changes only in the local sweep. In the broken phase Wolff therefore decorrelates $|m|$ more slowly than HMC (Appendix~\ref{app:scaling}). Near criticality Wolff decorrelates $\operatorname{sign}m$ quickly.

\begin{widefigure}
 \centering
 \includegraphics[width=0.98\textwidth,keepaspectratio]{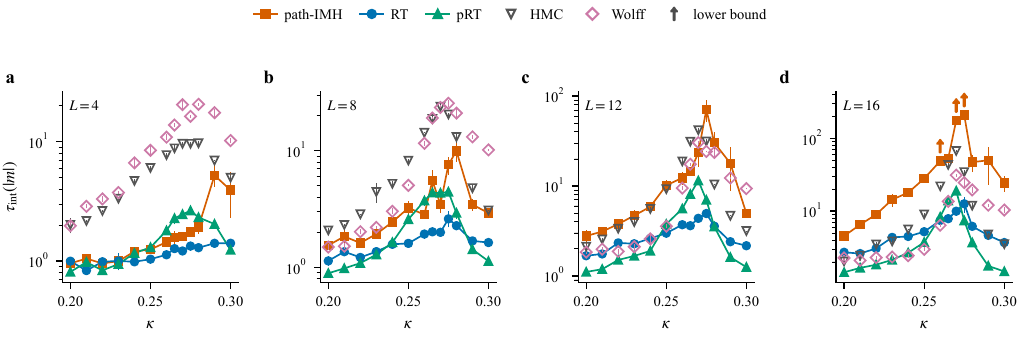}
 \caption{Magnetization-amplitude autocorrelations for models trained separately at each
 $(L,\kappa)$ (full version of Fig.~\ref{fig:measured}(b)). Panels (a) to (d) correspond to $L=4,8,12,16$, respectively.
 Filled symbols denote path-IMH, RT, and pRT. Open inverted triangles and open diamonds denote HMC and Wolff, respectively.
 All four sizes cover the $13$ couplings considered in the main text.
 The vertical axis is $\tau_{\mathrm{int}}(O)$ as defined in
 Eq.~\eqref{eq:nhmc_corr_convention}. Statistics as in Appendix~\ref{app:implementation}. The methods compute $-\nabla S$ a different number of times per update, about $10$ for HMC, and Supplemental Sec.~\ref{app:timing} gives the NHMC counts.}
 \label{fig:nhmc_mixing_all}
\end{widefigure}

\subsection{Autocorrelation times per stage}
\label{sm:tau_stage}

The three updates pass through different numbers of stages per update, $T$ for path-IMH, $2T$ for full round trips and $2d$ for truncated round trips. Table~\ref{tab:tau_stage} gives $\tau_{\rm int}(|m|)$ of the native $L=16$ models and $\tau_{\rm stage}$, the value multiplied by the number of stages per update. Each stage computes $-\nabla S$ the same number of times (Supplemental Sec.~\ref{app:timing}), so $\tau_{\rm stage}$ is also proportional to the number of evaluations of $-\nabla S$ needed for one independent sample.

\begin{table}[tb]
\caption{$\tau_{\rm int}(|m|)$ of the native $L=16$ models (in units of one update) and $\tau_{\rm stage}$ (in units of one stage, $T=16$, $d=4$). Values marked $\ge$ are lower bounds.}
\label{tab:tau_stage}
\centering
\begin{tabular}{rrrrrrr}
\toprule
 & \multicolumn{3}{c}{$\tau_{\rm int}$} & \multicolumn{3}{c}{$\tau_{\rm stage}$}\\
$\kappa$ & path-IMH & RT & pRT & path-IMH & RT & pRT\\
\midrule
$0.20$ & $4.5$ & $2.7$ & $1.4$ & $73$ & $87$ & $12$\\
$0.21$ & $6.5$ & $2.6$ & $1.7$ & $104$ & $83$ & $13$\\
$0.22$ & $9.0$ & $3.1$ & $1.8$ & $144$ & $98$ & $15$\\
$0.23$ & $14.4$ & $4.4$ & $2.1$ & $230$ & $139$ & $17$\\
$0.24$ & $17.9$ & $4.5$ & $2.7$ & $287$ & $144$ & $21$\\
$0.25$ & $28.0$ & $5.2$ & $3.7$ & $448$ & $167$ & $30$\\
$0.26$ & $\ge 49$ & $7.4$ & $8.3$ & $\ge 779$ & $238$ & $66$\\
$0.265$ & $52.6$ & $7.9$ & $13.8$ & $841$ & $252$ & $111$\\
$0.27$ & $\ge 176$ & $10.0$ & $19.0$ & $\ge 2815$ & $319$ & $152$\\
$0.275$ & $\ge 209$ & $12.6$ & $7.4$ & $\ge 3349$ & $403$ & $59$\\
$0.28$ & $47.0$ & $6.2$ & $3.7$ & $751$ & $198$ & $30$\\
$0.29$ & $49.9$ & $4.6$ & $1.8$ & $798$ & $149$ & $14$\\
$0.3$ & $24.4$ & $3.7$ & $1.5$ & $390$ & $118$ & $12$\\
\bottomrule
\end{tabular}
\end{table}

\subsection{Magnetization-sign autocorrelations}

Figure~\ref{fig:nhmc_mixing_sign} supplements Fig.~\ref{fig:measured}(b) in the main text with the magnetization-sign autocorrelations.
Both figures use the same models trained separately at each $(L,\kappa)$, and both use one complete update (including rejections) as the unit of time.
The HMC and Wolff references for sign autocorrelations come from separate runs that recorded signed magnetization, with $2\times10^5$ measurements after thermalization at each point,
while the amplitude references (Fig.~\ref{fig:nhmc_mixing_all}) use the HMC and Wolff reference chains with $10^6$ records per coupling (Appendix~\ref{app:implementation}).
The sign references cover all $13$ couplings at all four sizes.

\begin{figure*}[!htbp]
\centering
\includegraphics[width=0.98\textwidth,keepaspectratio]{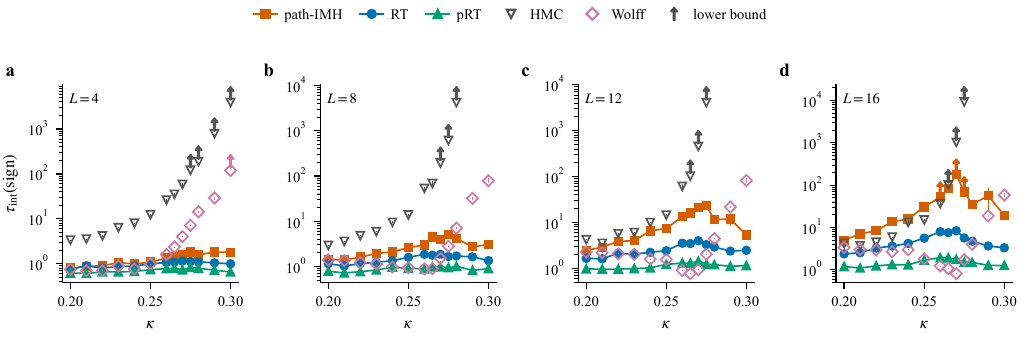}
\caption{Magnetization-sign autocorrelations for models trained separately at each $(L,\kappa)$.
Panels (a) to (d) correspond to $L=4,8,12,16$. Filled symbols with lines denote path-IMH, RT and pRT, and open down-triangles and open diamonds denote HMC and Wolff, respectively.
The vertical axis shows $\tau_{\mathrm{int}}(\operatorname{sign}m)$ as defined in Eq.~\eqref{eq:nhmc_corr_convention}.
NHMC points and error bars follow Fig.~\ref{fig:measured}(b) in the main text.
Upward arrows mark lower bounds. HMC chains without a sign flip ($\kappa\ge0.29$ at $L=8$ and $\kappa\ge0.28$ at $L=12,16$) are not shown.}
\label{fig:nhmc_mixing_sign}
\end{figure*}

\subsection{Acceptance and magnetization-amplitude changes after accepted updates}
\label{app:acceptance_comparison}

Figure~\ref{fig:nhmc_acceptance_all} extends the acceptance and amplitude-change comparison
in Fig.~\ref{fig:nhmc_acceptance} to all four lattice sizes.

\begin{figure*}[!htbp]
\centering
\includegraphics[width=0.98\textwidth,keepaspectratio]{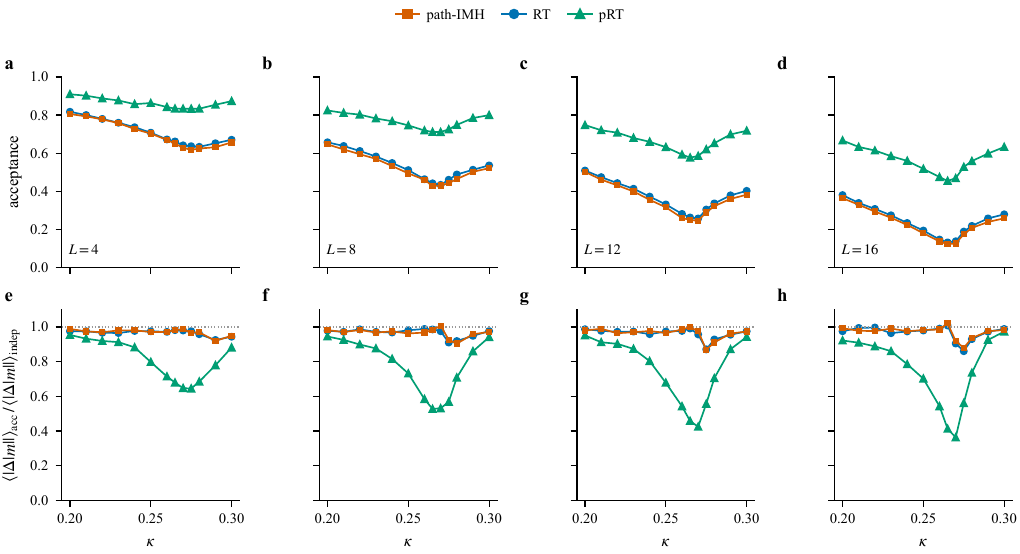}
\caption{Acceptance rates and accepted changes in magnetization amplitude for separately trained models at $L=4,8,12,16$, shown in successive columns.
(a) to (d) Acceptance rates of path-IMH, RT, and pRT.
(e) to (h) The mean absolute change in magnetization amplitude over accepted updates,
$\langle\bigl||m_{n+1}|-|m_n|\bigr|\rangle_{\mathrm{acc}}$,
divided by a baseline. For the baseline we pool the $|m|$ records of the four chains, shuffle them, and take the mean absolute difference between the original and shuffled sequences. A ratio of one means that an accepted update changes $|m|$ as much as two independent samples differ.
Dotted lines mark a ratio of one. Path-IMH is drawn on top of RT with smaller squares.
Each method pools the records of four chains. pRT uses $s=12$.}
\label{fig:nhmc_acceptance_all}
\end{figure*}

\subsection{Return-stage scan away from criticality}

Figure~\ref{fig:sscan_offcrit} repeats the return-stage scan of Fig.~\ref{fig:nhmc_turning_scan} at $\kappa=0.24$ and $0.30$. The acceptance rises monotonically with $s$. The autocorrelation time varies little for $s=6$ to $12$, becomes clearly longer again at $s=14$ for $\kappa=0.24$, and is shortest for $s=8$ to $14$ at $\kappa=0.30$ and $L\ge8$.

\begin{figure}[!htbp]
\centering
\includegraphics[width=0.8\linewidth,keepaspectratio]{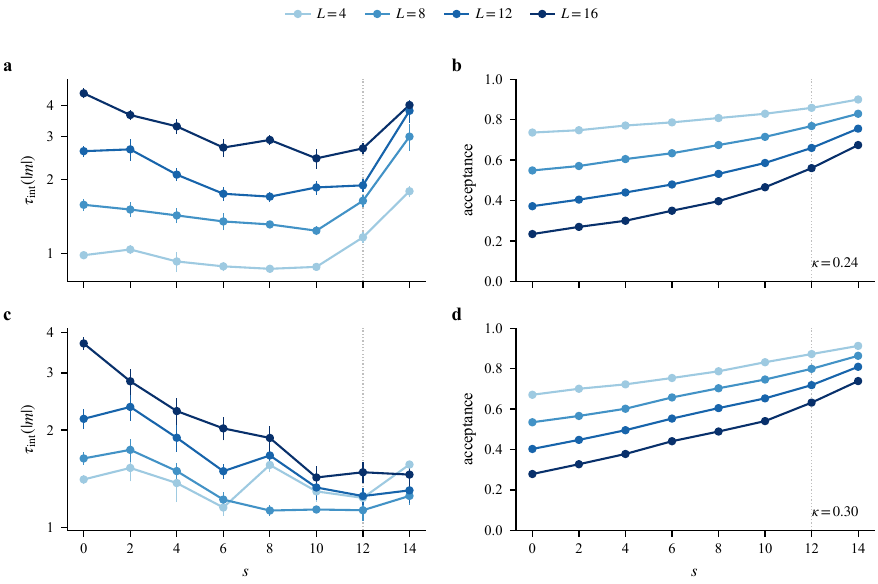}
\caption{Return-stage scan away from criticality ($T=16$, models trained at each size and coupling). Panels (a,b) are at $\kappa=0.24$ and (c,d) at $\kappa=0.30$. Panels (a,c) show $\tau_{\mathrm{int}}(|m|)$ and (b,d) the acceptance rate, with $\tau_{\mathrm{int}}$ as the mean and the acceptance as the median over four chains, with error bars as in Fig.~\ref{fig:nhmc_turning_scan}. $s=0$ is the full round trip, and the light gray dotted line marks $s=12$.}
\label{fig:sscan_offcrit}
\end{figure}

\subsection{Cumulative means and autocorrelation times}

For the first $n$ measurements define $\overline O_n=n^{-1}\sum_{j=1}^{n}O_j$,
and estimate $\tau_{\mathrm{int}}(O)$ on the first $n$ measurements of the chain.
Figure~\ref{fig:nhmc_chain_length} compares these quantities for absolute
magnetization at $L=16$, $\kappa=0.26$ and $0.27$.

\begin{figure*}[!htbp]
\centering
\includegraphics[width=0.94\textwidth,keepaspectratio]{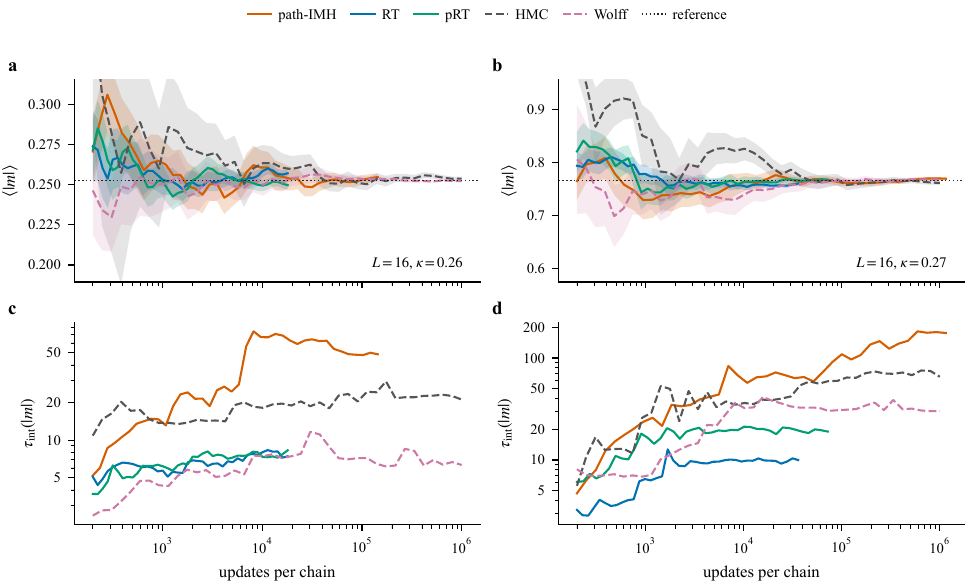}
\caption{Dependence on the number of updates per chain at $L=16$. NHMC models are trained separately at each $(L,\kappa)$. Path-IMH, full round trips and truncated round trips share the same model at each target point, and dashed curves give the HMC and Wolff references.
Panels (a,b) show the running mean of $|m|$ at $\kappa=0.26,0.27$, averaged over the four NHMC chains at each update count, with shading for estimated standard errors, and horizontal dotted lines for equilibrium reference values.
Panels (c,d) show the corresponding integrated autocorrelation times $\tau_{\mathrm{int}}(|m|)$, the larger of the Sokal automatic-window and batch-means estimates, and NHMC curves give the mean over four chains.
The horizontal axis counts updates per chain, and each curve ends at the chain length of that method.}
\label{fig:nhmc_chain_length}
\end{figure*}

At $L=16$ and $\kappa=0.26,0.27$ the cumulative magnetization is already
close to the reference, while the longest observed run of consecutive rejections in path-IMH still increases with chain length, so its $\tau_{\rm int}$ values are given as lower bounds. At $\kappa=0.27$, apart from path-IMH, the $L=16$ amplitude autocorrelation times decrease in the order
HMC, Wolff, truncated round trips, full round trips
(Table~\ref{tab:tau_per_size}).

\subsection{Representative autocorrelation functions}
Figure~\ref{fig:nhmc_l8_acf} shows the $|m|$ autocorrelation functions at $L=8$ for models trained separately at each $(L,\kappa)$. At $\kappa=0.27$ and $0.30$, path-IMH and full round trips have almost the same mean acceptance, yet the path-IMH autocorrelation has a long tail extending to $j\approx30$ and a longer $\tau_{\mathrm{int}}(|m|)$ (Table~\ref{tab:tau_per_size}). For truncated round trips ($s=12$), the autocorrelation function decays most slowly at $\kappa=0.27$ and most rapidly at $\kappa=0.30$. The $\kappa=0.27$ result is consistent with the return-stage scan of Fig.~\ref{fig:nhmc_turning_scan}.

\begin{figure*}[!htbp]
\centering
\includegraphics[width=0.91\textwidth,keepaspectratio]{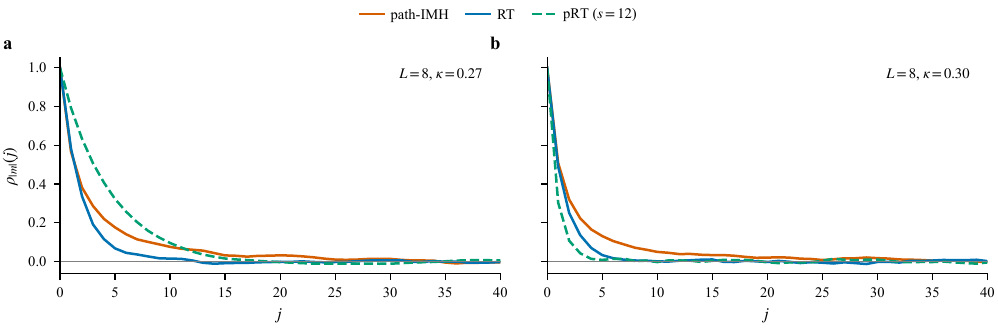}
\caption{Autocorrelation functions of the absolute magnetization at $L=8$ for models trained separately at each $(L,\kappa)$. (a) $\kappa=0.27$, (b) $\kappa=0.30$. Path-IMH, full round trips (RT) and truncated round trips (pRT, $s=12$) use the same model with $4$ chains each. Chains have $2\times10^4$ updates for pRT at $\kappa=0.27$ and path-IMH at $\kappa=0.30$ and $10^4$ for the others. We subtract from each chain its own mean and pool the four chains to estimate the autocorrelation function. The lag $j$ counts complete updates including rejected ones.}
\label{fig:nhmc_l8_acf}
\end{figure*}

\section{Model transfer}
\label{sm:transfer}

\subsection{Supplementary model-transfer results}
\label{app:v_numerics}

Figure~\ref{fig:transfer_all} gives the coupling-resolved results for transfer of the $L_{\rm src}=16$ models
to $L=24,32$ in Sec.~\ref{sec:size_transfer}.
Each coupling uses its corresponding source model, and four chains are initialized separately from configurations generated by that model at the target size.
Autocorrelation times follow Eq.~\eqref{eq:nhmc_corr_convention}.
Results at the source size are also shown to illustrate the behavior of the same learned parameters when the target size changes.

In panels (c,d), the mean absolute magnetization is an inverse-variance-weighted average over four chains. Figure~\ref{fig:transfer_obs_all} uses the same chains but pools all their records.

\begin{widefigure}
\centering
\includegraphics[width=0.96\textwidth]{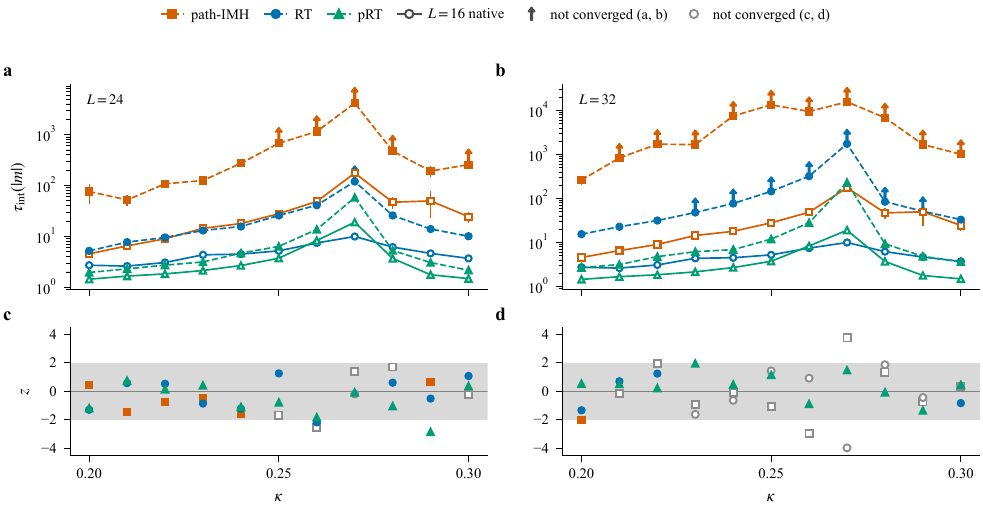}
\caption{Transfer of models trained separately at each coupling on $L=16$ to $L=24$ and $32$.
Each coupling retains the parameters of its corresponding $L=16$ model.
Panels (a,b) show the magnetization-amplitude integrated autocorrelation time
$\tau_{\mathrm{int}}(|m|)$ on $L=24,32$, respectively. Filled symbols with dashed lines give the transfer results, while open symbols with solid lines give the native $L=16$ results at the same coupling.
Panels (c,d) show standardized deviations of absolute magnetization from the independent reference for the transferred chains at each size, and the gray band marks $|z|\le2$. Panels (c,d) combine the four chains by inverse-variance weighting.
At $L=24$ and $32$ all three updates cover the $11$ couplings $\kappa=0.20$ to $0.30$ in steps of $0.01$. Upward arrows in (a,b) mark lower bounds, and open symbols with gray edges in (c,d) mark the same chains. Statistics as in Appendix~\ref{app:implementation}.
Lines guide the eye.}
\label{fig:transfer_all}
\end{widefigure}

After transfer, the path-IMH autocorrelation times at $L=24$ are about $6$ to $30$ times those of RT. At $L=32$ the path-IMH values are lower bounds of about $50$ to $900$ times the pRT values at the same coupling.
The RT and pRT autocorrelation times grow fastest with lattice size at $\kappa=0.27$.

Figures~\ref{fig:transfer_obs_all} and \ref{fig:nhmc_transfer_matrix_all} extend Fig.~\ref{fig:transfer_obs} and Fig.~\ref{fig:transfer_eff}(b). The first includes the Binder cumulant and the second the weight effective sample fraction of independent forward paths. The largest deviation from the reference, about $-4.5$ in $\langle|m|\rangle$, comes from path-IMH at $L=32$, $\kappa=0.26$. For truncated round trips the standardized deviations are all within $2.2$ at $L=48$. At $L=64$ four values lie outside $2$. Three are at $\kappa=0.27$, with $z$ about $2.6$, $-4.4$ and $3.2$ for $\langle|m|\rangle$, $\chi$ and $U_4$, and the fourth is $U_4$ at $\kappa=0.25$, with $z$ about $3.8$.

\begin{widefigure}
\centering
\includegraphics[width=0.96\textwidth]{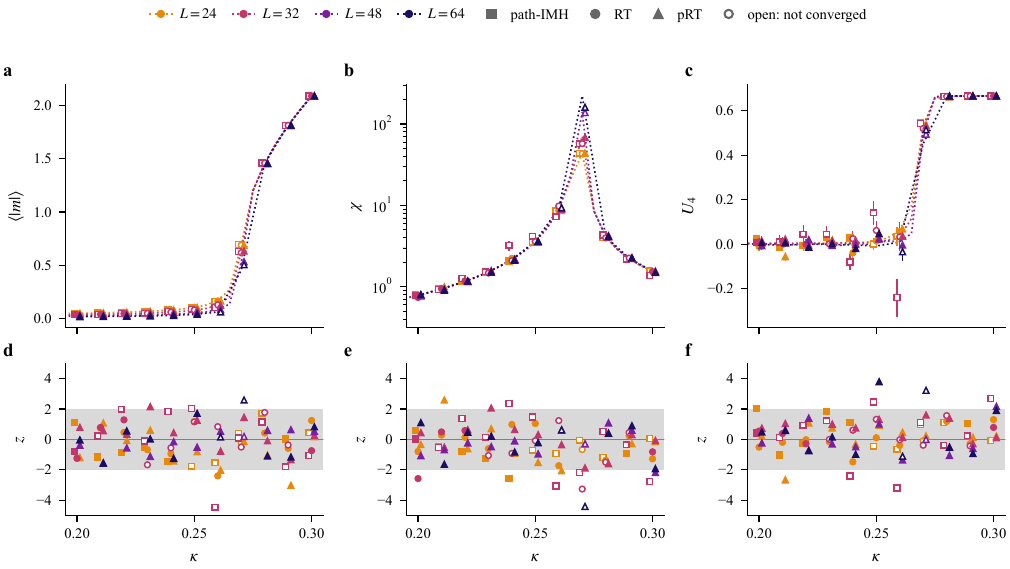}
\caption{Equilibrium observables after transferring models trained at each coupling
 on $L=16$ to $L=24$ to $64$ (truncated round trips only at $L=48,64$). It extends Fig.~\ref{fig:transfer_obs} to $U_4$.
 (a) to (c) show $\langle|m|\rangle$, $\chi$, and $U_4$, respectively.
 Dotted curves are HMC and Wolff references at the corresponding target sizes, with
 reference couplings spaced by $0.005$ ($0.01$ at $L=48,64$).
 Points are measurements from corrected chains, shaded from light to dark for $L=24,32,48,64$.
 Error bars are jackknife estimates over the pooled records of the four chains, each divided into $20$ blocks. Figure~\ref{fig:transfer_all}(c,d) combines the same chains by inverse-variance weighting instead. Open symbols mark chains that do not meet the convergence criterion.
 (d) to (f) show standardized deviations $z$ from the reference, with chain and
 independent-reference standard errors combined in the denominator, and gray bands
 indicate $|z|\le2$.
 All sizes cover the $11$ couplings $\kappa=0.20$ to $0.30$ in steps of $0.01$.}
\label{fig:transfer_obs_all}
\end{widefigure}

\begin{widefigure}
\centering
\includegraphics[width=0.98\textwidth,keepaspectratio]{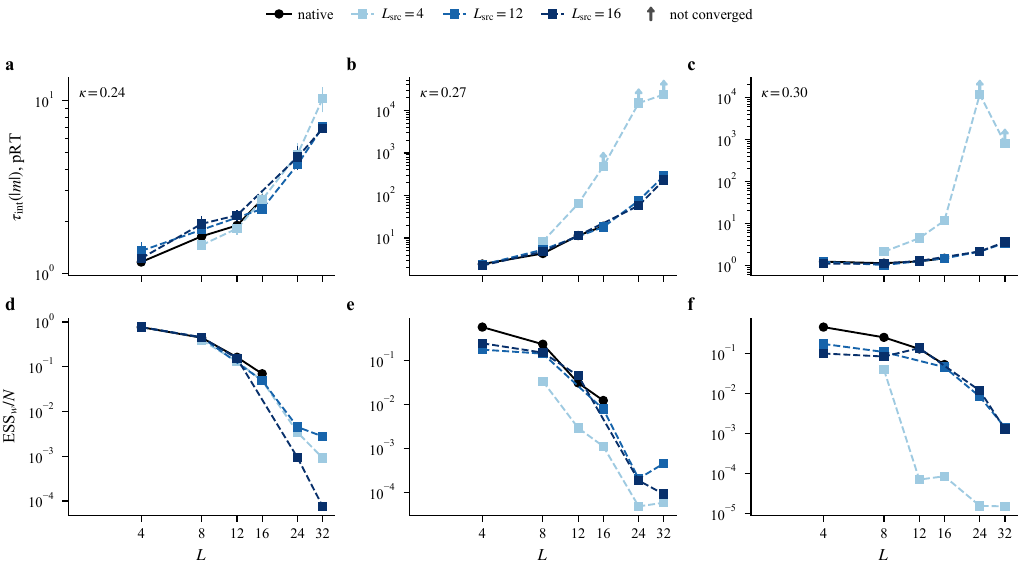}
\caption{pRT autocorrelation times and independent-path weights of frozen models at different target sizes (full version of Fig.~\ref{fig:transfer_eff}(b), including weight effective sample fractions). Columns correspond to $\kappa=0.24,0.27,0.30$, and the horizontal axis is the target size $L$.
(a) to (c) show $\tau_{\mathrm{int}}(|m|)$ for pRT with return stage $s=12$. Statistics as in Appendix~\ref{app:implementation}.
(d) to (f) show $\mathrm{ESS}_w/N$ of $65536$ independent forward paths, defined in Eq.~\eqref{eq:phi4_snis}.
Black circles denote models trained separately at $L=4,8,12,16$, and blue squares denote models transferred from $L_{\rm src}=4,12,16$, shaded from light to dark blue for $L_{\rm src}=4,12,16$. Upward arrows mark lower bounds.
Points with equal source and target sizes belong to the black curves. Lines are guides to the eye.}
\label{fig:nhmc_transfer_matrix_all}
\end{widefigure}

\section{Partition-function estimates}
\label{sm:logz}

\subsection{Partition functions on square lattices}

\begin{widefigure}[tb]
\centering
\includegraphics[width=0.98\textwidth,keepaspectratio]{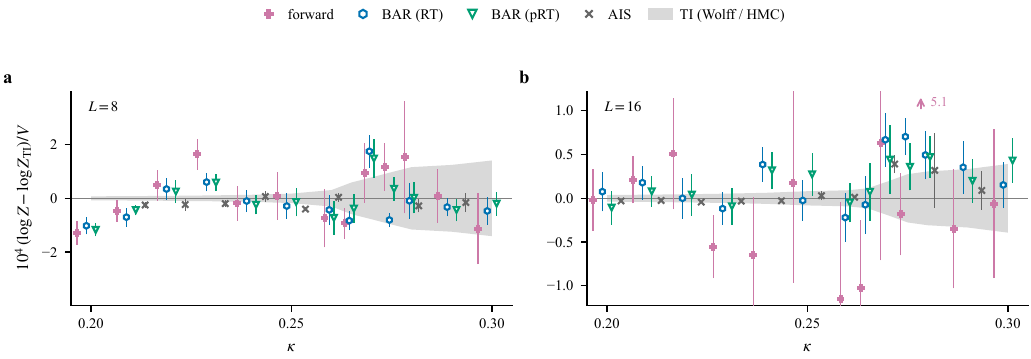}
\caption{Partition-function estimates on square lattices for NHMC models trained separately at each $(L,\kappa)$. The vertical axis is the deviation per site, $(\log Z-\log Z_{\mathrm{TI}})/V$, in units of $10^{-4}$.
(a) and (b) correspond to $L=8$ and $16$.
The points labeled forward use the mean of $e^{-W}$ over complete forward paths to estimate $Z$. BAR ($X$) denotes the BAR estimate (Appendix~\ref{app:bar_free_energy}) that combines the forward paths with the reverse records of the $X=$ RT or pRT chains.
Crosses are annealed importance sampling (AIS), and the gray band is the uncertainty of the HMC and Wolff TI reference. In (b) the forward point at $\kappa=0.28$ lies outside the vertical range and is marked by an arrow with its value. Error bars are the standard errors of the estimates.}
\label{fig:nhmc_square_logz}
\end{widefigure}

Figure~\ref{fig:nhmc_square_logz} compares partition-function estimates at $L=8$ and $16$.
At each $(L,\kappa)$ all estimates use the same frozen model parameters. Path-IMH, RT, and pRT chains supply the reverse records. The figure shows RT and pRT, and Fig.~\ref{fig:nhmc_square_logz_all} gives all four sizes and path-IMH.
TI provides the reference values.
The figure also shows annealed importance sampling (AIS)~\cite{Neal:2001hwj}. AIS involves no learning and provides an accuracy reference for a conventional nonequilibrium path estimate, independent of the TI reference. It starts from exact samples at $\kappa=0$ and raises $\kappa$ linearly to its target value, with $2000$, $4000$, $8000$ and $16000$ steps at $L=4,8,12,16$ and two HMC trajectories per step. Each estimate uses $4096$ to $8192$ parallel paths, and we quote the mean and standard error over $2$ to $4$ independent repetitions. Appendix~\ref{app:bar_free_energy} gives the settings of the TI reference.
The deviations quoted below are for the total $\log Z$.

\begin{widefigure}
\centering
\includegraphics[width=0.98\textwidth,keepaspectratio]{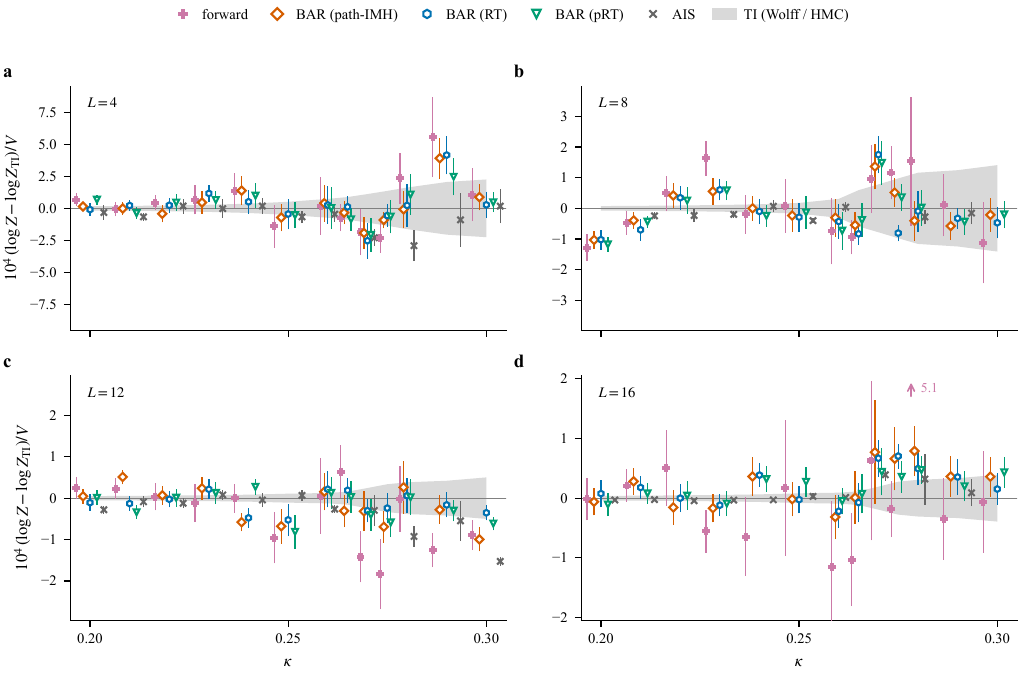}
\caption{Partition-function estimates on square lattices for NHMC models trained separately at each $(L,\kappa)$ (full version of Fig.~\ref{fig:nhmc_square_logz} for four sizes and three corrected chains). The vertical axis is the deviation per site, $(\log Z-\log Z_{\mathrm{TI}})/V$, in units of $10^{-4}$.
(a) to (d) correspond to $L=4,8,12,16$.
The points labeled forward use the mean of $e^{-W}$ over complete forward paths to estimate $Z$. BAR ($X$) is the BAR estimate of Eq.~\eqref{eq:bar_chain_bridge} that combines the reverse records of the $X=$ path-IMH, RT or pRT chain with the forward samples. The one-sided reverse estimates are shown in Fig.~\ref{fig:nhmc_square_logz_rev}.
Crosses are AIS, and the gray band is the uncertainty of the HMC and Wolff TI reference. In (d) the forward point at $\kappa=0.28$ lies outside the vertical range and is marked by an arrow with its value. Error bars are the standard errors of the estimates.}
\label{fig:nhmc_square_logz_all}
\end{widefigure}

\subsection{One-sided reverse estimates on square lattices}
\label{app:square_rev}
Figure~\ref{fig:nhmc_square_logz_rev} shows the one-sided reverse estimates from the same records as Fig.~\ref{fig:nhmc_square_logz}, i.e., the exponential averages of the complete reverse records supplied by the path-IMH, RT and pRT chains. For $L=4$ to $12$ their deviation from TI stays below about $0.035$. At $L=16$ it grows to about $0.05$, while the BAR estimates that combine the same reverse records with the forward samples stay within about $0.02$.

\begin{widefigure}
\centering
\includegraphics[width=0.98\textwidth,keepaspectratio]{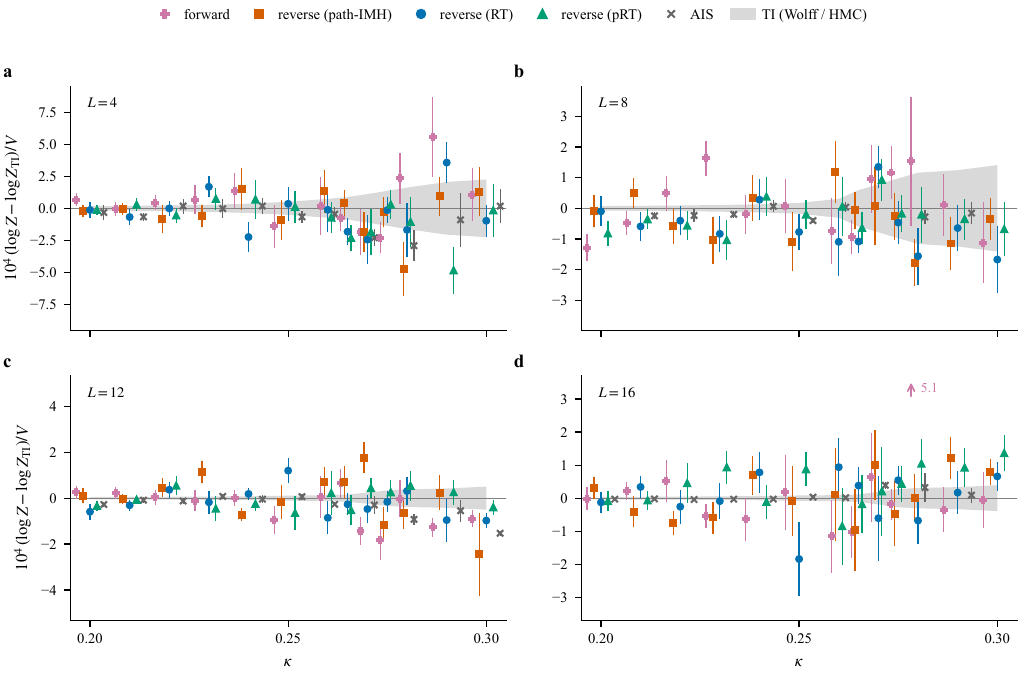}
\caption{One-sided reverse partition-function estimates on square lattices, on the same vertical axis as Fig.~\ref{fig:nhmc_square_logz}, $(\log Z-\log Z_{\mathrm{TI}})/V$ in units of $10^{-4}$. The points reverse ($X$) are the exponential averages of the complete reverse records supplied by the $X=$ path-IMH, RT or pRT chain. The forward points, AIS and the gray band are as in Fig.~\ref{fig:nhmc_square_logz}, and the arrow in (d) marks the forward point outside the vertical range.}
\label{fig:nhmc_square_logz_rev}
\end{widefigure}

\section{CMCD and DBS}
\label{sm:udb}

\textbf{Inference-time step refinement.} Each training step is divided into $n$ substeps. The substeps keep the network and time index of that step, and the step size shrinks to $1/n$. This refinement of the $T$ trained steps to $nT$ lowers the CMCD work variance monotonically in $n$ towards a plateau for all three numbers of training steps ($T=80,160,320$) and both couplings. For DBS at $\kappa=0.27$, the work variance decreases by about one third and then levels off. At $\kappa=0.30$, refinement increases the variance for $T=80$ and $T=320$ and decreases it for $T=160$ (Fig.~\ref{fig:udb_refine}). Truncated round trips (redrawing the last quarter) gain slightly in acceptance under refinement while their integrated autocorrelation time stays nearly unchanged. The forward and reverse kernels of CMCD are paired by the same network ($v^{\rm R}=-v^{\rm F}$), so refinement only reduces the discretization error, while DBS parameterizes the reverse drift with a separate network and carries no such guarantee after refinement.

\begin{widefigure}
\centering
\includegraphics[width=0.91\textwidth,keepaspectratio]{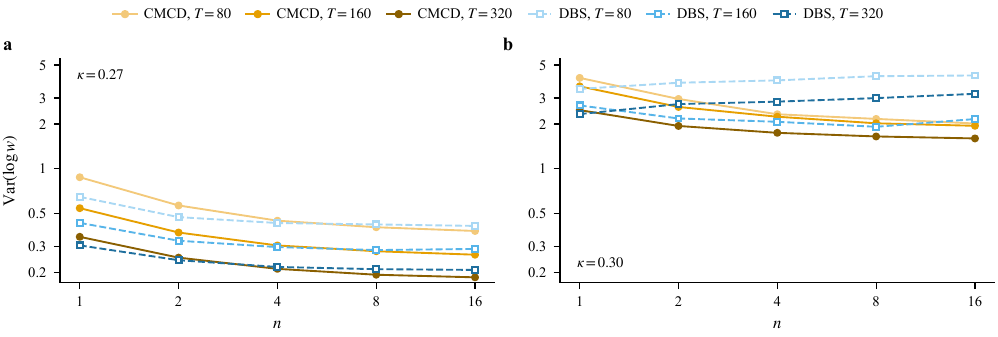}
\caption{$\operatorname{Var}(\log w)$ under inference-time step refinement, $L=8$, (a) $\kappa=0.27$, (b) $\kappa=0.30$. $n$ is the number of substeps per trained step at fixed total integration time. Solid lines with filled circles are CMCD, dashed lines with open squares DBS, and the color shade (light to dark) gives the number of training steps $T=80,160,320$.}
\label{fig:udb_refine}
\end{widefigure}

\section{Intermediate-field memory and flow between sign regions}
\label{app:iv_numerics}

\subsection{Controlled comparison of conditional region flow}
\label{sec:conditional_flow}

The old and new trajectories of an RT update share an intermediate field. Even an exact redraw
of the endpoint from that field need not reach another region if the
intermediate field retains information about the old region.
Let $A_+=\{x:m(x)>0\}$ and $A_-=\{x:m(x)<0\}$ be the two magnetization
regions, with target probabilities $\pi_+=\pi(A_+)$ and $1-\pi_+$. The record sets $\mathcal A_\pm(u)$ and the conditional probability $p_+(u)$ are defined below, and $\pi_s,h_s,f_s$ are as in Appendix~\ref{app:path_correction}.

Starting from $X\sim\pi$, reverse sampling produces an intermediate field
$U$ and a paired old forward record $a$. Write their joint law as
$\pi_s(du)h_s(da\mid u)$.
The return stage $s$ is fixed here.
Define $p_+(u)=\Pr_{h_s}(X\in A_+\mid U=u)$, the conditional
probability of a positive endpoint.
If the new record is drawn independently from $h_s(\cdot\mid U)$,
the old and new endpoints are conditionally independent given $U$. The ideal
flow $\Omega_{+-}$ is the probability that one such update moves the configuration
from the positive to the negative region. It is
\begin{equation}
 \begin{aligned}
 \Omega_{+-}&=\mathbb E_{\pi_s}[p_+(U)(1-p_+(U))]\\
 &=\pi_+(1-\pi_+)-\operatorname{Var}_{\pi_s}[p_+(U)].
 \end{aligned}
 \label{eq:nhmc_ideal_sector_flow}
\end{equation}
When $p_+(u)$ is usually close to zero or one, the intermediate field
almost determines the endpoint region, so a new trajectory is unlikely
to change it. A weaker dependence on $u$ weakens the constraint that the intermediate field places on the sign of the endpoint.

The actual update proposes a record from the learned $f_s(\cdot\mid U)$
and applies the Metropolis test. Define
$\Omega_{+-}^{\mathrm{RT}}=\Pr_\pi(X\in A_+,X'\in A_-)$, where $X'$ is the
state after acceptance or rejection. This is the joint probability of a
positive-to-negative transition per stationary update, including the
probability of starting in the positive region. For these two
complementary regions,
\begin{equation}
 0\leq\Omega_{+-}-\Omega_{+-}^{\mathrm{RT}}
 \leq\mathbb E_{\pi_s}
 d_{\mathrm{TV}}(f_s(\cdot\mid U),h_s(\cdot\mid U)),
 \label{eq:nhmc_flow_gap}
\end{equation}
where $d_{\mathrm{TV}}$ is total variation distance, one half of the
$L^1$ distance between the conditional densities. The ideal flow bounds the accepted flow between the two regions from above. A mismatch between the proposal and target conditional distributions reduces the accepted flow relative to the ideal flow by at most
$\mathbb E\,d_{\rm TV}$.
Supplemental Sec.~\ref{app:flow} derives both relations.

The two factors, the sign information kept by the intermediate field and the mismatch of the conditional proposal, can be varied separately. We use fixed
Gaussian momentum distributions that are not learned, at $L=8$, $\lambda=0.022$, and $\kappa=0.27$.
The first calculation changes the selection probability $p_{\mathrm{flip}}$
of an explicit global sign transformation to vary the connection between
regions. Its paths have $T=20$ stages with $5$ leapfrog substeps of size $0.135$ per stage. The second omits sign flips and changes the forward and reverse
width factors $\zeta_{\rm F},\zeta_{\rm R}$ relative to these fixed Gaussian distributions, either
separately or together, to probe conditional mismatch.
Figure~\ref{fig:nhmc_flow} compares the ideal and accepted flows in both
calculations.

\begin{widefigure}
 \centering
 \includegraphics[width=0.76\textwidth]{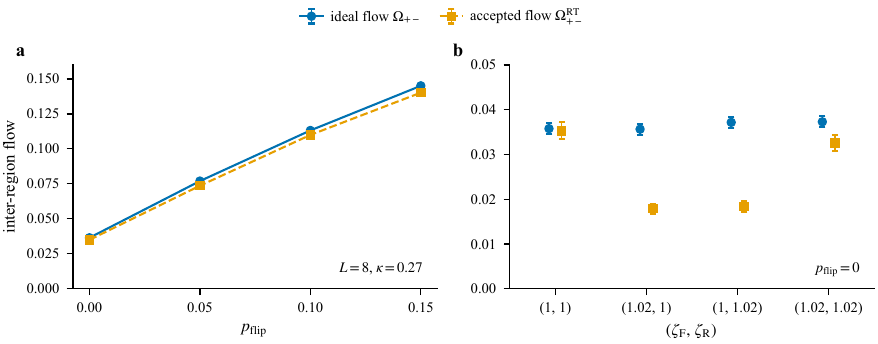}
 \caption{Ideal flow $\Omega_{+-}$ and accepted flow $\Omega_{+-}^{\mathrm{RT}}$ for paths with fixed Gaussian momentum distributions that are not learned, at $L=8$, $\lambda=0.022$, and $\kappa=0.27$. (a) Varying the selection probability $p_{\mathrm{flip}}$ of the explicit sign transformation, with lines guiding the eye. (b) Varying the forward and reverse width factors $(\zeta_{\rm F},\zeta_{\rm R})$ relative to these fixed distributions, without sign flips. Ideal flows use $512$ conditional records per intermediate field. Starting configurations come from several equilibrium HMC chains. Standard errors treat starting points from the same HMC chain as one group. Error bars in (a) come from the spread between groups and in (b) from a delete-one-group jackknife.}
 \label{fig:nhmc_flow}
\end{widefigure}

Region connectivity and conditional matching control the motion between
regions separately. As $p_{\mathrm{flip}}$ increases from $0$ to $0.15$, the ideal and
accepted flows rise together. The accepted flow grows by about a factor of
four, while the mean acceptance stays at about $0.955$. In the second calculation,
changing only one width (by $2\%$) leaves the estimated ideal flow nearly
unchanged but suppresses the accepted flow, while changing both widths together
raises the accepted flow again.

\subsection{Bound between the ideal and accepted conditional flow}
\label{app:flow}

At a fixed intermediate field $u$, let $h(a\mid u)$ be the density of the old
record under $P_\theta$ and $f(a\mid u)$ the new-record proposal density.
The accepted joint density is
$\min\{h(a)f(a'),f(a)h(a')\}$, with the conditioning on $u$ suppressed.
Let $\mathcal A_+(u)$ be the set of records ending in $A_+$ and $\mathcal A_-(u)$ its complement.
Writing $p_+=\int_{\mathcal A_+}h$ (the $p_+(u)$ above) and $\tilde p_+=\int_{\mathcal A_+}f$, the conditional accepted flow obeys
\begin{equation}
 \begin{aligned}
 \Omega^{\mathrm{RT}}_{+-}(u)&=\int_{\mathcal A_+}\!da\int_{\mathcal A_-}\!da'
 \min\{h(a)f(a'),f(a)h(a')\}\\
 &\leq\min\{p_+(1-\tilde p_+),(1-p_+)\tilde p_+\}\leq p_+(1-p_+).
 \end{aligned}
\end{equation}
The final inequality follows by considering $\tilde p_+\geq p_+$ and $\tilde p_+<p_+$
separately. Thus the ideal flow is an upper bound on the accepted flow
between two complementary regions.

For the lower bound, write $(h-f)_+(a)=\max\{h(a)-f(a),0\}$ and
$\delta=\int(h-f)_+\,da=d_{\mathrm{TV}}(h,f)$.
Pointwise,
\begin{equation}
 \begin{aligned}
 &\min\{h(a)f(a'),f(a)h(a')\}\\
 &\quad\geq h(a)h(a')-h(a')(h-f)_+(a)-h(a)(h-f)_+(a').
 \end{aligned}
\end{equation}
Integration gives
$p_+(1-p_+)-\Omega^{\mathrm{RT}}_{+-}(u)\leq(1-p_+)\int_{\mathcal A_+}(h-f)_++p_+\int_{\mathcal A_-}(h-f)_+\leq\delta$.
Averaging over $\pi_s(u)$ proves Eq.~\eqref{eq:nhmc_flow_gap}.
Finally, $\mathbb E_{\pi_s}p_+(U)=\pi_+$ yields
Eq.~\eqref{eq:nhmc_ideal_sector_flow} by expanding the variance.
These relations concern stationary one-step flow. A full integrated
autocorrelation time also contains the subsequent evolution of the field.

\end{document}